\documentclass[twoside,11pt]{article}

	\usepackage{jmlr2e}

	\firstpageno{1}
	
	\usepackage{amsmath}
	\usepackage{bbm}
	
	\makeatletter
	\@ifundefined {theorem} {
		\newtheorem{theorem}{Theorem}
		\newtheorem{lemma}[theorem]{Lemma} 
		\newtheorem{corollary}[theorem]{Corollary}
		\newtheorem{definition}[theorem]{Definition}
		\newtheorem{remark}[theorem]{Remark}
		\newtheorem{proof}[theorem]{Proof}
	}{}
	\makeatother
	\newtheorem{assumption}[theorem]{Assumption}

	\newcommand {\psucc} {\mathcal{P}}
	\newcommand {\psucce} {\mathcal{P}^{\varepsilon}}
	\newcommand {\eprs} {\mathcal{A}}
	\newcommand {\pr} {\mathbb{P}}

	\newcommand {\ops} {\eprs_0}
	\newcommand {\Etrav} {E_\text{trav}}
	\newcommand {\allNsa} {\langle N_{s,a} \rangle}
	\newcommand {\allNsas} {\langle N_{s,a,s'} \rangle}
	\newcommand {\allBarNsa} {\langle \bar{N}_{s,a} \rangle}
	\newcommand {\allBarNsas} {\langle \bar{N}_{s,a,s'} \rangle}
	\newcommand {\allsth}[1] {\langle #1 \rangle}
	\newcommand {\Vbf} {V^{\pi^{-+}_k}}
	\newcommand {\pbf} {{\pi^{-+}_k}} 
	\newcommand {\Var} {\mathrm{Var} \,}
	\newcommand {\mean} {\mathbb{E}}
	
	\newcommand{\dotsim}{\mathrel{\dot{\sim}}}
	
	\mathchardef \sslash = "2D
	\DeclareMathOperator*{\argmin}{arg\,min}
	
	\usepackage{enumitem}
	
	\usepackage{graphicx}
	\usepackage{caption}
	\usepackage{subcaption}
	\DeclareGraphicsExtensions{.pdf,.png}
	
	\usepackage{tikz}
	\usetikzlibrary{automata}
	\usetikzlibrary{arrows}
	\usetikzlibrary{shapes}
	
	\usepackage{bm} 

\begin{document}

\title{Success Probability of Exploration: \\a Concrete Analysis of Learning Efficiency} 

\author{\name Liangpeng Zhang \email udars@mail.ustc.edu.cn \\
	\name Ke Tang \email ketang@ustc.edu.cn\\
	\addr UBRI, School of Computer Science and Technology\\
	University of Science and Technology of China\\
	Hefei, Anhui, 230027, China\\
	\name Xin Yao \email x.yao@cs.bham.ac.uk\\
	\addr CERCIA, School of Computer Science\\
	The University of Birmingham\\
	Edgbaston, Birmingham B15 2TT, U.K. 
}

\editor{ }

\maketitle

\begin{abstract}
Exploration has been a crucial part of reinforcement learning, yet several important questions concerning exploration efficiency are still not answered satisfactorily by existing analytical frameworks.
These questions include exploration parameter setting, situation analysis, and hardness of MDPs, all of which are unavoidable for practitioners. 
To bridge the gap between the theory and practice, we propose a new analytical framework called \textit{the success probability of exploration}.
We show that those important questions of exploration above can all be answered under our framework, and the answers provided by our framework meet the needs of practitioners better than the existing ones.
More importantly, we introduce a concrete and practical approach to evaluating the success probabilities in certain MDPs without the need of actually running the learning algorithm.
We then provide empirical results to verify our approach, and demonstrate how the success probability of exploration can be used to analyse and predict the behaviours and possible outcomes of exploration, which are the keys to the answer of the important questions of exploration.
\end{abstract}

\begin{keywords}
  Reinforcement learning, exploration efficiency, learning theory, analytical framework.
\end{keywords}

\section{Introduction}
\label{secIntro}
\label{sec3Q}
Exploration is an essential process for Reinforcement Learning (RL) agents to resolve uncertainty \citep{SuttonBarto98}.
In an initially unknown environment, a learning agent has to explore through different states and actions to gather information, so that better policies can be discovered accordingly via its planning process.
Most RL algorithms include a specific part, often called an exploration strategy, that explicitly deals with exploration.
Numerous exploration strategies have been designed and proposed in the literature, and some of the most popular ones among them are $\varepsilon$-greedy, Boltzmann action selection, Explicit Explore or Exploit \citep{kearns2002near}, R-MAX \citep{brafman2002r}, Upper Confidence RL \citep{jaksch2010near}, and Bayesian approaches \citep{vlassis2012bayesian}.

There have been several long-standing crucial questions regarding exploration faced by every RL practitioners.
These questions can be categorized into three groups as follows.

\begin{description}
	\item[(Q1) Exploration parameter setting] For every common exploration strategy, there is at least one parameter that balances exploration and exploitation by controlling the activeness of exploration. 
	Very often the best setting of the exploration parameter varies in different learning tasks, and whether knowing it or not before running the algorithm has a great impact on the overall efficiency.
	Then, given a learning task, what is the corresponding best parameter setting for the exploration strategy?
	Is it possible to be found out without time-consuming trial-and-error, or even without running the learning algorithm at all?
	
	\item[(Q2) Situation analysis] RL practitioners often come into the situation where a learning algorithm has been executed for a while on a RL task, but the result is not satisfactory.
	In such cases, RL practitioners have to decide what to do next in order to improve the situation.
	Should the algorithm simply be kept running for some more time steps?
	Or should the practitioners tune its exploration parameter, or even reconsider its state and action representations?
	Is there any metric or descriptive statistics available that can be relied on to make the decision?
	
	\item[(Q3) Hardness of exploration] There have been many real-world applications where a vanilla algorithm could discover a surprisingly good policy within limited steps in an apparently complicated environment.
	There are also many seemingly simple environments that eventually turned out to be very hard to explore.
	RL practitioners have to decide how much resource should be allocated to the agent exploring the environment based on how difficult it is.
	Is there any practical methodology able to describe and quantify the hardness of exploration in different learning tasks?
	How to allocate resource for exploration accordingly?
\end{description}

Traditionally, these questions have been dealt with under different frameworks and approaches, but unfortunately none of them could give a solution practically useful and widely applicable.
Many solutions appear in the form of empirical rules extracted from experiences.
It often happens that these rules do not work well, and the time-consuming trail-and-error process has to be performed to find the best answer instead.
For example, in the scenario of exploration parameter setting where practitioner hopes to find the best value of $\varepsilon$ for the $\varepsilon$-greedy strategy, the default value $\varepsilon=0.1$ from the textbook \citep{SuttonBarto98} will usually be tried first.
If it leads to poor performance, then the practitioner has to manually try out values with different magnitude, say $0.001, 0.003, 0.01, 0.03, 0.1, 0.3$ etc., to decide the best one.

In addition to than the straightforward approach above, the first group of questions, exploration parameter setting, are mostly investigated under the framework of PAC analysis \citep{valiant1984theory, fiechter1994efficient, kakade2003sample, strehl2009pac} and the regret bound analysis \citep{auer2006logarithmic,jaksch2010near}.
In both of the frameworks, asymptotic bounds of some performance metric (sample complexity and regret, respectively) are derived for a given exploration strategy, and its parameter settings that lead to the corresponding bounds are specified.
Many well-known strategies, for example R-MAX and its variants (e.g. MoR-MAX \citep{szita2010mormax}, V-MAX \citep{rao2012v}, ICR and ICV \citep{zhang2015increasingly}), Model-Based Interval Estimation \citep{strehl2005mbie}, and UCRL$\gamma$ \citep{lattimore2014near}, have been proved to have sample complexity bounds polynomial to the scale parameters of the learning task.
The UCRL families are also proved to have regret bounds sublinear to the horizon of the cumulative rewards  \citep{jaksch2010near,ortner2012online}.

The main drawbacks of these analyses is that their theoretical results are not sufficiently relevant to the practical needs.
For example, the PAC theory for R-MAX \citep{strehl2009pac} requires its parameter $m$ to be set polynomial to the scale parameters of the learning task so that its sample complexity can be polynomial as well.
However, in practice $m$ is usually fixed to some value around 10-20 \citep{strehl2004empirical} regardless of the scale of the task, which violates the basic condition of the PAC theory.
Meanwhile, the PAC theory does not provide any prediction of the performance of R-MAX with its $m$ fixed to small values like 10-20.
This results in a strange dilemma where practitioners have to choose one between theoretical performance guarantee and actual efficiency, and in most cases the latter is chosen, leaving the former invalid in practice.
In \cite{zhang2015increasingly}, some workarounds are proposed so that the practitioners are not forced to discard theoretical guarantees in exchange for efficiency.
Nevertheless this cannot turn these existing guarantees to be more practically relevant and useful.

In general machine learning setting, situation analysis is often conducted via observing the learning curve, i.e. the plot of generalization error of the learned model against the size of the training dataset or the number of executed iterations \citep{perlich2011learning}.
This approach can also be applied to Reinforcement Learning by plotting the current total reward or the expected cumulative reward over time \citep{SuttonBarto98}.
However, the former does not directly represent the goodness of the learned policy, while evaluating the latter is actually a value prediction problem requiring an additional independent learning process, which can be as costly as the original learning process.
Even if such curves are accessible to practitioners, they are likely to be in a zigzag style in practice, and the decisions for improving the situation still have to be made according to experiences and domain knowledge.

There has been some works related to the hardness questions, such as action gap \citep{farahmand2011action} and distribution-norm \citep{maillard2014hard}.
The action gap captures the hardness of planning rather than exploration, and thus does not provide a direct answer to the hardness questions of exploration.
The distribution-norm is a hardness metric covering both planning and exploration parts of learning.
This metric explains why some common RL benchmarks are relatively easy in spite of having moderate hardness in terms of usual metrics such as problem size.
However, due to its abstractness, practitioners may find it difficult to turn this metric into any prediction of actual exploration behaviour of a learning agent in a given environment, or further into any practical advice on allocating resource for exploration.

In fact, the three groups of questions aforementioned can be better answered together rather than separately, because they are fundamentally interrelated.
The questions of situation analysis can be answered if we are able to predict what result can be expected for a given algorithm, executing under every reasonable parameter settings, on a given task for a certain time steps.
Being able to predict this directly provides us the answer to the questions of parameter setting, since what we need to do then is just to choose the parameter that yields best expected results within the time step budget.
The hardness questions can also be answered by comparing the expected goodness of results under the same setting of algorithm, its parameters, and the time step budget.

This paper proposes the \textit{success probability of exploration}, or success probability in short, as a unified answer to the three groups of questions.
The success probability of exploration is the probability $\psucc$ of an exploration strategy $\eprs$ under parameter setting $\theta$ yields a desired result $E$ on learning task $M$ at time step $\tau$.
We provide rigorous mathematical formulation of success probabilities, and present that knowing these probabilities is sufficient to answer all three groups of questions.

The success probability of exploration can be practically useful only if we are able to estimate it \textit{in prior} to executing the learning algorithm, or otherwise it will not free practitioners from tedious trial-and-error processes.
Therefore, we establish a concrete approach of assessing the success probability, by which its closed-form expression can be derived for certain prototype learning tasks.
In addition, we provide a practical approximation to the value of success probability, so that practitioners can use it in actual situations.
We then present our empirical results, which not only verify the correctness of our approach, but also display the high accuracy of our approximation.
Although our analyses are made mainly on the prototype tasks, we show that the results can be applied to a wider range of general domains and algorithms, which is also supported by the empirical results.

The rest of this paper is organized in three parts.
The first part consists of three sections.
In Section \ref{SectionPrelim}, we introduce the preliminary concepts of reinforcement learning that is relevant to this paper.
In Section \ref{SectionSPEBasic}, we formulate the success probability of exploration, compare it to the traditional PAC formulation, and provide its several elementary yet crucial properties.
In Section \ref{SectionSolve3Q}, we discuss how the success probability can be used to answer the three groups of questions related to exploration.

Then it comes to the second part, Sections \ref{SectionChainPerspective}, \ref{SectionSolvingSPE} and \ref{SectionApproxSPE}, where we elaborate our concrete approach to computing the success probability of exploration.
Specifically, in Section \ref{SectionChainPerspective}, we introduce the \textit{chain perspective}, which helps transform a general RL task to a more tractable one in the form of chain MDP.
In Section \ref{SectionSolvingSPE}, we derive the closed-form expression of success probability for a prototype exploration strategy running in chain MDPs.
Then in Section \ref{SectionApproxSPE}, we provide a practical approximation to the value of success probability.

The last part contains three sections as well.
In Section \ref{SectionExp}, we present our empirical results to justify our approach.
Readers may be more interested in Section \ref{SectionApplying}, where we demonstrate how our method can be applied to a general RL task, and then provide a short summary of our whole approach in the form of practice guide.
Finally in Section \ref{SectionDiscussion}, we discuss our approach at the macro level, and point out possible future researches relevant to this work.


\section{Preliminaries}
\label{SectionPrelim}

In this paper we follow the standard reinforcement learning framework in \citet{SuttonBarto98}, where an agent continuously interacts with a stochastic environment, learns its dynamic properties, and searches for the optimal policy that could lead to maximum expected cumulative rewards.
The environment here is formulated as a finite discounted Markov Decision Process (MDP) $M=(S, A, P, R, \gamma)$, where $S$ and $A$ are finite sets of states and actions respectively, $P$ is the transition function such that for all $s, s' \in S$ and $a \in A$, $P(s'|s,a):=\pr(s_{t+1}=s'|s_t=s,a_t=a)$ gives the fixed probability of state transition from state $s$ to $s'$ under action $a$ at arbitrary time $t$, $R:S \times A \times S \mapsto [{0, +\infty})$ is the reward function representing the numeric rewards of the transitions, and $\gamma \in (0,1)$ is a constant called the discount factor.

A (deterministic) policy $\pi: S \mapsto A$ maps each state to an action that should be taken by the agent when in that state.
The state value function of policy $\pi$, denoted by $V^\pi$, maps each state to the expected discounted cumulative reward the agent could get starting from that state and following policy $\pi$.
Given two arbitrary policies $\pi$ and $\pi'$, we write $V^\pi = V^{\pi'}$ to indicate that their state value functions satisfy $V^\pi(s) = V^{\pi'}(s)$ for all $s \in S$.

Let $\varPi$ denotes the set of all possible policies.
By Bellman equation, the following holds for any state $s \in S$ and any policy $\pi \in \varPi$:
\begin{equation}
\label{eqBellman1}
	V^\pi(s) = \sum_{s'\in S} P(s'|s,\pi(s))(R(s,\pi(s),s') + \gamma V^\pi(s')).
\end{equation}

A policy $\pi^* \in \varPi$ is called an optimal policy if $V^{\pi^*}(s) = \max_{\pi \in \varPi} V^\pi(s)$ holds for all $s \in S$, and the set of all optimal policies is denoted by $\varPi^*$.
There can be more than one optimal policies for some MDP.
However, for any given MDP, its optimal state value function $V^{\pi^*}$ is always unique by definition, and hence is often simply written as $V^*$.
The optimal state values satisfy Bellman optimality equation: for any state $s \in S$,
\begin{equation}
\label{eqBellman2}
	V^*(s) = \max_{a\in A} \sum_{s'\in S} P(s'|s,a)(R(s,a,s') + \gamma V^*(s')).
\end{equation}

If the environment is fully known by the agent, that is, the true transition function $P$ and reward function $R$ are given, then the optimal state values $V^*$ can be computed by planning algorithms designed based on Equations \ref{eqBellman1} or \ref{eqBellman2}.
Some popular planning algorithms, for example Value Iteration \citep{puterman1994markov}, have been proved that their calculated state values converge to the true optimal values in the limit, or to the near-optimal ones in polynomial time under some assumptions \citep{littman1995complexity}.

In more realistic reinforcement learning settings, the environment is initially unknown to the agent, and thus its dynamic properties must be estimated from the observations.
At each time step $t$, the agent receives an observation $(s_t, a_t, s_{t+1}, r_t)$ which represents the environment transiting from state $s_t$ to $s_{t+1}$ under the action $a_t$ while providing the agent an immediate reward of $r_t$. 
Obviously, the agent has $\tau$ observations $(s_1, a_1, s_{2}, r_1)$, $(s_2, a_2, s_{3}, r_2)$, ..., $(s_\tau, a_\tau, s_{\tau+1}, r_\tau)$ after $\tau$ steps.
This sequence of observations is called a trajectory (of length $\tau$) and is denoted by $\psi_\tau$.

We define the visit numbers of state-action pairs and transitions on the trajectory $\psi_\tau$ respectively as
\[
	N_{s,a}^{\psi_\tau} := \sum_{t=1}^\tau \mathbbm{1}(s_t = s)\mathbbm{1}(a_t = a), \;
\text{ and } \;
	N_{s,a,s'}^{\psi_\tau} := \sum_{t=1}^\tau \mathbbm{1}(s_t = s)\mathbbm{1}(a_t = a)\mathbbm{1}(s_{t+1} = s'),
\]
where $\mathbbm{1}(X)$ is the indicator function which equals 1 if expression $X$ is true and 0 otherwise.
Then the transition probabilities can be estimated straightforwardly by
$\hat{P}^{\psi_\tau}(s'|s,a) = N_{s,a,s'}^{\psi_\tau} / N_{s,a}^{\psi_\tau}$.
The immediate rewards, on the other hand, are deterministic with respect to the transitions, and thus estimating them is trivial (i.e. $\hat{R}^{\psi_\tau}(s,a,s') = r_t$ such that $(s_t,a_t,s_{t+1})=(s,a,s')$).
The resulting tuple $\hat{M}^{\psi_\tau} = (S, A, \hat{P}^{\psi_\tau}, \hat{R}^{\psi_\tau}, \gamma)$ is called an estimated model of the environment.

To simplify the notations, we remove $\psi_\tau$ from the visit numbers and estimated models and write them as $N_{s,a}$, $N_{s,a,s'}$, $\hat{P}$, $\hat{R}$, and $\hat{M}$ when there is no ambiguity from the context.
Further, we use $\allNsa$ and $\allNsas$ to collectively represent visit numbers of all state-action pairs and transitions.
In such cases, the subscripts do not refer to any specific state or actions.

A model-based reinforcement learning agent explicitly maintains an estimated model $\hat{M}$ of the true MDP $M$, and uses $\hat{M}$ instead of the unavailable $M$ as the input of its planning algorithm.
Intuitively, if more observations are used to estimate the transition probabilities, then the resulting $\hat{M}$ will probably be more accurate, and the output of the planning algorithm (optimal or near-optimal policy with respect to $\hat{M}$) will possibly get closer to the true ones with respect to $M$.
The model-free learning algorithms such as Temporal Difference and Q-Learning \citep{SuttonBarto98}, on the other hand, do not build models explicitly, but use Equations \ref{eqBellman1}, \ref{eqBellman2} or their modified versions to update the estimated values directly.
Still, they can be seen as model-based learning agents that utilize fast but degraded planning algorithms, and it is clear that they benefit from abundant observations just in the same way as model-based agents.

The key problem here is, the observations often come at a price, and in many real-world applications that involve interactions between concrete objects (e.g. Robotics), they can be even more expensive than the computational power.
Therefore, it is crucial that the agent choose the action wisely, avoiding unnecessary observations of the easily-estimated transitions (e.g. the deterministic ones) or the less relevant ones (for example, playing with a cat in a cooking task), and biasing to the more uncertain and relevant transitions. 
By doing so, more useful information can be gathered within fewer steps, and consequently high-quality policies are likely to be discovered earlier.

This bias or tactic is often referred as an \textit{exploration strategy}, denoted $\eprs$, which can be formulated as a function $\eprs(\psi_{t-1}, s_t)$ mapping from the current trajectory $\psi_{t-1}$ and state $s_t$ to an action $a_t$.
It can also be viewed as a non-stationary policy followed by the agent during learning that changes over time.
The exploration strategy provides the agent a useful heuristic regarding how many observations should be collected for each state-action pair, and in what order these state-action pairs should be visited.



\section{The Success Probability of Exploration}
\label{SectionSPEBasic}

As discussed in the last section, the observations usually don't come for free, and hence it is critical to know the relation between the cost, in terms of observations, and the outcome, namely the goodness of the policy derived from these observations.
In this section, we first formulate this cost-outcome relation through the success probability of exploration, then compare it to the PAC analysis, and finally highlight some elementary yet useful properties of the success probability of exploration.

\subsection{Formulating the Cost-outcome Relation}
The most straightforward representation of the cost in this context is the total number of observations.
It is the total number of time steps $\tau$ as well, since the agent receives exactly one observation at each time step.
Further, it is also the sum of visit numbers due to the fact that $\tau = \sum_{s,a} N_{s,a} = \sum_{s,a,s'} N_{s,a,s'}$.

In some more complicated settings, acquiring observations for certain state-action pairs might be more expensive than the others.
In this case, the cost can be represented as a weighted sum of visit numbers $\sum_{s,a} w_{s,a} N_{s,a}$.
However, this case can be transformed to the non-weighted one by augmenting the original MDP with sequences of trivial transitions corresponding to the weights.
For example, if the weight is 1 for all state-action pairs except for $(s_0, a_0)$, which has a weight of 4, then it is mathematically equivalent to the non-weighted case where $(s_0, a_0)$ always leads to some additional $(s_0', a_0')$, $(s_0'', a_0'')$, and $(s_0''', a_0''')$ before transiting to its original destination.
Therefore, it is sufficient to focus on the non-weighted case.

The formulation of the outcome should reflect the purpose of reinforcement learning, that is, looking for (near-)optimal policies.
Naturally, the outcome can be said to be desirable if and only if a learning algorithm outputs a (near-)optimal policy as its result of learning.
Therefore, we define the concept of \textit{success} used throughout this paper as follows.


\begin{definition}
\label{defEpsSuccess}
	A run of learning algorithm is said to be an \textbf{$\boldsymbol{\varepsilon}$-success} if and only if its output policy $\hat{\pi}^*$ satisfies $\hat{\pi}^* \in \varPi^{\varepsilon}$ where $\varepsilon$ is a non-negative number and $\varPi^{\varepsilon} := \{\pi|\forall s \in S,\, V^\pi(s) \geq V^*(s) - \varepsilon\}$. 
\end{definition}

\begin{remark}
\label{defSuccess}
	As a special case, when $\varepsilon=0$, the algorithm has to output an optimal policy $\hat{\pi}^* \in \varPi^*$ where $\varPi^{*} := \{\pi|\forall s \in S,\, V^\pi(s) = V^*(s)\}$ to achieve an $\varepsilon$-success.
	We say this kind of $\varepsilon$-success as a \textbf{strict success}.
\end{remark}

\begin{remark}
\label{defPiSuccess}
	Another kind of special case is that the algorithm achieves a success by exactly choosing one specific desired policy $\pi$ as its output.
	We say this kind of success as a \textbf{$\boldsymbol{\pi}$-success}.
\end{remark}

To simplify mathematical notation, we write $E^{\varepsilon}$ to represent an $\varepsilon$-success event, and write $E^*$ to represent a strict success event.
We also write $E^\pi$ to indicate a $\pi$-success event which ends up with the output being $\pi$.


Although many reinforcement learning algorithms are guaranteed to converge to optimal policies if all state-action pairs have been visited infinitely many times \citep{SuttonBarto98}, in reality the resources for acquiring observations are far less than infinite.
With a limited budget of observations, the estimates of transitions made by the agent can be inaccurate.
In that case, the algorithm can sometimes fails to discover a (near-)optimal policy, and hence the probability it achieves a success can be less than 1.
We call this probability the \textit{success probability of exploration}; the formal definition is as follows.

\begin{definition}
\label{defSPE}
	Let $M$ be an MDP, $\eprs$ be an exploration strategy, $\theta$ be its parameters, $\tau$ be a number of time steps, and $E^{\varepsilon}$ be an $\varepsilon$-success event.
	Then the \textbf{$\boldsymbol{\varepsilon}$-success probability of exploration} for $\eprs(\theta)$ running $\tau$ steps in $M$ is defined as: 
	\begin{equation*}
		\psucce_{M,\eprs(\theta),\tau} := \pr(E^{\varepsilon}|M,\eprs(\theta),\tau).
	\end{equation*}
\end{definition}

In the special case where $\varepsilon=0$, which corresponds to the case of strict success, we write $\psucc^*_{M,\eprs(\theta),\tau}$.
The difference between the near-optimal success and strict success is trivial from the mathematical perspective, since the choice of the value of $\varepsilon$ has impact only on the set of desired policies ($\varPi^\varepsilon$ or $\varPi^*$) in both cases.
Therefore, if the exact value of $\varepsilon$ is not important for the discussion, we can drop them from $\psucce_{M,\eprs(\theta),\tau}$ and write $\psucc_{M,\eprs(\theta),\tau}$ for convenience.
Additionally, if the subscripts are clear from the context, then we can drop them as well and simply write $\psucce_{\theta,\tau}$, $\psucce$, or $\psucc$.

\begin{remark}
For the special case of $\pi$-success, we write $\psucc_{M,\eprs(\theta),\tau}^\pi$ to represent the corresponding success probability, which equals to $\pr(E^\pi|M,\eprs(\theta),\tau)$.
In this case, the exact value of $\varepsilon$ is irrelevant because the desired policy has already been explicitly specified.
\end{remark}

The success probability of exploration defined above provide us a precise description of the cost-outcome relation we are interested.
The number of time steps $\tau$ represents the cost, the success event corresponds to the outcome, and the conditional probability $\psucc$ connects them together.
We put emphasis on exploration rather than planning here because it is the exploration strategy $\eprs(\theta)$ that decides the distribution of the observations $\allNsa$ and $\allNsas$, which is the unique source of information that the planning algorithms relies on to make plans.

\subsection{Comparison to the Traditional PAC Analysis}
\label{secCompTradPAC}
Our formulation is inspired by the notion of Probably Approximately Correct (PAC, \cite{valiant1984theory, fiechter1994efficient, kakade2003sample}) which tries to figure out how many observations are needed to be $\varepsilon$-optimal with probability at least $1-\delta$.
However, there are several key differences between our new formulation and the existing PAC notions when applied to reinforcement learning.



Firstly, the traditional PAC analysis provides results in the form of asymptotic bounds of sample complexity.
There lies a big gap between the best upper and lower bounds \citep{szita2010mormax, lattimore2014near} been discovered.
This means that either the upper bound is too loose, or the specific hard problem used to derive the lower bound is still not difficult enough.
Whichever the case, the current gap makes it difficult to definitely compare the efficiency between algorithms.
Therefore, a non-asymptotic approach of analysis can be more helpful in practice.
Our formulation and approach to the success probability of exploration, as can be seen in the later sections, are not based on asymptotic analysis.
By applying our approach, concrete relations between the cost and outcome can be obtained, which can be more convenient and functional compared to loose bounds.

Secondly, the traditional PAC analysis focuses on the theoretical sample complexity bounds for the most difficult MDP under a given setting of scale parameters ($|S|$, $|A|$, $\gamma$, etc.).
However, the actual MDPs most RL practitioners encounter in real-world applications may not be as difficult as the ones used for PAC analysis.
This leads to a paradoxical situation where, if one decide to set the exploration parameters according to the PAC theories, then the learning agent is very likely to over-explore as if it is in the most difficult MDP, resulting in poor actual performance despite its PAC guarantee \citep{kolter2009near,zhang2015increasingly}.
Our formulation directly addresses the success probability with respect to the given MDP, rather than to the hardest one, thus avoids this problem.

Thirdly, the sample complexity assessed in PAC-MDP analysis does not actually correspond to the total number of observations required to achieve an $(\varepsilon;\delta)$-PAC.
Instead, it corresponds to the number of locally non-$\varepsilon$-optimal steps, that is, the steps where the current policy has a non-$\varepsilon$-optimal value at the current state.
There can be arbitrarily many locally $\varepsilon$-optimal steps between any two successive locally non-$\varepsilon$-optimal steps.
In extreme cases, an agent with a polynomial sample complexity may still requires infinitely many steps to discover a globally $\varepsilon$-optimal policy (i.e. has $\varepsilon$-optimal values at all states).
This is not desirable if a locally $\varepsilon$-optimal step is as costly as a locally non-$\varepsilon$-optimal step, which is more likely to be the case in reality.
Therefore, it is crucial to have some theoretical result revealing the relation between the \textit{total number of steps}, including both locally $\varepsilon$-optimal and non-$\varepsilon$-optimal ones, and the outcome, which is exactly what we are trying to achieve in this paper.

Finally, the optimality discussed here is mathematically stronger than the ones in the literature concerning PAC analyses in RL.
The PAC optimality in \cite{fiechter1994efficient} refers to a local $\varepsilon$-optimality in the fixed start state, while in the PAC-MDP analyses \citep{kakade2003sample,strehl2009pac} it refers to a local $\varepsilon$-optimality along the states the agent actually visits during learning.
In our formulation, the output policy has to be $\varepsilon$-optimal in all states of the MDP in order to be $\varepsilon$-successful.
Therefore, an $\varepsilon$-success must be $\varepsilon$-optimal in the PAC framework of \cite{fiechter1994efficient} and \cite{kakade2003sample}, but the converse is not necessarily correct.

\subsection{Basic Properties of the Success Probability of Exploration}
\label{secBasicProperties}

To provide some more clear ideas to the readers about the success probability of exploration, we introduce some of its elementary yet useful properties in this subsection.
The first lemma below states the relationship between $\psucce$ and $\psucc^\pi$.

\begin{lemma}[First-level Decomposition]
	\label{lemL1Decomp}
	\begin{equation*}
	\psucce_{\theta,\tau} = \sum_{\pi \in \varPi^\varepsilon} \psucc_{\theta,\tau}^\pi.
	\end{equation*}
\end{lemma}
\begin{proof}
	Because the planning algorithms output only one policy at a time, success events $E^\pi$ and $E^{\pi'}$ for any $\pi \neq \pi'$ cannot happen at the same time.
	Therefore, we have $E^\varepsilon = \bigcup_{\pi \in \varPi^\varepsilon} E^\pi$, 
	and hence $\psucce_{\theta,\tau} 
	= \pr(\bigcup_{\pi \in \varPi^\varepsilon} E^\pi|\theta,\tau) 
	= \sum_{\pi \in \varPi^\varepsilon} \pr(E^\pi|\theta,\tau)
	= \sum_{\pi \in \varPi^\varepsilon} \psucc_{\theta,\tau}^\pi$.
\end{proof}

Lemma \ref{lemL1Decomp} actually states that the $\varepsilon$-success probability equals to the sum of  $\pi$-success probabilities ending up with each policy in $\varPi^\varepsilon$.
The most intriguing part of this lemma is, it shows that the $\varepsilon$-success probability $\psucce_{\theta,\tau}$ in only determined by the optimality parameter $\varepsilon$ in the set of relevant policies $\varPi^\varepsilon$.
Therefore, changing $\varPi^\varepsilon$ to any arbitrary set of interested policies does not make this lemma invalid, even if the policies in the new set share no $\varepsilon$-optimality, as stated in Remark \ref{remarkL1Decomp}.

\begin{remark}
	The set of $\varepsilon$-optimal policies $\varPi^\varepsilon$ can be changed to any subset of policies $\tilde{\varPi} \subseteq \varPi$ and the first-level decomposition still holds, that is, 
	\begin{equation*}
		\psucc^{\tilde{\varPi}}_{\theta,\tau} = \sum_{\pi \in \tilde{\varPi}} \psucc_{\theta,\tau}^\pi.
	\end{equation*}
	\label{remarkL1Decomp}
\end{remark}

This can be particularly convenient in real-world applications, because it is often impossible to figure out what the $\varepsilon$-optimality is for the desired policies. 
Actually, it is often the case that RL practitioners have a set of desired policies in their mind, but cannot specify what $\varepsilon$ should be to describe this set.
By Lemma \ref{lemL1Decomp}, the practitioners only need to work out all $\pi$-success probabilities for the desired policies, and add them together to obtain the overall success probability, avoiding the necessity of specifying $\varepsilon$.
Nevertheless, we will continue using $\varPi^\varepsilon$ to refer to the set of desired policies in the rest of this paper.

To determine the $\pi$-success probabilities, we need to know how exactly the planning algorithm chooses its output policy.
Therefore, the following assumptions about the algorithm discussed in this paper has to be introduced in order to specify its behaviour.

\begin{assumption}
	\label{assPlan}
	The planning algorithm outputs exactly one policy $\pi$ that satisfies $\hat{V}^\pi = \hat{V}^*$, where $\hat{V}^*(s) = \max_{\pi \in \varPi} \hat{V}^{\pi}(s)$ for any $s$.
	If there is more than one policy satisfying $\hat{V}^\pi = \hat{V}^*$, the planning algorithm randomly and uniformly choose one of them as its output.
\end{assumption}

\begin{assumption}
	If a state-action pair is never visited, then the output policy of the planning algorithm will not contain that state-action pair. 
\end{assumption}

The first assumption is rather mild and reasonable because it actually holds for most popular planning algorithms.
The second one is also reasonable because without any information about a state-action pair, the agent should have no idea how good it is, and hence should not choose it if it has other better options.
This assumption is violated when generalization or prior knowledge is involved in planning.
By utilizing such techniques and knowledge, the algorithm may have non-trivial estimates to the values of the unvisited state-action pairs. 
This issue will be dealt with in Section \ref{SectionChainPerspective}, where we propose the chain perspective to include the effect of generalization into the representation of environment.

Now it is clear that in order to obtain a $\pi$-success, the agent must visit all of the state-action pairs that are related to $\pi$ at least once.
This leads to the lemma which we call \textit{the second-level decomposition} of success probability as below.

\begin{definition}
	\label{defTrav}
	A \textbf{traverse event of $\boldsymbol{\pi}$}, denoted $E^\pi_\text{trav}$, is said to have occurred if and only if $N_{s, \pi(s)} \geq 1$ holds for all $s \in S$.
\end{definition}


\begin{lemma}[Second-level decomposition]
	\label{lemL2Decomp}
	\begin{equation*}
		\psucc^{\pi}_{\theta,\tau} = 
		\pr(E^{\pi}|\Etrav^{\pi},\theta,\tau) \,
		\pr(\Etrav^{\pi}|\theta,\tau).
	\end{equation*}
\end{lemma}
\begin{proof}
	By the law of total probability, it holds that 
	$\psucc^{\pi}_{\theta,\tau} = 
	\pr(E^{\pi}|\Etrav^{\pi},\theta,\tau)
	\pr(\Etrav^{\pi}|\theta,\tau) + 
	\pr(E^{\pi}|\neg \Etrav^{\pi},\theta,\tau)
	\pr(\neg \Etrav^{\pi}|\theta,\tau)$,
	where $\neg \Etrav^{\pi}$ represents that the traverse event of $\pi$ is not happened.
	However, if the traverse event of $\pi$ is not happened, then there exists some state $s$ such that $N_{s, \pi(s)} = 0$, and hence the action of the output policy at state $s$ will never be $\pi(s)$.
	This leads to $\pr(E^{\pi}|\neg \Etrav^{\pi},\theta,\tau) = 0$, and hence the theorem.
\end{proof}

Lemma \ref{lemL2Decomp} allows us to compute the probability of successful traverse events and the conditional success probability separately, then combine them together to get the unconditional success probability.
As can be seen in later sections, this can be very helpful because the conditional success probability is easier to work out than the unconditional one.

Additionally, this lemma provides an intuition about why the optimism principle proposed by \cite{kaelbling1996reinforcement} is so broadly accepted in designing the exploration strategies.
The optimism principle requires every state-action pair be tried for some times before any conclusion on its utility is made.
This often leads to the occurrence of traverse events, which enables the possibility of success, and thus plays a key role in theoretically guaranteeing the performance of learning algorithms, even if such a success is not the exact objective for these algorithms.

In contrast, some naive strategies without optimism, $\varepsilon$-greedy for example, are more likely to fail in traversing, which result in further failures in planning high-quality policies due to lack of key information.
Although non-optimistic strategies are widely utilized in real-world applications (e.g. \cite{abbeel2005exploration, riedmiller2009reinforcement, mnih2015human}), their successes are often more dependent on the state/action feature engineering, prior knowledge, the (near-)deterministic environment, and generalization techniques.
In less deterministic environments with less prior knowledge available to both machines and RL practitioners, the impact of insufficient exploration becomes more significant, worsening the performance of vanilla strategies. 

It is worth noting that, although literally traversing an MDP can be impractical in large-scale real-world environments due to state explosion, the discussion for traverse events remains essential and reasonable for analyses.
To deal with large-scale (or even continuous) MDPs, we will introduce \textit{the chain perspective} in Section \ref{SectionChainPerspective}.
In short, the chain perspective reduce an arbitrary MDP to a chain MDP with much smaller scale, such that the success probability of exploration in the original MDP can be approximated by that in the reduced chain MDP.
It is sufficient to traverse the reduced chain MDP, rather than the original one, for a learning agent to achieve a success in our analyses.
This properly reflects the fact that, with generalization techniques applied, learning agents are able to discover desirable policies if all key states are sufficiently explored.
Traversing the reduced chain MDPs are both realistic and sensible milestones to success for agents, and therefore the discussion of traverse events are necessary even in realistic scenarios.
We will come back to this topic in later sections.

Combining the first-level and the second-level decomposition lemmas, we get the following theorem.
\begin{theorem}[Divide-and-conquer the success probability]
	\begin{equation*}
		\psucce_{\theta,\tau} = 
		\sum_{\pi \in \varPi^\varepsilon} \pr(E^{\pi}|\Etrav^{\pi},\theta,\tau) \, \pr(\Etrav^{\pi}|\theta,\tau).
	\end{equation*}
	\label{theoDecomp}
\end{theorem}
\begin{proof}
	By Lemma \ref{lemL1Decomp} and Lemma \ref{lemL2Decomp}.
\end{proof}

By Theorem \ref{theoDecomp}, the success probability can be dealt with in a divide-and-conquer manner.
In this way, the necessity of analysing a complex event as a whole is avoided, turning the complicated task of deriving the expression of success probability to a more tractable one.

\section{The Solution to the Three Groups of Questions of Exploration}
\label{SectionSolve3Q}
Having established a new perspective of the cost-outcome relation, it is crucial to examine whether this new perspective is able to reflect our practical needs.
In the following subsections, we demonstrate that our success probability of exploration can be used to answer the three groups of questions mentioned in Section \ref{sec3Q}.
To focus on this purpose, we leave the details of our approach to concretely estimating the success probability of exploration to later sections.
For now, let us just assume that we have already derived the closed-form expression of the $\psucc_{\theta,\tau}$ for MDP $M$ and exploration strategy $\eprs$. 

\subsection{(Q1) Questions of Exploration Parameter Setting}
\label{secSolveQ1}
The first group of questions concern how to find the most suitable parameter setting $\theta$ for the given learning task such that minimum number of observations $\tau$ are required to guarantee the failure probability $(1-\psucc_{\theta,\tau})$ not exceeding some threshold $\delta$.
It corresponds to the question that the traditional PAC analysis tries to answer (``With what parameters the algorithm can be PAC with polynomial sample complexity?'') but have not come with a satisfying result (see Section \ref{secCompTradPAC}).
Therefore, the success probability of exploration ought to be more productive in answering these questions.

More formally, we are interested in the best parameter setting $\theta_\delta$ such that
\begin{equation}
\label{eqThetaDelta}
	\theta_\delta = \argmin_\theta \tau: 1-\psucc_{\theta,\tau} \leq \delta.
\end{equation}
The success probability $\psucc_{\theta,\tau}$ here can be seen as a function of $\tau$ and write as $\psucc_\theta(\tau)$.
Intuitively, if more observations are obtained and are used to build the model $\hat{M}$, then $\hat{M}$ is more likely to be accurate, and because its estimates of transitions are unbiased, the success probability should increase in most cases.
In the ideal case where $\psucc_\theta(\tau)$ strictly increases with $\tau$ and thus has an inverse function $\psucc^{-1}_\theta(\delta)$, we have
\begin{align*}
	\theta_\delta 
	&= \argmin_\theta \tau: \psucc_\theta(\tau) \geq 1 -\delta \\
	&= \argmin_\theta \tau: \tau \geq \psucc^{-1}_\theta(1-\delta) \\
	&= \argmin_\theta \psucc^{-1}_\theta(1-\delta),
\end{align*}
and thus the best parameter setting can be discovered directly by calculating the minimum point of the function above.
Even if $\psucc_\theta(\tau)$ is not strictly monotonic, Equation \ref{eqThetaDelta} can still be view as a constrained optimization problem and be solved via  function optimisation techniques.
By doing so, the best parameter setting of $\eprs$ for achieving ($\varepsilon,\delta$)-optimality can be discovered without trial-and-error. 
In this way, the actual budget of observations required for training the agent can be significantly reduced, at the price of merely some numerical computations.
This can be precious for real-world applications where observations are far more expensive than computational power.

In conclusion, knowing the closed-form expression of the success probability is sufficient to answer the questions of exploration parameter setting.

\subsection{(Q2) Questions of Situation Analysis}

Suppose that we have run the learning algorithm for a while, but have not achieved a success.
We then have to conduct situation analysis, i.e. evaluate the learning process so far to figure out whether we have to increase the sample size, change the parameters of the exploration strategy, or even improve the representation of states and actions.

The need of situation analysis arises even if a practitioner chooses the best parameter setting with respective to ($\varepsilon,\delta$)-optimality using the method introduced above.
The reason is that parameter setting only guarantees a probability $1-\delta$ of success, and there still exist chances of failure that cannot be ignored.
When a failure occurs, the practitioner has to conduct situation analysis to decide what should be done with the information already obtained as well as the remaining resource budgets.
Although the need of situation analysis can be removed by controlling the failure probability to a tiny value, this often leads to a greater need of time steps $\tau$ for success, or a worse $\varepsilon$-optimality.
Therefore, it is reasonable to trade off some success probability for a better optimality and a less need of time steps, and rely on situation analysis to cope with undesirable consequences in a iterative manner.

Comparing the success probability of exploration under different settings leads us to a clear solution to the situation analysis.
Specifically, given the parameter setting $\theta$ and the time step $\tau$, the possible current situations and their corresponding solutions are as follows:
\begin{enumerate}[label={(\alph*)}]
	\item
	\begin{description}[align=left,noitemsep]
		\item [Situation] The success probability $\psucc_{\theta,\tau}$ is sufficiently high. 
		\item [Solution] The current failure is probably due to bad luck (e.g. some rare events happened in a row). 
		Therefore, re-run the learning algorithm with the same setting of $\theta$ and $\tau$ should be sufficient.
	\end{description}
	\item
	\begin{description}[align=left,noitemsep]
		\item [Situation] The success probability $\psucc_{\theta,\tau}$ is not high enough. 
		However, if we increase $\tau$ to a higher yet reasonable value $\tau'$, then $\psucc_{\theta,\tau'}$ becomes acceptable.
		\item [Solution] The observations are insufficient. 
		If there is still enough time step budget, then we should keep the parameter setting $\theta$ unchanged, continue learning and wait for the agent to collect more observations.
	\end{description}
	\item
	\begin{description}[align=left,noitemsep]
		\item [Situation] The success probability $\psucc_{\theta,\tau}$ is not high enough. 
		However, if we change the parameter setting $\theta$ to some reasonable $\theta'$, then $\psucc_{\theta',\tau}$ notably improves.
		\item [Solution] The current parameter setting of the exploration strategy is highly risky, and we were not lucky enough to achieve a success. 
		We should change the parameters to $\theta'$ for higher success probability and re-run the learning algorithm for $\tau$ steps.
	\end{description}
	\item
	\begin{description}[align=left,noitemsep]
		\item [Situation] The success probability $\psucc_{\theta,\tau}$ is not high enough. 
		Furthermore, no matter how we change the value of  $\theta$ and $\tau$ within reasonable range, the success probability remains unacceptable.
		\item [Solution] It is too risky to expect the learning succeeds under the current problem formulation or representation.
		Unless there is sufficient time step budget to bear the risk, we may have to examine the problem formulation or the state/action representation and try to improve them.
	\end{description}
\end{enumerate}

As we can see above, the success probability of exploration is helpful for assessing the present situation of the learning process, identifying the critical factors that may influence the performance and providing the possible solutions accordingly.
Additionally, this method of analysis is not based on the learning curves, and hence does not require to keep track of the performance of the output policies at every time step, avoiding the potential problems mentioned in Section \ref{sec3Q}.

\subsection{(Q3) Questions of Hardness of Exploration}
\label{secSolveQ3}
The hardness of reinforcement learning tasks is not obvious, especially for the exploration part since it involves uncertainty.
Literature have shown that some benchmarks considered non-trivial are actually relatively easy \citep{maillard2014hard}, while some seemingly impossible tasks are solved efficiently by rather simple algorithms (e.g. \cite{mnih2015human,silver2016mastering}).
Therefore, it is crucial to develop some metrics that can be used to compare the hardness of exploration in different MDPs.
Further, this metric should be directly related to the cost of learning so that practitioners can make their budget accordingly.

The success probability of exploration reflects the internal properties of MDPs as well as the interaction between exploration strategies and MDPs, and thus provides effective solution to these questions.
Given two MDPs $M_1$ and $M_2$, their success probabilities $\psucc_{M_1,\eprs(\theta),\tau}$ and $\psucc_{M_2,\eprs(\theta),\tau}$ using some baseline exploration strategy $\eprs$ can be compared to decide their relative hardness.
For example, fixing the exploration parameter setting $\theta$, if we find that for any number of steps $\tau>0$ it holds that $\psucc_{M_1,\eprs(\theta),\tau} > \psucc_{M_2,\eprs(\theta),\tau}$, then undoubtedly the MDP $M_1$ is easier than $M_2$ for exploration strategy $\eprs(\theta)$.
Of course, comparing the success probabilities under all $\tau>0$ is not always necessary for practical purpose, and it may suffices to compare them within some reasonable budget range $\tau_\text{min} < \tau < \tau_\text{max}$.

Alternatively, given some small failure probability $\delta$, say 0.05, we can compare the minimum numbers of steps $\tau_1$ and $\tau_2$ required to satisfy $\psucc \geq (1-\delta) = 0.95$, respectively in $M_1$ and $M_2$.
Here $\tau_1$ and $\tau_2$ can be computed by function optimization methods as in Section \ref{secSolveQ1}, namely by solving the minimization problems
\[
	\tau_1 = \min_\theta \tau : \psucc_{M_1,\eprs(\theta)}(\tau) \geq 1-\delta, \,
	\text{ and } \,
	\tau_2 = \min_\theta \tau : \psucc_{M_2,\eprs(\theta)}(\tau) \geq 1-\delta,
\]
and for strictly increasing functions, they can be simplified to  
\[
	\tau_1 = \min_\theta \psucc^{-1}_{M_1,\eprs(\theta)} (1-\delta), \,
	\text{ and } \,
	\tau_2 = \min_\theta \psucc^{-1}_{M_2,\eprs(\theta)} (1-\delta).
\]
Comparing $\tau_1$ and $\tau_2$ reflects, under the best parameters settings of $\eprs$ respectively for $M_1$ and $M_2$, which MDP needs more steps to explore.

Furthermore, we can draw the whole $\theta$-$\tau$-$\psucc_{\theta,\tau}$ surface for $M_1$ and $M_2$, and see whether one of them is above the other (i.e. $\psucc_{M_1,\eprs(\theta),\tau} \geq \psucc_{M_2,\eprs(\theta),\tau}$ for any $\theta$ and $\tau$).
If the two surfaces share the same analytic expression except for some parameters, it is possible that just comparing these parameters is sufficient to decide the relative hardness of these MDPs.
These parameters may even give a clue about some essential properties of these MDPs, leading to further discovery and deeper understanding of the tasks.

In conclusion, the success probability of exploration is helpful for describing and discerning the hardness of exploration over different MDPs, and thus provide possible answers the third group of questions related to exploration.

%
%
%

\section{The Chain Perspective of MDP}
\label{SectionChainPerspective}
Last section has shown that the success probability of exploration does provide useful answers to all three groups of questions of exploration, and obtaining the closed-form expression of the success probability is the key.
Working out the closed-form expression itself, on the other hand, is not a trivial task since it involves complicated interactions between MDPs and the learning algorithms.
To avoid falling into the trap of ad hoc analysis where one has to analyse every MDP from scratch, we introduce the \textit{chain perspective of MDP}, which helps form a generalizable foundation for the later discussion.

In short, the chain perspective of MDP is to abstract a more complicated MDP as a chain MDP, which is composed of several key elements as follows:

\begin{description}
	\item [Start state] There is one unique start state $s_1$ in a chain, and the learning agent always begins the interaction with the environment from this state.
	
	\item [Goal state \& Goal reward] There is one unique goal state $s_n$ at the end of the chain, where $n$ is the length of the chain, or, the number of chain states. 
	There is a positive goal reward $r_G$ for taking the right action at the goal state.
	It is set large enough such that the optimal policy should always tries to obtain it as frequently as possible.
	
	\item [Other chain states] The main body of a chain MDP is the states $s_2, ..., s_{n-1}$ connecting the start state $s_1$ and the goal state $s_n$.
	
	
	\item [Forward actions] At each chain state $s_i$ except for the goal state, there is at least one forward action by which the agent is most likely to travel to the next state $s_{i+1}$ and thus getting \textit{closer to} the goal state.
	
	\item[Goal action] This is the only action that collects the goal reward at the goal state $s_n$.
	
	\item [Backward actions \& Distracting rewards] At each state $s_i$, there is one or more backward actions by which the agent is more likely to travel back to state $s_1, ..., s_{i-1}$, getting \textit{away from} the goal state.
	A learning agent might be confused by these possible actions because they often yield or lead to some distracting rewards $r_D$, which might seem to be more appealing than the goal reward for the agent under uncertainty.
\end{description}

\begin{figure}
	\centering
	\begin{tikzpicture}[
		auto, ->, >=latex,
		thick, 
		font=\scriptsize,
		main node/.style={circle,draw,font=\small},
		special node/.style={circle,draw,font=\small, line width=2pt},
		b reset/.style={densely dashed, bend left = 33, max distance = 40},
		b start/.style={densely dashed, in = 160, out = 205, looseness = 8, left = 1},
		a fail/.style={in = 105, out = 60, looseness = 8, above = 1},
		a goal/.style={in = -20, out = 25, looseness = 8, right = 1},
	]

		\foreach [evaluate=\x as \pos using 2.5*\x-1] \x in {1, 5}
		{
			\node[special node] (\x) at (\pos,0) {$\x$};
		}
		\foreach [evaluate=\x as \pos using 2.5*\x-1] \x in {2,...,4}
		{
			\node[main node] (\x) at (\pos,0) {$\x$};
		}

		\foreach \x/\y in {1/2,2/3,3/4,4/5}
		{
			\path (\x) 	edge [bend left = 20, looseness = 0.8] node {$a^+, 0.8, 0$} (\y)
						edge [a fail] node {$a^+, 0.2, 0$} (\x);
		}
		\foreach  [evaluate=\x as \bent using 0.7-0.07*\x] \x in {2,...,4}
		{
			\path (\x)  edge [b reset, looseness = \bent] node {} (1);
		}
		\path (5) edge [b reset] node {$a^-, 1, 0$} (1);
		\path (1) edge [b start] node {$a^-, 1, +0.1$} (1);
		\path (5) edge [a goal] node {$a^+, 1, +1$} (5);
		
	\end{tikzpicture}
	\caption{A typical chain MDP used as benchmark in the literature.}
	\label{figChainExample}
\end{figure}
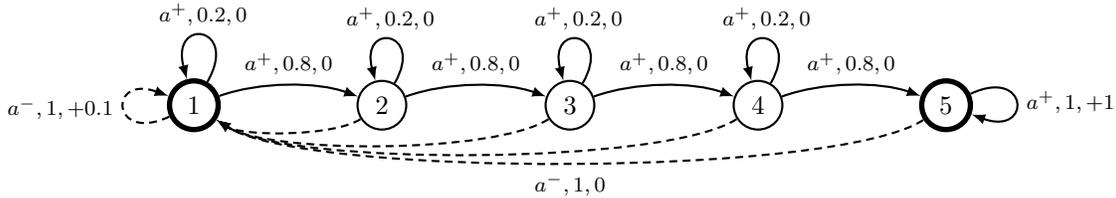

A typical chain MDP that can be easily found as a benchmark problem in various literature (e.g. \cite{dearden1998q,strehl2004empirical,kolter2009near}) is shown in Figure \ref{figChainExample}. 
There are five chain states in this example, and the start state $s_1$ and the goal state $s_5$ are highlighted using bold circles.
The triple $(a,p,r)$ written next to a transition indicates that this transition occurs with probability $p$ under action $a$ and yields reward $r$.
For clarity and convenience, all forward actions as well as the goal action are written as $a^+$ and drawn in solid arrows, while all backward ones are written as $a^-$ and drawn in dashed arrows.

As can be seen from the figure, forward actions at state $s_1$ to $s_4$ have a $p=0.8$ chance of sending the agent to the next state, and a $(1-p)=0.2$ chance of remaining at the current state.
In contrast, all backward actions transit the agent back to the start state $s_1$ with probability 1.
The only non-zero distracting reward $r_{1,a^-} = 0.1$ in this example is produced by taking backward action $a^-$ at the start state.
Although this reward is small compared to the goal reward $r_{5,a^+} = 1$, it is still distracting for those agents with pessimistic estimates to the probability $p_i$'s of moving closer to the goal by taking forward actions.
In addition, if an agent has never collected the goal reward before, then staying at the start state and taking action $a^-$ may appear to be optimal, in spite that the real optimal policy is to take action $a^+$ at every state $s\in S$, given that the discounting factor $\gamma$ is sufficiently large.

Chain MDPs have played a key role in analysing the efficiency of reinforcement learning algorithms due to the simplicity of its structure.
In \cite{whitehead1991q}, it has been proved that in a homogeneous problem solving task, the expected number of observations required by a Q-Learning agent with an $\varepsilon$-greedy exploration strategy to find an optimal policy is exponential to the number of steps required by the optimal policy to arrive at the goal state.
Another expression of this theorem, proposed by \cite{li2012sample}, is as follows: in a chain MDP, a Q-Learning agent with an $\varepsilon$-greedy exploration strategy, starting from the start state, needs observations exponential to the length of the chain in order to reach the goal state. 
This theorem indicates the inefficiency of $\varepsilon$-greedy exploration strategy, and inspired a deeper investigation to the exploration strategies thereafter.

While the simplicity of chain MDPs serves well for the purpose of theoretical analysis, they also contain most of the fundamental elements necessary to describe complex real-world tasks.
Specifically, there is always some kind of objective for the agent to achieve in a task, which can be represented by a goal state.
In order to achieve the objective, the agent has to select the right actions throughout a series of states, without wrongly taking hazardous actions or being distracted by irrelevant rewards.

\begin{figure}
	\centering
	\begin{subfigure}[b]{1\textwidth}
		\centering
		\includegraphics[scale=0.6]{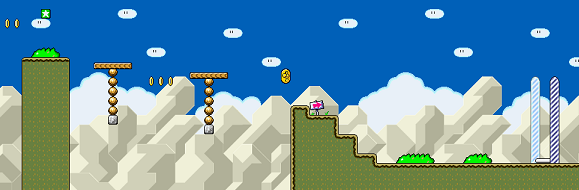}
		\caption{}
	\end{subfigure}
	\\
	\begin{subfigure}[b]{1\textwidth}
		\centering
		\begin{tikzpicture}[
			auto, ->, >=latex,
			thick, 
			font=\scriptsize,
			main node/.style={circle,draw,font=\small},
			special node/.style={circle,draw,font=\small, line width=2pt},
			b reset/.style={densely dashed, bend left = 33, max distance = 40},
			b start/.style={densely dashed, in = 160, out = 205, looseness = 8, left = 1},
			a fail/.style={in = 100, out = 65, looseness = 8, above = 1},
			a goal/.style={in = -20, out = 25, looseness = 8, right = 1},
			a goal new/.style={bend left = 80, looseness = 0.9,  max distance = 26},
		]
	
			\foreach [evaluate=\x as \pos using 1.8*\x-1] \x in {1, 7}
			{
				\node[special node] (\x) at (\pos,0) {$\x$};
			}
			\foreach [evaluate=\x as \pos using 1.8*\x-1] \x in {2,...,6}
			{
				\node[main node] (\x) at (\pos,0) {$\x$};
			}
	
			\foreach \x/\y in {1/2,2/3,3/4,4/5,5/6,6/7}
			{
				\path (\x) 	edge [bend left = 20, looseness = 0.8] node {$a^+$} (\y);
			}
			\foreach  [evaluate=\x as \bent using 0.7-0.07*\x] \x in {2,...,3}
			{
				\path (\x)  edge [b reset, looseness = \bent] node {} (1);
			}
			\path (4)  edge [b reset, looseness = 0.4] node {$a^-$} (1);
			\foreach  [evaluate=\x as \bent using 0.7-0.07*\x] \x/\y in {5/4, 6/5, 7/6}
			{
				\path (\x)  edge [densely dashed, bend left = 20, looseness = 0.8] node {$a^-$} (\y);
			}
			\path (1) edge [b start] node {$a^-$} (1);
			\path (7) edge [a goal, transparent] node {$aa$} (7);
			\path (7) edge [a goal new] node {$a^+, r_G$} (1);
			
		\end{tikzpicture}
		\caption{}
	\end{subfigure}
	\caption{(a) A part of the game stage in Super Mario World by Nintendo. 
	(Picture from http://www.vgthought.com/article/smw\_frames\_and\_themes)
	(b) Its chain MDP abstraction.}
	\label{figMario}
\end{figure}
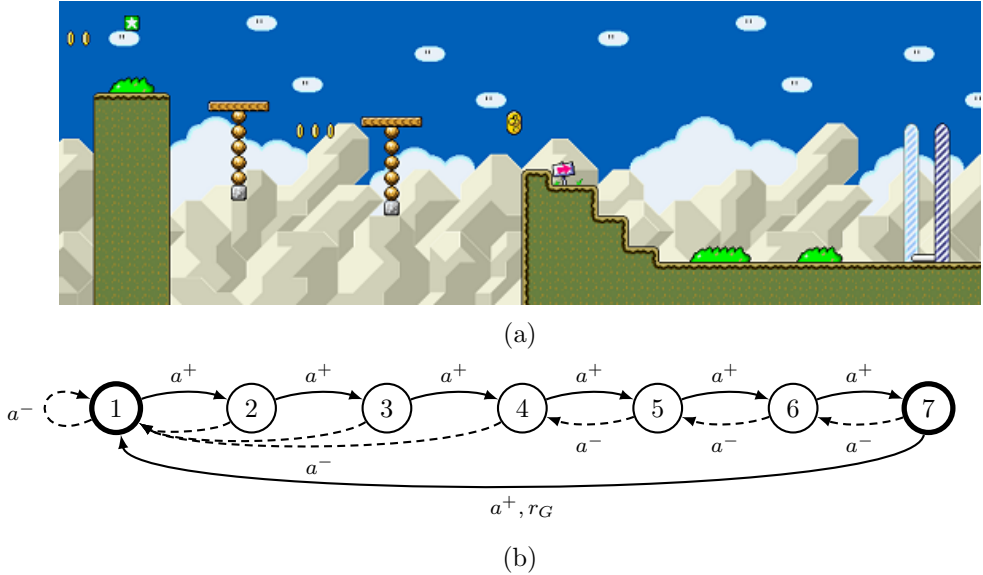

An illustration of this is provided in Figure \ref{figMario}.
In the game stage shown in (a), Mario has to make some successful jumps from the leftmost platform to the rightmost one without falling into the pit, and then run through the goal sign.
One possible chain abstraction for this is (b), where the agent starts from the state numbered 1 corresponding to the leftmost platform, travels through the chain states 2 to 6 by forward actions, eventually reaches the state 7 indicating the goal sign, and takes the goal action to collect the goal reward.

Although much details of the original game stage are lost in the chain abstraction, it nevertheless captures the main dynamic properties of the original stage.
In particular, the hardness of exploring the game stage is properly reflected in its chain abstraction.
Taking wrong actions in the first half of the stage (the short platforms and the cliff) will cause Mario to fall into the pit and die, which means Mario has to re-challenge the game stage from the start point.
This hazardousness is appropriately expressed in state $s_1$ to $s_4$ since the backward actions at these states send the agent back to $s_1$.
The second half of the stage, on the other hand, are less difficult because taking wrong actions only bring Mario back to the previous positions without killing him.
Exploring $s_5$ to $s_7$ in the chain abstraction is correspondingly easier than $s_1$ to $s_4$. 
Finally, arriving the goal point indicates a ``level complete'', resulting in Mario being sent to the next stage.
If the player wants to continue exploring the current stage, Mario has to restart from the start state.
This is correctly represented in the chain abstraction as well.

The most important aspect of the chain abstraction is, if the effect of generalization techniques applied coincide the chain abstraction, i.e. the underlying policy sees the environment exactly as the reduced chain, then the success probability of the learning algorithm executed on the original environment is equivalent to that of an algorithm without generalization executed on the chain MDP.
In the Mario example above, this can be done by create seven grids on the horizontal axis accordingly, and represent the state by the number of the grid where Mario is currently in.

In this way, the effect of generalization is completely encapsulated into the representation of the environment.
Not only this removes the necessity of simultaneously considering generalization and exploration in our analyses, it also provides useful hints regarding which kind of generalization can be more efficient.
Specifically, if some generalization technique reduce the original learning task on MDP $M$ to a chain MDP $M'$, and learning on $M'$ has a sufficiently high success probability within reasonable resource budget, then that generalization technique is very likely to be an efficient one.


\section{Solving the Success Probability in Chain MDPs}
\label{SectionSolvingSPE}
In this section, we elaborate our approach to the closed-form expression of success probability in chain MDPs.

\subsection{The Dependency Graph of Relevant Variables}
Considering that the success probability involves almost everything in reinforcement learning, it is crucial to figure out the interrelationship of all relevant variables as a preliminary step.
The resulting \textit{dependency graph} is shown in Figure \ref{figDependency}.
If there is a directed edge from node $X$ to $Y$, then it means that the variable $X$ is dependent to variable $Y$.
If some variable is written in angle brackets, then it represents a tuple of all variables of the same kind.
For example, $\langle E^\pi \rangle$ denotes the tuple of $E^\pi$ of all possible policies $\pi \in \varPi$.

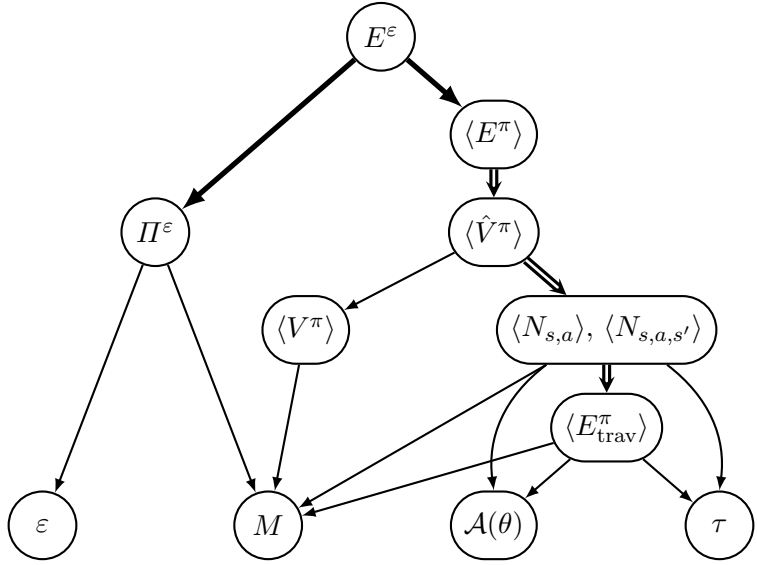
\begin{figure}
	\centering
	\begin{tikzpicture}[
		auto, <-, >=latex,
		thick, 
		font=\normalsize,
		main node/.style={draw, rounded rectangle, minimum size=9mm},
		1st/.style={line width=2pt},
		2nd/.style={very thick, double distance=1pt, >=stealth},
	]

		\node[main node] (eps) at (0,0) {$\varepsilon$};
		\node[main node] (M) at (3,0) {$M$};
		\node[main node] (algo) at (6,0) {$\eprs(\theta)$};
		\node[main node] (tau) at (9,0) {$\tau$};
		
		\node[main node] (Etrav) at (7.5,1.3) {$\langle \Etrav^\pi \rangle$};
		
		\node[main node] (Nsa) at (7.5,2.6) {$\allNsa$, $\allNsas$};
		\node[main node] (Vpi) at (3.5,2.6) {$\langle V^\pi \rangle$};
		
		\node[main node] (pie) at (1.5,3.9) {$\varPi^\varepsilon$};
		\node[main node] (hatVpi) at (6,3.9) {$\langle \hat{V}^\pi \rangle$};
		\node[main node] (Epi) at (6,5.2) {$\langle E^\pi \rangle$};
		
		\node[main node] (Eeps) at (4.5,6.5) {$E^\varepsilon$};
				
		\path (eps) edge [] node {} (pie);
		\path (M) edge [] node {} (pie);
		\path (M) edge [] node {} (Vpi);
		\path (M) edge [] node {} (Nsa);
		\path (algo) edge [bend left] node {} (Nsa);
		\path (algo) edge [] node {} (Etrav);
		\path (tau) edge [] node {} (Etrav);
		\path (M) edge [] node {} (Etrav);
		\path (tau) edge [bend right] node {} (Nsa);
		\path (Etrav) edge [2nd] node {} (Nsa);
		\path (Nsa) edge [2nd] node {} (hatVpi);
		\path (Vpi) edge [] node {} (hatVpi);
		\path (hatVpi) edge [2nd] node {} (Epi);
		\path (Epi) edge [1st] node {} (Eeps);
		\path (pie) edge [1st] node {} (Eeps);
		
	\end{tikzpicture}
	\caption{Dependency graph of $\varepsilon$-success event.}
	\label{figDependency}
\end{figure}

The thick arrows from $\varepsilon$-success event $E^\varepsilon$ to the set of $\varepsilon$-optimal policies $\varPi^\varepsilon$ and the tuple of all $\pi$-success events $\langle E^\pi \rangle$ correspond to Lemma \ref{lemL1Decomp}, the first-level decomposition of $\varepsilon$-success probability, in which the probability of $E^\varepsilon$ equals to the sum of probabilities of $E^\pi$ over all $\pi \in \varPi^\varepsilon$.

The double arrows connecting $\langle E^\pi \rangle$, estimated values $\langle \hat{V}^\pi \rangle$, visit numbers $\allNsa$ and $\allNsas$, and traverse events $\langle \Etrav^\pi \rangle$ correspond to Lemma \ref{lemL2Decomp}, the second-level decomposition.
The occurrence of traverse events will be considered first.
Only when it occurs, all relevant visit numbers will be non-zero, so that they can be used to estimate the state values.
The estimation process (or the planning process) also involves the expressions of state values $\langle V^\pi \rangle$, which are usually derived implicitly from Bellman Equations.
The estimated state values are thereafter compared with each other, deciding which $\pi$-success event occurs.

The detailed analysis of $\psucc_{M_C, \ops(m), \tau}$ will be carried out in a bottom-up manner according to the dependency graph in the following sections.
We start from the in-depth discussions of chain MDPs and exploration strategies respectively in Sections \ref{secProtoChains} and \ref{secProtoStrategy}.
They are followed by $\langle \Etrav^\pi \rangle$ in Section \ref{secClosedTrav}, then $\allNsa$ in Section \ref{secClosedVisitNum}.
Then it comes to be the actual and estimated value functions $\allsth{V^\pi}$ and $\allsth{\hat{V}^\pi}$ together in Section \ref{secClosedValueFunc}.
Finally, the probability of $\pi$-success and $\varepsilon$-success events will be assembled in Section \ref{secClosedSPE}.

\subsection{The Prototypes of Chain MDPs}
\label{secProtoChains}
As hinted by the dependency graph, there are four critical factors involved in the success probability of exploration: degree of near-optimality $\varepsilon$, MDP $M$, exploration strategy $\eprs$ and its parameter setting $\theta$, and the total number of steps (or observations) $\tau$.
It is unrealistic to express the very nature of the MDP (even for chain MDPs) and the exploration strategy in only a limited number of numerical variables.
Therefore, it is not likely that there exists some uniform expression suitable for all possible MDPs and exploration strategies.
Rather, we focus the discussion to some \textit{prototypes} of chain MDPs and exploration strategies, so that the result derived in this paper can be extended with additional effort as less as possible.

This subsection focuses on the four prototypes of chain MDPs.
These prototypes should capture the main characteristics that has critical impact on the difficulty of exploration.
In general, there are two characteristics of interest in chain MDPs: hazardousness of backward actions, and productivity of goal action.

The \textit{hazardousness of backward actions} refers to how bad the situation will be when a backward action is taken.
The degree of hazardousness, denoted by $H$, can be reflected through the number of chain states the agent is thrown back when taking backward actions.
In the least hazardous case, the agent is transited one state back, meaning that it will be sent to $s_{i-1}$ from $s_i$ by taking a backward action.
In the most hazardous case, on the other hand, the agent is transited to the start state $s_1$ regardless how far it is from the current state.
We write the former case as $H=1$, while the latter as $H=\infty$.
Other cases lie between these two extreme point, taking other positive integers as the degree of hazardousness.
A typical chain MDP is a mix of these cases, with each chain state $s_i$ having a hazardousness level $H_i$.
It is of particular interest to investigate the two extreme scenarios for chain MDPs, namely the ones with all states having $H_i=1$, and the ones with all states having $H_i = \infty$.

The \textit{productivity of goal action}, denoted by $G$, is the second key characteristic, which decides the value of goal as a function of goal reward $r_G$.
In the most fruitful case, taking the goal action at the goal state not only yields the goal reward $r_G$, but also transits the agent to exactly the same goal state $s_n$. 
This allows the agent to receive the goal reward every time step thereafter by simply repeating the goal action.
The other extreme case is that the goal action send the agent back to the start state $s_1$ while providing the goal reward.
In this case, the agent has to travel again through all chain states in order to obtain the goal reward once more, resulting in the least productivity.
We write the former case as $G=1$, while the latter as $G=\frac 1 \infty$.
Just as the hazardousness, the productivity of goal action for typical chain MDPs lie between these two extreme cases, and therefore it is suitable to regard the extreme cases as the prototype.

\begin{figure}
	\centering
	\begin{subfigure}[b]{1\textwidth}
		\centering
		\begin{tikzpicture}[
			auto, ->, >=latex,
			thick, 
			font=\scriptsize,
			main node/.style={circle,draw,font=\small},
			special node/.style={circle,draw,font=\small, line width=2pt},
			b reset/.style={densely dashed, bend left = 33, max distance = 40},
			b start/.style={densely dashed, in = 160, out = 205, looseness = 8, left = 1},
			a fail/.style={in = 100, out = 70, looseness = 8, above = 1},
			a goal/.style={in = -20, out = 25, looseness = 8, right = 1},
			hidden node/.style={draw=none,font=\huge, line width=2pt},
		]
	
			\node[special node] (1) at (1.5,0) {$1$};
			\node[special node] (5) at (11.5,0) {$n$};
			\foreach [evaluate=\x as \pos using 2.5*\x-1] \x in {2,...,3}
			{
				\node[main node] (\x) at (\pos,0) {$\x$};
			}
			\node[hidden node] (4) at (9,0) {$...$};
	
			\foreach \x/\y in {1/2,2/3,3/4}
			{
				\path (\x) 	edge [bend left = 20, looseness = 0.8] node {$p_\x$} (\y);
			}
			\path (4) 	edge [bend left = 20, looseness = 0.8] node {$p_{n-1}$} (5);
			\foreach \x in {1,...,3}
						{
				\path (\x) 	edge [a fail] node {$1- p_\x$} (\x);
			}
			
			\foreach \x/\y in {1/2,2/3,3/4,4/5}
			{
				\path (\y) 	edge [densely dashed, bend left = 20, looseness = 0.8] node {$1$} (\x);
			}
			\path (1) edge [b start] node {$1, r_D$} (1);
			\path (5) edge [a goal] node {$1, r_G$} (5);
			
		\end{tikzpicture}
		\caption{$H = 1, G = 1.$}
	\end{subfigure}
	\\
	\begin{subfigure}[b]{1\textwidth}
		\centering
		\begin{tikzpicture}[
			auto, ->, >=latex,
			thick, 
			font=\scriptsize,
			main node/.style={circle,draw,font=\small},
			special node/.style={circle,draw,font=\small, line width=2pt},
			b reset/.style={densely dashed, bend left = 33, max distance = 40},
			b start/.style={densely dashed, in = 160, out = 205, looseness = 8, left = 1},
			a fail/.style={in = 100, out = 70, looseness = 8, above = 1},
			a goal/.style={in = -20, out = 25, looseness = 8, right = 1},
			hidden node/.style={draw=none,font=\huge, line width=2pt},
		]
	
			\node[special node] (1) at (1.5,0) {$1$};
			\node[special node] (5) at (11.5,0) {$n$};
			\foreach [evaluate=\x as \pos using 2.5*\x-1] \x in {2,...,3}
			{
				\node[main node] (\x) at (\pos,0) {$\x$};
			}
			\node[hidden node] (4) at (9,0) {$...$};
	
			\foreach \x/\y in {1/2,2/3,3/4}
			{
				\path (\x) 	edge [bend left = 20, looseness = 0.8] node {$p_\x$} (\y);
			}
			\path (4) 	edge [bend left = 20, looseness = 0.8] node {$p_{n-1}$} (5);
			\foreach \x in {1,...,3}
						{
				\path (\x) 	edge [a fail] node {$1- p_\x$} (\x);
			}
			
			\foreach  [evaluate=\x as \bent using 0.7-0.07*\x] \x in {2,...,4}
			{
				\path (\x)  edge [b reset, looseness = \bent] node {} (1);
			}
			\path (5) edge [b reset] node {$1$} (1);
			\path (1) edge [b start] node {$1, r_D$} (1);
			\path (5) edge [a goal] node {$1, r_G$} (5);
			
		\end{tikzpicture}
		\caption{$H = \infty, G = 1.$}
	\end{subfigure}
	\\
	\begin{subfigure}[b]{1\textwidth}
		\centering
		\begin{tikzpicture}[
			auto, ->, >=latex,
			thick, 
			font=\scriptsize,
			main node/.style={circle,draw,font=\small},
			special node/.style={circle,draw,font=\small, line width=2pt},
			b reset/.style={densely dashed, bend left = 33, max distance = 40},
			b start/.style={densely dashed, in = 160, out = 205, looseness = 8, left = 1},
			a fail/.style={in = 100, out = 70, looseness = 8, above = 1},
			a goal/.style={in = -20, out = 25, looseness = 8, right = 1},
			a goal new/.style={in = 310, out = 250, looseness = 8,  max distance = 20},
			hidden node/.style={draw=none,font=\huge, line width=2pt},
		]
	
			\node[special node] (1) at (1.5,0) {$1$};
			\node[special node] (5) at (11.5,0) {$n$};
			\foreach [evaluate=\x as \pos using 2.5*\x-1] \x in {2,...,3}
			{
				\node[main node] (\x) at (\pos,0) {$\x$};
			}
			\node[hidden node] (4) at (9,0) {$...$};
	
			\foreach \x/\y in {1/2,2/3,3/4}
			{
				\path (\x) 	edge [bend left = 20, looseness = 0.8] node {$p_\x$} (\y);
			}
			\path (4) 	edge [bend left = 20, looseness = 0.8] node {$p_{n-1}$} (5);
			\foreach \x in {1,...,3}
						{
				\path (\x) 	edge [a fail] node {$1- p_\x$} (\x);
			}
			
			\foreach \x/\y in {1/2,2/3,3/4,4/5}
			{
				\path (\y) 	edge [densely dashed, bend left = 15, looseness = 0.6] node {$1$} (\x);
			}
			\path (1) edge [b start] node {$1, r_D$} (1);
			\path (5) edge [a goal, transparent] node {$1, r_G$} (5);
			\path (5) edge [a goal new] node {$1, r_G$} (1);
			
		\end{tikzpicture}
		\caption{$H = 1, G = \frac 1 \infty.$}
	\end{subfigure}
	\\
	\begin{subfigure}[b]{1\textwidth}
		\centering
		\begin{tikzpicture}[
			auto, ->, >=latex,
			thick, 
			font=\scriptsize,
			main node/.style={circle,draw,font=\small},
			special node/.style={circle,draw,font=\small, line width=2pt},
			b reset/.style={densely dashed, bend left = 33, max distance = 40},
			b start/.style={densely dashed, in = 160, out = 205, looseness = 8, left = 1},
			a fail/.style={in = 100, out = 70, looseness = 8, above = 1},
			a goal/.style={in = -20, out = 25, looseness = 8, right = 1},
			a goal new/.style={bend left = 80, looseness = 0.9,  max distance = 36},
			hidden node/.style={draw=none,font=\huge, line width=2pt},
		]
	
			\node[special node] (1) at (1.5,0) {$1$};
			\node[special node] (5) at (11.5,0) {$n$};
			\foreach [evaluate=\x as \pos using 2.5*\x-1] \x in {2,...,3}
			{
				\node[main node] (\x) at (\pos,0) {$\x$};
			}
			\node[hidden node] (4) at (9,0) {$...$};
	
			\foreach \x/\y in {1/2,2/3,3/4}
			{
				\path (\x) 	edge [bend left = 20, looseness = 0.8] node {$p_\x$} (\y);
			}
			\path (4) 	edge [bend left = 20, looseness = 0.8] node {$p_{n-1}$} (5);
			\foreach \x in {1,...,3}
						{
				\path (\x) 	edge [a fail] node {$1- p_\x$} (\x);
			}
			
			\foreach  [evaluate=\x as \bent using 0.7-0.07*\x] \x in {2,...,4}
			{
				\path (\x)  edge [b reset, looseness = \bent] node {} (1);
			}
			\path (5) edge [b reset] node {$1$} (1);
			\path (1) edge [b start] node {$1, r_D$} (1);
			\path (5) edge [a goal, transparent] node {$1, r_G$} (5);
			\path (5) edge [a goal new] node {$1, r_G$} (1);
			
		\end{tikzpicture}
		\caption{$H = \infty, G = \frac 1 \infty.$}
	\end{subfigure}
	\caption{Four prototype chains with $\text{length} = n$.}
	\label{figProtoChains}
\end{figure}
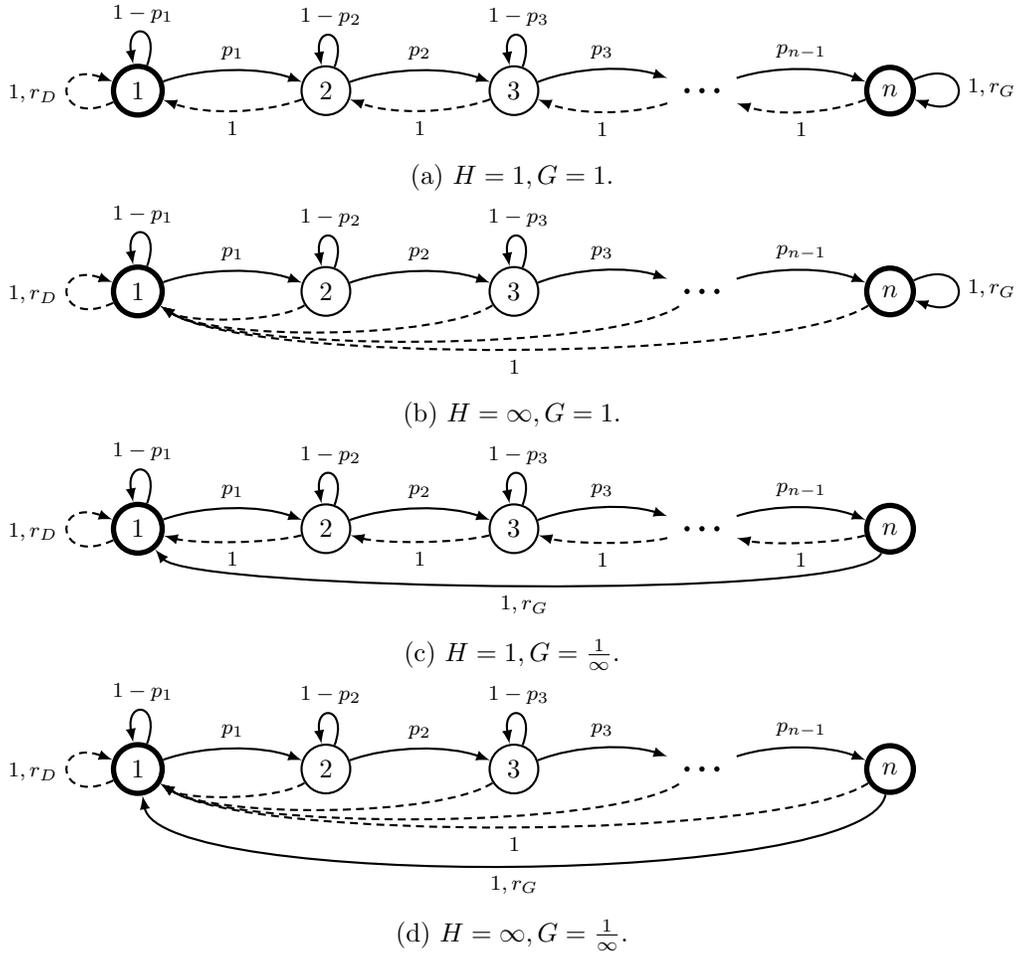

Combining the two characteristics, each with two extreme cases, four prototype chains are formed, as illustrated in Figure \ref{figProtoChains}.
The corresponding hazardousness $H$ and goal productivity $G$ of prototype chains are labelled in their captions.
For clarity, zero-value rewards as well as the action marks $a^+$ and $a^-$ are omitted in this figure. 

Actions are still distinguishable by their line pattern.
As in previous sections, forward and goal actions are drawn in solid lines, while backward actions are drawn in dashed lines.
By taking a forward action at state $s_i$ ($i<n$), the agent has probability $p_i$ to be transited to $s_{i+1}$, and $1-p_i$ to remain at $s_i$.
The transition probabilities $p_1, ..., p_{n-1}$ are not required to be equal in the same chain, but they must be non-zero, otherwise the goal state is unreachable.

The goal reward $r_G$ is assumed to be set large enough than the distracting reward $r_D$ such that the optimal policy in these prototype chains is to take forward action $a^+$ at any state.
The exact constraint that should be met for $r_G$ and $r_D$ will be left until Section \ref{secClosedValueFunc}.

For convenience, these prototype chains will be referred collectively as $M_C$, and respectively as $M_C(H=1,G=1)$, $M_C(H=\infty,G=1)$, $M_C(H=1,G=\frac 1 \infty)$, and $M_C(H=\infty,G=\frac 1 \infty)$ in the rest of this paper.
Also, we write $M_C(H=1)$ to refer to both $M_C(H=1,G=1)$ and $M_C(H=1,G=\frac 1 \infty)$, and likewise for $M_C(H=\infty)$, $M_C(G=1)$, and $M_C(G=\frac 1 \infty)$.
Although the length of the chain $n$ and forward probabilities $\allsth {p_i}$ are equally important factors, we will not explicitly state them here in order to keep notation short and clear.

\subsection{The Prototype Exploration Strategy}
\label{secProtoStrategy}
The second issue that should be addressed before further investigation is the prototype of exploration strategies.
It is difficult to find a representative strategy that is able to reflect the fundamental properties of both vanilla strategies such as $\varepsilon$-greedy and Boltzmann selection rule \citep{SuttonBarto98}, and the more advanced strategies.
Fortunately, there does exist a common principle shared by a large number of non-Bayesian advanced exploration strategies.
It is often known as ``optimism in the face of uncertainty'' or ``the optimism principle'' \cite{kaelbling1996reinforcement}, which has been already mentioned in the discussion of the second-level decomposition (Lemma \ref{lemL2Decomp}).

The main idea of the optimism principle is that the agent should assume any state-action pair to be highly rewarding ones as long as there is no sufficient evidence against this optimistic assumption for that state-action pair.
This principle forces the agent to try every action at every state for several times, ensuring a base amount of observations for all visited state-action pairs.
The optimism principle is widely adopted among PAC-MDP strategies \citep{kakade2003sample, szita2010mormax, lattimore2014near, zhang2015increasingly} as well as the strategies with regret bound guarantees \citep{jaksch2010near}.
It is also applied in some Monte-Carlo tree search algorithms such as UCT \citep{kocsis2006bandit}.

As discussed in Section \ref{secBasicProperties}, it is less likely for a primitive exploration strategy without optimism, for example $\varepsilon$-greedy, to complete any of the traverse events.
As a result, without the help of some additional techniques that could efficiently exploit the prior knowledge, the success probability of exploration for these strategies can be rather low, sometimes even close to zero.
There also exists a different family of non-optimistic exploration strategies, namely the Bayesian strategies \citep{vlassis2012bayesian}.
However, these works are based on a different optimality criteria called Bayesian optimality, which is not equivalent to the traditional criteria used in this paper.
Therefore, within the scope of this paper, it is sufficient to base our prototype strategy on the optimism principle.

Putting all consideration together, we have the following prototype strategy.
\begin{definition}[Optimistic Prototype Strategy]
	\label{defProtoStrategy} 
	The Optimistic Prototype Strategy, denoted by $\ops(m)$, is an exploration strategy that can be described as follows.
	\begin{itemize}[nosep]
		\item There is a universal parameter, denoted by $m$, which is a non-negative integer. 
		
		\item At any state $s \in S$, if there exists some action $a \in A$ that its number of observations $N_{s,a} < m$, then state $s$ is said to be \textit{under-explored}, and the agent takes action $a$ to collect one more observation.
		
		\item If there is more than one such action, then the agent randomly and uniformly choose one of them.
		
		\item If there is no such action at all, then state $s$ is said to be \textit{fully explored}, and the agent moves to the under-explored state with least expected step required from $s$.
		
		\item If all visited states have become fully explored, then the agent stops optimistic exploration and outputs a policy according to Assumption \ref{assPlan}. The agent follows the output policy thereafter, if the learning process is not stopped.
	\end{itemize}
\end{definition}

This prototype strategy is based on the famous baseline PAC-MDP strategy R-MAX \citep{brafman2002r, kakade2003sample}.
Actually, there is no behavioural difference between R-MAX and the Optimistic Prototype Strategy when they are executed on finite chain MDPs, given that the discount factor $\gamma$ is sufficiently large.
Other advanced optimistic strategies, no matter PAC-MDP or not, can be seen as an extension of this prototype, where the universal parameter $m$ is changed into a function $m_\psi(s,a)$ such that its value varies for different state-action pairs, and may even changes during learning.
Therefore, we will focus on analysing the exploration behaviour of this prototype strategy in the rest of this paper, and make our analysis as flexible and compatible as possible for other optimistic strategies.

\subsection{The Traverse Events}
\label{secClosedTrav}
By Definition \ref{defTrav}, a $\pi$-traverse event is said to have occurred if and only if for all $s \in S$, state-action pair $(s, \pi(s))$ has been tried at least once.
Although this seems to be dependent to the visit numbers, rather than the converse as in Figure \ref{figDependency}, the traverse events are placed deliberately at a level lower than the visit numbers. 
The reason for this is, the traverse events only care about a small fraction of state-action pairs, and thus is easier to decide than the visit number of all state-action pairs.

In the scenario where optimistic prototype strategy $\ops(m)$ is applied to the agent in prototype chain $M_C$'s, the following lemma can be figured out rather effortlessly:
\begin{lemma}
	For any policy $\pi \in \varPi$, a traverse event $\Etrav^\pi$ occurs if and only if the agent successfully arrives at the goal state $s_n$.
\end{lemma}
\begin{proof}
	Because the agent starts from $s_1$, this lemma is obvious by Figure \ref{figProtoChains}.
\end{proof}

Whether or not the agent could reach $s_n$ depends on $M_C$ and $\ops(m)$, and of course, the total number of steps $\tau$.
Actually, there is some $\tau_m$ that is of particular importance: the exact time step that the agent finishes its active exploration behaviour promoted by the optimism principle.
Its formal definition is as follows.

\begin{definition}
	The \textbf{ending of exploration} for $\ops(m)$, denoted by $\tau_m$, is the first time step that, after the transition takes place, $N_{s, a} \geq m$ holds for all $s \in S \setminus S_u$ and $a \in A$, where $S_u:=\{s|\forall a \in A, N_{s,a}=0\}$ is the set of unvisited states.
\end{definition}

After this time step, the agent stops optimistic exploration and begins exploitation according to Definition \ref{defProtoStrategy} of prototype strategy.
It is less likely that a traverse event not occurred so far happens after $\tau_m$.
Consequently, the probability of traverse events (as well as the success probabilities discussed in later sub-sections) can be directly inferred from the ones at $\tau_m$.
In the case of $\tau < \tau_m$, on the other hand, the probabilities can be interpolated through different setting of $m$.
Therefore, it is sufficient to consider the case of $\tau=\tau_m$.
By considering the ending of exploration, the number of steps $\tau$ are actually integrated into the parameter $m$, simplifying the following discussions.

Now we have the closed-form expression for the probability of $\pi$-traverse events as follows.

\begin{theorem}[Probability of Traverse Events]
	\label{theoTrav}
	For any $\pi \in \varPi$, it holds that 
	\begin{equation*}
		\pr(\Etrav^{\pi}|M_C,\ops(m),\tau_m) = \prod_{i=1}^{n-1}(1-(1-p_i)^m).
	\end{equation*}
\end{theorem}
\begin{proof}
If the agent intends to reach $s_n$ eventually, it has to succeed moving ahead through forward actions $a^+$ at every chain state $s_1$, $s_2$, ..., $s_{n-1}$ at least once.
According to Definition \ref{defProtoStrategy}, the agent will try $a^+$ at least $m$ times at any chain state $s_i$.
However, these are also the only chances for the agent to try $a^+$.
The reason is, if all $m$ tries at $s_i$ fail, then the agent will never be informed of the existence of $s_{i+1}$, nor the existence of a goal reward.
In the following learning process, the agent will either explore the states $s_1$, ..., $s_{i-1}$, the backward action $a^-$ at $s_i$, or exploits the distraction reward $r_D$; no further attempt to reach goal will be carried out.
Therefore, the probability of a traverse event happens is the probability that the agent successively succeeds moving forward at least once at every chain state within $m$ attempts, hence the theorem.
\end{proof}

An interesting fact of this theorem is that it is dependent neither to $\pi$, nor to the hazardousness $H$ or productivity $G$ of prototype chain $M_C$.
Considering that $\tau_m$ has been integrated into $m$, it suffices to write the traverse probability simply as $\pr(\Etrav|m)$ in later discussions.

\subsection{The Visit Numbers}
\label{secClosedVisitNum}
In this subsection we investigate the visit numbers $\allNsa$ and $\allNsas$ given that the traverse event $\Etrav$ has happened.
All following discussions are based on the occurrence of the traverse event unless otherwise stated, and this precondition will not be repeated every time for the sake of briefness.

The direct impact of the occurrence of traverse event on the visit numbers is $N_{s_i, a^+, s_{i+1}} \geq 1$ for all $i<n$.
Consequently, the agent is informed of the existence of the all chain states, including the goal state $s_n$.
This guarantees that there actually exists an ending step of exploration $\tau_m$ within finite steps such that $N_{s,a} \geq m$ holds for all $s \in S$ and $a \in A$.
Then it leads to the estimated transition probability $\hat{P}(s'|s,a) = \frac{N_{s,a,s'}} {N_{s,a}} > 0$ for any $s,s' \in S$ and $a \in A$ in $M_C$.

However, due to the stochastic nature of environment, it is not very likely that two independent runs of learning algorithm result in the same trajectory.
Therefore, the visit numbers should be treated as random variables, and their distributions have to be revealed in order to tell the distributions of the estimated transition probabilities $\hat{P}$.

The major obstacle here comes from the very nature of interacting in an MDP: the state at the next time step must be the state the agent be transited to at the current time step.
This leads to the inter-dependency within all visit numbers $\allNsa$ and $\allNsas$.
Fortunately, not all part of this dependency is relevant to the success probability.
As an example, assume that by time step $\tau$, some transition $(s_i, a_j, s_k)$ occurred three times within five tries of $a_j$ at $s_i$.
Then exactly at which time steps these $(s_i, a_j, s_k)$ transitions took place does not make any difference for  estimating the transition probability $\hat{P}(s_k|s_i,a_j)$; it will always be $3 / 5 = 0.6$.
Actually, it is not even relevant which state is the current one, although it is one of the most important factors in interacting with the environment.

Therefore, in the context of deriving the success probability, it suffices to deal with the visit numbers in the following manner.

\begin{definition}[The Three Rules of Visit Numbers] \
	\begin{description}
		\item [Rule 1] Regard the expected visit numbers $\allBarNsa$ and $\allBarNsas$ as constant numbers that are only dependent to other state-actions that share a common state.
		Specifically, the expected total number of actions taken at any state should equal to the expected total number of being transited into that state.
		
		\item [Rule 2] Regard the actual state-action visit numbers $\allNsa$ as constant numbers that equals to $\allBarNsa$.
		
		\item [Rule 3] Regard the actual transition visit numbers $\allNsas$ as random variables following some binomial distributions. 
		More precisely, $N_{s,a,s'} \sim \text{Binomial}(\bar{N}_{s,a}, P(s'|s,a))$ for all $s,s' \in S$ and $a \in A$.
	\end{description}
	\label{defThreeRulesNsa}
\end{definition}

Although the last two rules may introduce a loss of dispersion to the actual visit numbers, this loss is rather insignificant and thus negligible.
The reason is, the agent exploits the information from visit numbers through estimated transition probabilities, which are computed by $\hat{P}(s'|s,a) = N_{s,a,s'} / N_{s,a}$.
As long as $N_{s,a,s'} \sim \text{Binomial}(\bar{N}_{s,a}, P(s'|s,a))$, the estimates $\hat{P}(s'|s,a)$ are unbiased, and the loss of dispersion only has impact on the variance of $\hat{P}(s'|s,a)$, which is decided mainly by the variance introduced by the binomial distributions rather than $\allNsa$.
The empirical results in Section \ref{secVerifyNsa} also shows that the dispersion are actually preserved well by these rules, further supporting their effectiveness.
Therefore, the above three rules are sufficient for the purpose of this paper.

Now let us consider the closed-form expression of expected visit numbers $\allBarNsa$ for prototype strategy $\ops(m)$ running in chains $M_C$.
For convenience, we write $\bar{N}_{s_i, a^+}$,  $\bar{N}_{s_i, a^-}$, $\bar{N}_{s_i, a^+, s_j}$, and $\bar{N}_{s_i, a^-, s_j}$ respectively as $\bar{N}_i^+$,  $\bar{N}_i^-$, $\bar{N}_{i,j}^+$, and $\bar{N}_{i,j}^-$ in the rest of this section.

Let $\lambda_i$ denotes the expected total number of being transited into chain state $s_i$ by any transitions except for $(s_{i-1}, a^+, s_i)$.
This includes failures of forwarding $(s_i, a^+, s_i)$,  goal transitions $(s_n, a^+, s_i)$, and backward transitions $(s_j, a^-, s_i)$ with $j \geq i$.
More formally,

\begin{definition}
	For any $1 \leq i \leq n$, we define $\lambda_i := \bar{N}_{i,i}^+ + \bar{N}_{n,i}^+ + \sum_{j=i}^{n} \bar{N}_{j,i}^- $.
	\label{defLambda}
\end{definition}

\begin{lemma}
	The following hold in corresponding chain MDP $M_C$.
	\begin{description}[labelwidth = 3.2cm]
		\item[(a) $\bm{M_C(H=1)}$] $\lambda_{i} =  (1-p_i)\bar{N}_{i}^+ + \bar{N}_{i+1}^-$, for $1<i<n$.
		
		\item[(b) $\bm{M_C(H=\infty)}$] $\lambda_{i} = (1-p_i)\bar{N}_{i}^+$, for $1<i<n$.
		
		\item[(c) $\bm{M_C(G=1)}$] $\lambda_n = \bar{N}_n^+$.
		
		\item[(d) $\bm{M_C(G=\frac 1 \infty)}$] $\lambda_n = 0$.
		
	\end{description}
	\label{lemLambda}
\end{lemma}
\begin{proof}
	For all $1 \leq i<n$ it holds that $\bar{N}_{i, i}^+ = (1-p_i)\bar{N}_{i}^+$.
	Then by Definition \ref{defLambda}.
\end{proof}

Lemma \ref{lemLambda} shows an interesting trait of chain MDPs, that is, if two chains share the same hazardousness $H$ at state $s_i$, then they share the expression of $\lambda_{i}$ for $1<i<n$ as well; the same can be said for the productivity $G$ and $\lambda_n$,
Actually it is easy to see that these also hold for non-prototype chains, meaning that even for a chain with mixed hazardousness and productivity, the properties of $\lambda$ above can be directly applied.
This could be very helpful when analysing more complicated chains, as can be seen in Section \ref{SectionApplying}.

According to Rule 1 of Definition \ref{defThreeRulesNsa}, the following equations hold for all four prototype chains.

\begin{equation}
\begin{cases}
\bar{N}_i^+ + \bar{N}_i^- = \lambda_i + p_{i-1} \bar{N}_{i-1}^+ & (i \neq 1) \\
\bar{N}_1^+ + \bar{N}_1^- = \lambda_1.
\end{cases}
\label{eqNsa}
\end{equation}

The above equations are only a paraphrase of Rule 1, in the sense that the left side of equations are the expected total number of tries at $s_i$,  while the right side are the expected total number of being transited into $s_i$.
The latter includes successful forwarding transitions, which is expected to be $p_{i-1} \bar{N}_{i-1}^+$ times, and the other transitions with $\lambda_i$ times.
In the special case of start state, there is no forwarding transitions, and hence the right side being simply $\lambda_1$.

By applying Lemma \ref{lemLambda} to Equations \ref{eqNsa}, the expected visit numbers $\allBarNsa$ in four prototype chains can be derived accordingly.
The result is as follows.

\begin{lemma} [The Expected Visit Numbers at $\bm{\tau=\tau_m}$] \
	\begin{description}
		\item[(a) Backward actions] In all four prototype chains, it holds that $\bar{N}_i^- = m$ $(1 \leq i \leq n)$.
		
		\item[(b) Goal actions] In all four prototype chains, it holds that $\bar{N}_n^+ = m$.
		
		\item[(c) Forward actions] The following holds for $1 \leq i < n$.
		\begin{description}[labelwidth = 5cm]
			\item[(c.1) $\bm{M_C(H=1, G=1)}$] $\bar{N}_i^+ = \frac{m \;(+1)}{p_i}$.
			
			\item[(c.2) $\bm{M_C(H=\infty, G=1)}$] $\bar{N}_i^+ = \frac{(n-i)m \;(+1)}{p_i}$.
		
			\item[(c.3) $\bm{M_C(H=1, G=\frac 1 \infty)}$] $\bar{N}_i^+ = \frac{2m}{p_i}$.
			
			\item[(c.4) $\bm{M_C(H=\infty, G=\frac 1 \infty)}$] $\bar{N}_i^+ = \frac{(n+1-i)m}{p_i}$.
		\end{description}
	\end{description}
	\label{lemExpressionNsa}
\end{lemma}
\begin{proof}
	\begin{description}[leftmargin = 0pt]
		\item[(a)] By Definition \ref{defProtoStrategy}, it holds that $N_i^- \geq m$ at time step $\tau_m$.
		Because it is significantly easier to move backward than forward in prototype chains, the bottleneck of exploration is always at the state nearer to the goal state.
		Thus it is not likely that the agent take backward action $a^-$ at any $s_i$ after it has been tried $m$ times in order to get back to an under-explored state $s_j$ with $j<i$.
		Therefore, $\bar{N}_i^- = m$ holds for any $1 \leq i \leq n$ at $\tau_m$.
		
		\item[(b)] As in (a), the bottleneck of exploration is to reach the states closer to the goal state.
		Therefore, it is not likely that the agent further takes $a^+$ at $s_n$ after $\bar{N}_n^+ = m$ in order to move to other under-explored states, and hence the result.
		
		\item[(c.1)-(c.4)] The results can be computed recursively by Equation \ref{eqNsa} from $n-1$ down to $1$, with all $\lambda$'s taking values according to Lemma \ref{lemLambda}. 
		
	\end{description}
	Note that the goal transitions in chains with $G=1$ are self-loops, and exploring the goal actions in these chains has no impact on the exploration of other states.
	Therefore, these state-actions are likely be explored less than $m$ times when all other state-actions are fully explored.
	In this case, all forward actions have to succeed exactly one more time, so that the agent could get to the goal state and finish its exploration by repeating the goal action, resulting in an additional $+1$ adjustment to the numerator for (c.1) and (c.2).
	The most appropriate adjustment should be a value between 0 and 1, but the impact of the actual choice here is negligible.
\end{proof}

\begin{remark}
	By definition it holds that $\tau_m = \sum_{i=1}^{n} (N_i^+ + N_i^-)$ and $\mean \, \tau_m = \sum_{i=1}^{n} (\bar{N}_i^+ + \bar{N}_i^-)$ at $\tau=\tau_m$.
	This can be used to evaluate how many steps have passed when the exploration process ends.
\end{remark}

Having worked out the closed-form expressions for expected visit numbers $\allBarNsa$, the distribution of actual visit numbers $\allNsa$ and $\allNsas$ can be obtained without effort by applying Rule 2 and Rule 3 of Definition \ref{defThreeRulesNsa}.

\subsection{The Actual and Estimated State Value Functions}
\label{secClosedValueFunc}
In this subsection we will derive the expressions of the state value functions for chain prototypes.
As can be seen in the dependency graph (Figure \ref{figDependency}), the actual state value functions ${V^\pi}$ are only dependent to the MDP $M$, and thus this is a relatively easy part compared to the others.
According to Bellman equation (see Equation \ref{eqBellman1}), actual state value $V^\pi(s)$ for a state $s \in S$ can be seen as an expression in terms of all relevant transition probabilities $P(s'|s,\pi(s))$ and rewards $R(s,\pi(s),s')$.
Therefore, the main task here is to solve Bellman equations for all chain states $s_1, ..., s_n$ to obtain the expressions of state values in terms of $p_1, ..., p_{n-1}$, $r_G$, and $r_D$.

After the expressions of actual state values ${V^\pi}$ are acquired, it is trivial to turn them to estimated values ${\hat{V}^\pi}$.
Actually, expressions of ${\hat{V}^\pi}$ can be obtained by directly replacing all $P(s'|s,a)$ with $\hat{P}(s'|s,a) = N_{s,a,s'}/N_{s,a}$, and all $R(s,a,s')$ with $\hat{R}(s,a,s')$.
Although most planning algorithms do not estimate state values explicitly in this way, utilizing any iterative algorithm that is able to asymptotically converge to the actual values is nevertheless effectively equivalent in limit to estimating them straightforwardly via these expressions, and in this paper the computation cost is of no concern.

The only additional matter here is to decide which policies should be considered, because there are $2^n$ different possible policies in total for a chain MDP with length $n$.
Fortunately, in chain MDPs, most of the possible policies are unimportant because they are unconditionally dominated by a specific family of policies, meaning that no matter what the values are for $p_1, ..., p_{n-1}$, $r_G$, and $r_D$, these dominated policies have some state value lower than the dominants.
The dominant family can be specified as follows.

\begin{definition}
	Let $\pbf$ be a policy that chooses $a^-$ at states $s_1$, $s_2$, ..., $s_{k}$, and chooses $a^+$ at states $s_{k+1}$, ..., $s_n$, where $0 \leq k  \leq n$.
	Let $\allsth{\pbf}$ collectively denotes all such policies.
	\label{defPiBackFwd}
\end{definition}

Clearly, there are $(n+1)$ policies in $\allsth{\pbf}$.
In the extreme case where $k=0$, the policy always chooses $a^+$ at any state, while in the other extreme case where $k=n$, the policy always chooses $a^-$.
The dominating power of $\allsth{\pbf}$ will be explained in later part of this subsection.
Before that, some useful properties of $\allsth{\pbf}$ will be introduced first.

\begin{lemma}
	Let $s_i$ be an arbitrary non-goal chain state in $M_C$.
	For any policy $\pi$ such that $\pi(s_i)=a^+$, it holds that $V^\pi(s_i) = \frac{\gamma p_i}{1 - \gamma (1-p_i)} V^\pi(s_{i+1})$.
	\label{lemViToNext}
\end{lemma}
\begin{proof}
	By Bellman equation it holds that $V^\pi(s_i) = \gamma (p_i V^\pi(s_{i+1}) + (1-p_i)V^\pi(s_i))$, hence the lemma.
\end{proof}
\begin{corollary}
	$\Vbf(s_i) = \frac{\gamma p_i}{1 - \gamma (1-p_i)} \Vbf(s_{i+1})$ holds for all $k<i<n$ in $M_C$.
	\label{lemVBFiToNext}
\end{corollary}

Now we are ready for working out the closed-form expressions of state value functions for $\allsth{\pbf}$ in prototype chains.

\begin{lemma}[State Value Functions of $\bm{\allsth{\pbf}}$] 
	Let $F_j := \prod_{i=j}^{n-1} \frac{\gamma p_i}{1 - \gamma (1-p_i)}$ for $j<n$ and $F_n := 1$. 
	Then the following equations hold in corresponding prototype chains.
	\begin{description}[labelwidth = 4.7cm]
		\item[(a) $\bm{j \leq k, \; M_C(H = 1)}$] $V^{\pbf}(s_j) = \frac{\gamma^{j-1}}{1-\gamma} r_D$.
		
		\item[(b) $\bm{j \leq k, \; M_C(H = \infty)}$] $V^{\pbf}(s_j) = \frac{\gamma}{1-\gamma} r_D$ if $j>1$, 
		and $V^{\pbf}(s_1) = \frac{1}{1-\gamma} r_D$.
		
		\item[(c) $\bm{j > k, \; M_C(G = 1)}$] $V^{\pbf}(s_j) = \frac{F_j}{1-\gamma} r_G$.
		
		\item[(d) $\bm{j > k, \; M_C(G = \frac 1 \infty)}$] $V^{\pbf}(s_j) = \frac{F_j}{1-\gamma F_1} r_G$ if $k=0$, and 
		
		\item 
		$V^{\pbf}(s_j) = F_j(r_G+\frac{\gamma}{1-\gamma}r_D)$ if $k>0$.
	\end{description}
	\label{lemExpressionV}
\end{lemma}
\begin{proof}
		\textbf{(a)}
		This rule (as well as (b)) will not be used if $k=0$ because $j$ ranges from 1 to n.
		In the case that $k>0$, $\Vbf(s_1)=r_D + \gamma \Vbf(s_1) = \frac{1}{1-\gamma} r_D$.
		In chains with $H=1$, we have $\Vbf(s_j) = \gamma \Vbf(s_{j-1}) = ... = \gamma^{j-1} \Vbf(s_{1})$ for $1 < j \leq k$, hence the result.
		
		\textbf{(b)}
		As in (a), $\Vbf(s_1) = \frac{1}{1-\gamma} r_D$ when $k>0$.
		In chains with $H=\infty$, we have $\Vbf(s_j) = \gamma \Vbf(s_{1})$ for $1 < j \leq k$, hence the result.
		
		\textbf{(c)}
		This rule (as well as (d)) will not be used if $k=n$ because $j$ ranges from 1 to n.
		In the case that $k<j<n$, Corollary \ref{lemVBFiToNext} applies, and hence $\Vbf(s_j) = F_j \Vbf(s_n)$.
		In chains with $G=1$, we have $\Vbf(s_n) = r_G + \gamma \Vbf(s_n) = \frac{r_G}{1-\gamma}$, hence the result.
		
		\textbf{(d)}
		As in (c), it holds that $\Vbf(s_j) = F_j \Vbf(s_n)$.
		In chains with $G=\frac 1 \infty$, we have  $\Vbf(s_n) = r_G + \gamma \Vbf(s_1)$.
		However, $\Vbf(s_1)$ here is decided by whether $k=0$ or not.
		If $k=0$, then since $\Vbf(s_1) = F_1 \Vbf(s_n)$, we have $\Vbf(s_n) = \frac {r_G} {1-\gamma F_1}$, hence the first part of (d).
		If $0<k<n$, then $\Vbf(s_1) = \frac{r_D}{1-\gamma}$ as in (a) and (b), resulting in $\Vbf(s_n) = r_G+\frac{\gamma}{1-\gamma}r_D$, and hence the second part of (d).
\end{proof}

Lemma \ref{lemExpressionV} covers the closed-form expressions of state values for all four prototype chains, all chain states, and all possible $\allsth{\pbf}$.
For example, to get the expressions for policy $\pi^{-+}_3$ in an $(H = \infty, G=1)$ chain with length $n=8$, rule (b) should be used for $s_1$ to $s_3$, and rule (c) for $s_4$ to $s_8$.
Rule (c) and (d) are particularly useful, because they are not limited to prototype chains, but can also be applied to any chain MDPs with $G=1$ or $G = \frac 1 \infty$.

The following lemma reveals the dominating power of $\allsth{\pbf}$, which significantly reduces the difficulty of further analysis.

\begin{lemma} [Dominating Power of $\bm{\allsth{\pbf}}$]
	In any prototype chain with any possible settings of $p_1, ..., p_{n-1} \in [{0,1}]$, $r_G>0$, and $r_D>0$,
	every policy $\pi$ that is not within the family of $\allsth{\pbf}$ is dominated by some policy in $\allsth{\pbf}$, i.e. there exist some $0 \leq k \leq n$ and $0 \leq i \leq n$ such that $V^\pi(s) \leq \Vbf(s)$ for all $s \in S$ and $V^\pi(s_i) < \Vbf(s_i)$.
	\label{lemDominance}
\end{lemma}
\begin{proof}
	If a policy $\pi$ does not belong to the family of $\allsth{\pbf}$, then by Definition \ref{defPiBackFwd}, there exists some $i<n$ such that $\pi(s_i) = a^+$ and $\pi(s_{i+1}) = a^-$.
	
	For $M_C(H=1)$, we have $V^\pi(s_{i+1}) = \gamma V^\pi(s_i)$.
	By Lemma \ref{lemViToNext} we have $V^\pi(s_i) = \frac{\gamma p_i}{1 - \gamma (1-p_i)} V^\pi(s_{i+1})$, and therefore $V^\pi(s_{i+1}) = V^\pi(s_i) = 0$.
	However, by Lemma \ref{lemExpressionV}, both $\Vbf(s_i)$ and $\Vbf(s_{i+1})$ are positive and thus greater than 0.
	This even holds when $p_1=...p_{n-1}=0$ as long as $k=n$.
	Therefore, $\pi$ is dominated by the policies in $\allsth{\pbf}$ for $M_C(H=1)$.
	
	For $M_C(H=\infty)$, we have $V^\pi(s_{i+1}) = \gamma V^\pi(s_1)$.
	Again by Lemma \ref{lemViToNext}, we have $V^\pi(s_i) = \frac{\gamma p_i}{1 - \gamma (1-p_i)} V^\pi(s_{i+1}) = \frac{\gamma^2 p_i}{1 - \gamma (1-p_i)} V^\pi(s_{1})$.
	However $\frac{\gamma p_i}{1 - \gamma (1-p_i)} < 1$, and thus if we change $\pi(s_i)$ to $a^-$, then $V^\pi(s_{i})$ will increased to $\gamma V^\pi(s_{1})$, dominating the original $\pi$.
	If we continue improving $\pi$ by eliminating any cases such that $\pi(s_i) = a^+$ and $\pi(s_{i+1}) = a^-$, it will eventually turns into a policy that belongs to $\allsth{\pbf}$.
	Therefore, $\pi$ is dominated by some policy in $\allsth{\pbf}$ for $M_C(H=\infty)$.
	
	Other than these states it holds that $V^\pi(s) \leq \Vbf(s)$ by recursively applying Lemma \ref{lemViToNext}. 
	Since $M_C(H=1)$ and $M_C(H=\infty)$ covers all four prototype chains, this lemma has been proved.
\end{proof}

As mentioned at the beginning of this subsection, the estimated state values $\hat{V}$ can be obtained by substitute the estimated transition probabilities $\hat{P}$ for $P$ in the expressions of $V$.
Because $\hat{V}$ can be seen as actual state values in estimated chain $\hat{M}_C$, which differs from $M_C$ only in $p_1,...,p_{n-1}$, and because Lemma \ref{lemDominance} holds regardless to $p_1,...,p_{n-1}$, this lemma also holds for $\hat{V}$.
In other words, all policies outside $\allsth{\pbf}$ family is dominated by $\allsth{\pbf}$ with respect to $\hat{V}$.

\subsection{Solving the Success Probabilities}
\label{secClosedSPE}

We have worked out the expressions of traverse probability $\pr(\Etrav|M_C,\ops(m),\tau_m)$ that holds for any $\pi$.
Then we have the expressions of expected visit numbers $\allsth{\bar{N}_{s,a}}$ and $\allsth{\bar{N}_{s,a,s'}}$ at $\tau_m$, given that the traverse event happened.
According to the rules in Definition \ref{defThreeRulesNsa}, we have the distribution of $\allNsa$ and $\allNsas$, and then the distribution of estimated transition probabilities $\hat{P}$.
Lastly, we have the expressions of state value function $\Vbf$, and by replacing $P$ with $\hat{P}$ we have the expressions of estimated state values $\hat{V}^\pbf$ for the $\allsth{\pbf}$ family.

The time has come to confront the $\pi$-success probabilities $\allsth{\psucc^\pi}$ and the $\varepsilon$-success probability $\psucce$.
Actually, according to Lemma \ref{lemL1Decomp}, the first-level decomposition, the $\varepsilon$-success probability $\psucce$ is simply the sum of $\pi$-success probabilities $\psucc^\pi$ over all relevant policies $\pi \in \Pi^\varepsilon$.
As discussed in Remark \ref{remarkL1Decomp}, specifying the set of desired policies is a task that can be handled by RL practitioners themselves, and thus is not within the main concern of this paper.
Therefore, the ultimate challenge comes from the $\pi$-success probabilities $\allsth{\psucc^\pi}$.

As Lemma \ref{lemDominance} implies, all non-$\allsth{\pbf}$ policies are unconditionally dominated by $\allsth{\pbf}$, regardless of the estimated transition probabilities $\hat{P}$.
As a direct consequent, we have the following result for these dominated policies.

\begin{theorem}[$\bm{\pi}$-success Probability for Non-$\bm{\allsth{\pbf}}$ Policies]
	For any policy $\pi$ that does not belong to the $\allsth{\pbf}$ family, it holds that $\psucc^\pi_{M_C,\ops(m),\tau_m}=0$.
\end{theorem}
\begin{proof}
	Because Lemma \ref{lemDominance} does not rely on any specific setting of transition probabilities, this means that no matter what the estimated transition probabilities $\hat{p}_1, \hat{p}_2, ..., \hat{p}_{n-1}$ are (even if they are utterly inaccurate), those non-$\allsth{\pbf}$ policies are still dominated by $\allsth{\pbf}$ in their estimated values, and thus will never be outputted by the planning algorithm.
	Therefore, their $\pi$-success probabilities are always 0.
\end{proof}

This theorem allows us to simply ignore all policies outside the $\allsth{\pbf}$ family thereafter.
Now let us focus on the $\pi$-success events among the $\allsth{\pbf}$ family.
Clearly, the one in $\allsth{\pbf}$ that dominate every other family members must be the one that dominate all policies in $\varPi$.
Therefore, the occurrence of $\pi$-success event is equivalent to $\pi$ being the dominant within $\allsth{\pbf}$.

The dominance relationship themselves are decided by comparing the estimated state values $\hat{V}^\pi$ between family members.
Although there are $(n+1)$ policies in the $\allsth{\pbf}$ family, the following lemma removes the necessity of comparing $\binom{n+1}{2}$ pairs of policies, or $n \binom{n+1}{2}$ pairs of estimated state values.

\begin{lemma} [$\bm{\pi^{-+}_j}$-success in $\bm{M_C(G=1)}$]
	In $M_C(G=1)$, policy $\pi^{-+}_j$ with some $0 \leq j \leq n$ dominates $\allsth{\pbf}$ on $\hat{V}$ if and only if the following hold simultaneously.
	
	\emph{\textbf{(1)}} $\hat{V}^{\pi^{-+}_j}(s_j) > \hat{V}^{\pi^{-+}_{j-1}}(s_j)$.
		
	\emph{\textbf{(2)}} $\hat{V}^{\pi^{-+}_j}(s_{j+1}) > \hat{V}^{\pi^{-+}_{j+1}}(s_{j+1})$.
	
	For $\pi^{-+}_0$, only \emph{(2)} is required, while for $\pi^{-+}_n$ only \emph{(1)} is required.
	\label{lemCompPBF1}
\end{lemma}
\begin{proof}
	Consider an arbitrary $\pbf$ such that $k>j$. 
	By Lemma \ref{lemExpressionV}, $\hat{V}^{\pi^{-+}_j}(s_i) = \hat{V}^{\pi^{-+}_k}(s_i)$ for all $i \leq j$ and $i > k$.
	For $j < i \leq k$, on the other hand, it holds that $\hat{V}^{\pi^{-+}_j}(s_{j+1}) < \hat{V}^{\pi^{-+}_j}(s_{j+2}) < ... < \hat{V}^{\pi^{-+}_j}(s_{k})$, and $\hat{V}^{\pi^{-+}_k}(s_{j+1}) \geq \hat{V}^{\pi^{-+}_k}(s_{j+2}) \geq ... \geq \hat{V}^{\pi^{-+}_k}(s_{k})$.
	Therefore, $\pi^{-+}_j$ dominates $\pbf$ if and only if $\hat{V}^{\pi^{-+}_j}(s_{j+1}) > \hat{V}^{\pi^{-+}_k}(s_{j+1})$.
	Since $\hat{V}^{\pi^{-+}_k}(s_{j+1}) = \hat{V}^{\pi^{-+}_{j+1}}(s_{j+1})$, that is equivalent to $\hat{V}^{\pi^{-+}_j}(s_{j+1}) > \hat{V}^{\pi^{-+}_{j+1}}(s_{j+1})$.
	Likewise, for $k<j$, $\pi^{-+}_j$ dominates $\pbf$ if and only if $\hat{V}^{\pi^{-+}_j}(s_j) > \hat{V}^{\pi^{-+}_{j-1}}(s_j)$.
	Hence the lemma.
\end{proof}

\begin{lemma} [$\bm{\pi^{-+}_j}$-success in $\bm{M_C(G=\frac 1 \infty)}$]
	
%
%
%
%

	In $M_C(G==\frac 1 \infty)$, policy $\pi^{-+}_0$ dominates $\allsth{\pbf}$ on $\hat{V}$ if and only if
	
	\emph{\textbf{(0)}} $\hat{V}^{\pi^{-+}_0}(s_1) > \hat{V}^{\pi^{-+}_1}(s_1)$.
	
	If this does not happen, then policy $\pi^{-+}_j$ with some $0 < j \leq n$ dominates $\allsth{\pbf}$ on $\hat{V}$ if and only if the following hold simultaneously.
	
	\emph{\textbf{(1)}} $\hat{V}^{\pi^{-+}_j}(s_j) > \hat{V}^{\pi^{-+}_{j-1}}(s_j)$.
		
	\emph{\textbf{(2)}} $\hat{V}^{\pi^{-+}_j}(s_{j+1}) > \hat{V}^{\pi^{-+}_{j+1}}(s_{j+1})$.
	
	For $\pi^{-+}_n$ only \emph{(1)} is required.

	\label{lemCompPBF2}
\end{lemma}
\begin{proof}
	The additional condition ($0$) here is actually the special case of (2) with $j=0$.
	The complication of logic here is due to the fact that the always-forward policy $\pi^{-+}_{0}$ has a different state value expression compared to other $\allsth{\pbf}$ in $M_C(G=\frac 1 \infty)$.
	Let $\hat{F}_k := \prod_{i=k}^{n-1} \frac{\gamma \hat{p}_i}{1 - \gamma (1-\hat{p}_i)}$ for $k<n$ and $\hat{F}_n := 1$. 
	Consider the dominance relationship between $\pi^{-+}_{0}$ and any $\pi^{-+}_j$ with $j \neq 0$.
	By Lemma \ref{lemExpressionV}, for all $i>j$ we have $\hat{V}^{\pi^{-+}_0}(s_{i}) - \hat{V}^{\pi^{-+}_j}(s_{i}) 
	= \gamma \hat{F_i}(\frac{\hat{F}_1}{1 - \gamma \hat{F}_1} r_G - \frac{1}{1-\gamma} r_D) 
	= \gamma \hat{F_i} (\hat{V}^{\pi^{-+}_0}(s_1) - \hat{V}^{\pi^{-+}_1}(s_1))$.
	Therefore, $\pi^{-+}_0$ dominates any $\pi^{-+}_j$ with $j \neq 0$ if and only if $\hat{V}^{\pi^{-+}_0}(s_1) > \hat{V}^{\pi^{-+}_1}(s_1)$, just as in Lemma \ref{lemCompPBF1} for $\pi^{-+}_0$.
	If this condition is not satisfied, then $\pi^{-+}_{0}$ dominates no one in $\allsth{\pbf}$, and therefore Lemma \ref{lemCompPBF1} for $\pi^{-+}_j$ with $j \neq 0$ can be applied as usual, hence the lemma.
\end{proof}

By Lemma \ref{lemCompPBF1} and \ref{lemCompPBF2}, the occurrence of $\pi$-success event for any policy in $\allsth{\pbf}$ can now be decided simply by comparing one to three pairs of estimated state values.

More importantly, in both $\hat{V}^{\pi^{-+}_j}(s_j) > \hat{V}^{\pi^{-+}_{j-1}}(s_j)$ and $\hat{V}^{\pi^{-+}_j}(s_{j+1}) > \hat{V}^{\pi^{-+}_{j+1}}(s_{j+1})$, one side of the inequalities are always computed using rule (a) or (b) in Lemma \ref{lemExpressionV}, that is, $V^{\pi^{-+}_j}(s_j) = \frac{\gamma^{j-1}}{1-\gamma} r_D$ for $M_C(H=1)$, or $V^{\pi^{-+}_j}(s_j) = \frac{\gamma}{1-\gamma} r_D$ and $V^{\pi^{-+}_1}(s_1) = \frac{1}{1-\gamma} r_D$ for $M_C(H=\infty)$.
There is no estimated transition probability $\hat{P}$ involved in these expressions at all, and thus these estimated state values are constant and always equal to the actual values, regardless of $\hat{P}$.
The other side of the inequalities, on the other hand, are all in the form of $\hat{V}^\pbf(s_{k+1})$ for some $0 \leq k < n$.
Therefore, the $\pi$-success probability $\psucc^\pi$ in $M_C$ are essentially the joint probability that some specific $\hat{V}^\pbf(s_{k+1})$ as functions of random variables $\hat{P}$ exceed or be exceeded by some constants.

To provide readers a more direct impression, the expressions of success probability $\psucc^{\pi^{-+}_0}$ for the always-forward policy $\pi^{-+}_0$ is given in the following theorem.

\begin{theorem}[$\bm{\pi^{-+}_0}$-success Probability] 
	Let $\hat{F}_1 := \prod_{i=1}^{n-1} \frac{\gamma \hat{p}_i}{1 - \gamma (1-\hat{p}_i)}$, then it holds that
	\begin{equation*}
		\psucc^{\pi^{-+}_0}_{M_C(G=1),\ops(m),\tau_m} = \pr(\frac{\hat{F_1}}{1-\gamma} r_G > \frac{1}{1-\gamma}r_D|\allsth{N_{s,a}}) \, \pr(\Etrav|m),
	\end{equation*}
	\begin{equation*}
		\psucc^{\pi^{-+}_0}_{M_C(G=\frac 1 \infty),\ops(m),\tau_m} = \pr(\frac{\hat{F_1}}{1-\gamma F_1} r_G > \frac{1}{1-\gamma}r_D|\allsth{N_{s,a}}) \, \pr(\Etrav|m).
	\end{equation*}
	The visit numbers $\allsth{N_{s,a}}$ here are decided by Lemma \ref{lemExpressionNsa} and Definition \ref{defThreeRulesNsa}, and $\pr(\Etrav|m)$ equals to the one in Theorem \ref{theoTrav}.
	\label{theoPiSPE}
\end{theorem}
\begin{proof}
	By Lemma \ref{lemCompPBF1}, Lemma \ref{lemCompPBF2} and Lemma \ref{lemExpressionV} (a)(b), we have \[\pr(E^{\pi^{-+}_0}|\Etrav,M_C,\ops(m),\tau_m) = \pr(\hat{V}^{\pi^{-+}_0}(s_1) > \frac{1}{1-\gamma}r_D|\allsth{N_{s,a}}).\]
	Then by Lemma \ref{lemExpressionV} (c)(d) and the second level decomposition Lemma \ref{lemL2Decomp}.
\end{proof}

For the sake of space we will not list all expressions of the $\pi$-success probability for $\allsth{\pbf}$ here.
The expressions for other policies in $\allsth{\pbf}$ can be obtained effortlessly by applying Lemma \ref{lemExpressionV} to Lemma \ref{lemCompPBF1} and Lemma \ref{lemCompPBF2} in the same manner as Theorem \ref{theoPiSPE}.

\begin{remark}
	By change $\hat{p}$ back to $p$ in the expressions of Theorem \ref{theoPiSPE}, we get the sufficient and necessary conditions of $r_G$ and $r_D$ that make $\pi^{-+}_0$ to be the actual optimal policy.
	\begin{description}[itemsep=1pt, topsep=4pt]
		\item[(a) $\bm{M_C(G=1)}$] $\frac{F_1}{1-\gamma} r_G > \frac{1}{1-\gamma}r_D$
		
		\item[(b) $\bm{M_C(G=\frac 1 \infty)}$] $\frac{F_1}{1-\gamma F_1} r_G > \frac{1}{1-\gamma}r_D$
	\end{description}
	where $F_1$ is defined as in Lemma \ref{lemExpressionV}. This answers the question of the value setting of $r_G$ and $r_D$ in Section \ref{secProtoChains}.
	\label{remarkRewardSetting}
\end{remark}

Since $\pr(\Etrav|m)$ can be computed by Theorem \ref{theoTrav} without effort, we have eventually transformed the $\pbf$-success probabilities into computing the joint probability of events $\hat{V}^\pbf(s_{k+1})>C$ and $\hat{V}^{\pi^{-+}_{k-1}}(s_{k})<C'$ for some constants $C$ and $C'$ that can also be computed easily by following Lemma \ref{lemExpressionV}.
Therefore, the question remain is whether the cumulative distribution of these $\hat{V}^\pbf(s_{k+1})$ can be computed or not.

The answer is ``Yes, arithmetically''.
By observing (c) and (d) in Lemma \ref{lemExpressionV}, it is clear that all relevant $\hat{V}^\pbf(s_{k+1})$ here are some function of $\hat{F}_{k+1}$, while $\hat{F}_{k+1}$ is a function of $\hat{P}$.
By Rule 2 and 3 of Definition \ref{defThreeRulesNsa}, the probability mass function of $\hat{P}$ is 
\begin{equation*}
	\pr\Big(\hat{p}_i = \frac b {\bar{N}_i^+}\Big) = \binom{\bar{N}_i^+}{b} p_i^b (1-p_i)^{\bar{N}_i^+ - b}.
\end{equation*}
Then the probability mass function of $\hat{F}_{k+1}$ is
\begin{equation*}
	\pr\bigg(\hat{F}_{k+1} = \prod_{i={k+1}}^{n-1} \frac{\gamma \frac{b_i}{\bar{N}_i^+}}{1 - \gamma (1-\frac{b_i}{\bar{N}_i^+})}\bigg) 
	= \prod_{i = k+1}^{n-1} \binom{\bar{N}_i^+}{b_i} p_i^{b_i} (1-p_i)^{\bar{N}_i^+ - {b_i}}.
\end{equation*}
By further substitute $\hat{F}_{k+1}$ for  $\hat{V}^\pbf(s_{k+1})$, we have the probability mass function of $\hat{V}^\pbf(s_{k+1})$ and thus its cumulative distribution function.
For example, in $M_C(G=1)$, it is
\begin{equation*}
	\pr(\hat{V}^\pbf(s_{k+1}) \leq x) = \sum_{\bm{b}} \mathbbm{1}\bigg(x \geq \frac{r_G}{1-\gamma}\prod_{i={k+1}}^{n-1} \frac{\gamma \frac{b_i}{\bar{N}_i^+}}{1 - \gamma (1-\frac{b_i}{\bar{N}_i^+})}\bigg) \prod_{i = k+1}^{n-1} \binom{\bar{N}_i^+}{b_i} p_i^{b_i} (1-p_i)^{\bar{N}_i^+ - {b_i}},
\end{equation*}
where $\mathbbm{1}(X)$ is the indicator function, $\bm{b}=(b_1, b_2, ..., b_{n-1})$ and $b_i$ ranges from $0$ to $\hat{N}^+_i$.
In conclusion, by combining Lemma \ref{lemExpressionV} (the expressions of $\hat{V}$), Lemma \ref{lemExpressionNsa} (the expressions of $\allsth{N_{s,a}}$), Lemma \ref{lemCompPBF1} and Lemma \ref{lemCompPBF2} ($\pbf$-success conditions), the success probabilities $\allsth{\psucc^\pbf}$ are solved from the arithmetic point of view.

\section{A Practical Approximation}
\label{SectionApproxSPE}
Despite having an arithmetic solution already, the actual computation process for $\pbf$-success probabilities is still rather complicated.
Therefore, we introduce a useful approximation to $\pbf$-success probabilities in this section, which is both easier to compute and more intuitive.

By examining (c) and (d) in Lemma \ref{lemExpressionV} more closely, it can be found that except for the case of $k=0$ in $M_C(G=\frac 1 \infty)$, the expressions of $\hat{V}^\pbf(s_{k+1})$ are linear to $\hat{F}_{k+1}$.
Since $\hat{F}_{k}$ itself takes the form of the product of sequence of $\frac{\gamma \hat{p}_i}{1 - \gamma (1-\hat{p}_i)}$, this inspires us to approximate the cumulative distributions of $\hat{V}^\pbf(s_{k+1})$ into log-normal distributions by applying the central limit theorem.

\begin{lemma}
	Let $X_i := \ln \frac{\gamma \hat{p}_i}{1 - \gamma (1-\hat{p}_i)}$, 
	$\mu_i := \mathbb{E}X_i$,
	$\sigma^2_i := \text{Var} \, X_i$, and $k$ be a positive integer.
	Then it holds that 
	\begin{equation*}
		\frac{\sum_{i=k}^{n-1}(X_i - \mu_i)}{\sqrt{\sum_{i=k}^{n-1}\sigma^2_i}} \xrightarrow{d} \mathcal{N}(0,1)
	\end{equation*}
	where $\mathcal{N}(0,1)$ is a standard normal distribution.
	\label{lemLognormal}
\end{lemma}
\begin{proof}
	\textbf{(sketch)} By applying Lindeberg's Central Limit Theorem.
\end{proof}

For the sake of space and relevance, the full proof of this lemma will not be presented here.
The next lemma states that $\hat{F}_{k}$ can be approximated to log-normal distribution.

\begin{lemma}
	If $n$ is sufficiently large, then we have the approximation
	\begin{equation*}
		\hat{F}_{k} \dotsim \ln \mathcal{N}(\sum_{i=k}^{n-1}\mu_i, \sum_{i=k}^{n-1}\sigma^2_i),
	\end{equation*} 
	where $\ln\mathcal{N}(0,1)$ is a standard log-normal distribution such that $\ln X \sim \mathcal{N}(0,1)$ is equivalent to $X \sim \ln\mathcal{N}(0,1)$.
	\label{lemLognormalF}
\end{lemma}
\begin{proof}
	Because $\hat{F_k} = \prod_{i=k}^{n-1} \frac{\gamma \hat{p}_i}{1 - \gamma (1-\hat{p}_i)}$, we have $\ln \hat{F_k} = \sum_{i=k}^{n-1} X_i$.
	Then by Lemma \ref{lemLognormal}, we have $	\frac{\ln \hat{F}_{k} - \sum_{i=k}^{n-1}\mu_i}{\sqrt{\sum_{i=k}^{n-1}\sigma^2_i}} \xrightarrow{d} \mathcal{N}(0,1)$. 
	Therefore, if $n$ is sufficiently large, we have $\ln \hat{F}_{k} \dotsim \mathcal{N}(\sum_{i=k}^{n-1}\mu_i, \sum_{i=k}^{n-1}\sigma^2_i)$, or $\hat{F}_{k} \dotsim \ln \mathcal{N}(\sum_{i=k}^{n-1}\mu_i, \sum_{i=k}^{n-1}\sigma^2_i)$.
\end{proof}

By Lemma \ref{lemLognormalF} and Lemma \ref{lemExpressionV}, $\hat{V}^\pbf(s_{k+1})$ can also be approximated to log-normal, except for the special case of $k=0$ in $M_C(G=\frac 1 \infty)$.
However, from the empirical results in Section \ref{SectionExp} it can be observed that, not only $\hat{V}^\pbf(s_{k+1})$ follows log-normal distribution very closely in general cases, even in the case of $k=0$ in $M_C(G=\frac 1 \infty)$ the approximation is rather good as well.

The remaining task in this section is to obtain the parameters of the log-normal distributions for $\hat{V}^\pbf(s_{k+1})$.
Actually, the parameters of a log-normal distribution can be written as a function of its mean and variance. Therefore, working out the mean and variance of $\hat{V}^\pbf(s_{k+1})$, rather than the original parameters of log-normal distribution, is sufficient for the purpose.

The following lemma, a special case of the delta method \citep{oehlert1992note}, provides a useful tool for approximating the mean and variance of functions random variables.

\begin{lemma}
	Suppose $X$ is a random variable with finite moments, $\mu_X$ being its mean and $\Var X$ being its variance.
	Suppose $f$ is a sufficiently differentiable function.
	Then it holds that
	\begin{equation*}
		\mathbb{E} f(X) \approx f(\mu_X),
	\end{equation*}
	\begin{equation*}
		\Var f(X) \approx f'(\mu_X)^2 \, \Var X.
	\end{equation*}
	\label{lemDeltaMethod}
\end{lemma}
\begin{proof}
	By Taylor's theorem, $f(X) = f(\mu_X) + f'(\mu_X)(X-\mu_X) + \text{Remainder}$.
	Therefore, $\mean f(X) = \mean f(\mu_X) + \mean f'(\mu_X)(X-\mu_X) + \mean[\text{Remainder}] \approx \mean f(\mu_X) + \mean f'(\mu_X)(X-\mu_X) = f(\mu_X) + f'(\mu_X) \mean (X-\mu_X)$.
	Because $\mean (X-\mu_X) = 0$, we have $\mean f(X) \approx f(\mu_X)$.
	
	Similarly, $\Var f(X) \approx \Var f(\mu_X) + \Var f'(\mu_X)(X-\mu_X) = f'(\mu_X)^2 \, \Var(X-\mu_X)$.
	Because $\Var(X-\mu_X) = \Var X$, we have $\Var f(X) \approx f'(\mu_X)^2 \, \Var X.$
\end{proof}

By applying Lemma \ref{lemDeltaMethod} to $\hat{F}_{k+1}$, the mean and the variance of $\hat{F}_{k+1}$ can be derived.
Then by applying the lemma again to $\hat{V}^\pbf(s_{k+1})$, we have the mean and the variance of $\hat{V}^\pbf(s_{k+1})$ that decide the log-normal distribution.
The concrete results are as follows.

\begin{lemma}
	Let $Y_i := \frac{\gamma \hat{p}_i}{1 - \gamma (1-\hat{p}_i)}$, and $C := \frac{1-\gamma}{\gamma}$.
	Then it holds that
	\begin{equation*}
		\mean \hat{F}_{k+1} = \prod_{i=k+1}^{n-1} \mean Y_i
	\end{equation*}
	\begin{equation*}
		\Var \hat{F}_{k+1} = \prod_{i=k+1}^{n-1} \big(
			\Var Y_i + (\mean Y_i)^2
		\big) - \prod_{i=k+1}^{n-1} (\mean Y_i)^2
	\end{equation*}
	where 
	\begin{equation*}
		\mean Y_i \approx \frac{p_i}{p_i + C}
	\end{equation*}
	\begin{equation*}
	 	\Var Y_i \approx \frac{C^2}{(p_i+C)^4} \cdot \frac{p_i(1-p_i)}{\bar{N}^+_i}.
	\end{equation*}
	\label{lemMeanVarF}
\end{lemma}
\begin{proof}
	By Definition \ref{defThreeRulesNsa} and $\hat{p}_i = \frac{N_{i,i}^+}{N_i^+}$, we have $\mean \hat{p}_i = p_i$ and $\Var \hat{p}_i = \frac{p_i(1-p_i)}{\bar{N}^+_i}$.
	Because $Y_i := \frac{\gamma \hat{p}_i}{1 - \gamma (1-\hat{p}_i)} = \frac{\hat{p}_i}{\hat{p}_i + C}$, 
	by applying Lemma \ref{lemDeltaMethod} we have $\mean Y_i \approx \frac{p_i}{p_i + C}$ and $\Var Y_i \approx \frac{C^2}{(p_i+C)^4} \cdot \frac{p_i(1-p_i)}{\bar{N}^+_i}$.
	Because $\allsth{N_{i,i}^+}$ are regarded as independent variables, $\allsth{Y_i}$ are uncorrelated. 
	Therefore $\mean [\prod_{i=k+1}^{n-1} Y_i] = \prod_{i=k+1}^{n-1} \mean Y_i$.
	Also it holds that $\Var [\prod_{i=k+1}^{n-1} Y_i] = \mean[\prod_{i=k+1}^{n-1} Y_i^2] - (\mean[\prod_{i=k+1}^{n-1} Y_i])^2 = \prod_{i=k+1}^{n-1} \mean Y_i^2 - \prod_{i=k+1}^{n-1} (\mean Y_i)^2 = \prod_{i=k+1}^{n-1} \big(\Var Y_i + (\mean Y_i)^2 \big) - \prod_{i=k+1}^{n-1} (\mean Y_i)^2$.
	Since $\hat{F}_{k+1} = \prod_{i=k+1}^{n-1} Y_i$, the lemma is proved.
\end{proof}

\begin{lemma}
	The following holds for $\hat{V}^\pbf(s_{k+1})$ with $0 \leq k < n$ in corresponding $M_C$.
	\begin{description}
		\item[(a) $\bm{M_C(G=1)}$]
		\begin{align*}
			\mean \hat{V}^\pbf(s_{k+1}) &= \frac{r_G}{1-\gamma} \, \mean \hat{F}_{k+1}
		\\
			\Var \hat{V}^\pbf(s_{k+1}) &= \frac{r_G^2}{(1-\gamma)^2} \, \Var \hat{F}_{k+1}.
		\end{align*}
		
		\item[(b) $\bm{M_C(G=\frac 1 \infty)}$, $\bm{\pbf}$ with $\bm{k \neq 0}$]
		\begin{align*}
			\mean \hat{V}^\pbf(s_{k+1}) &= (r_G + \frac{\gamma}{1-\gamma} r_D) \, \mean \hat{F}_{k+1}
		\\
			\Var \hat{V}^\pbf(s_{k+1}) &= (r_G + \frac{\gamma}{1-\gamma} r_D)^2 \, \Var \hat{F}_{k+1}.
		\end{align*}
		
		\item[(c) $\bm{M_C(G=\frac 1 \infty)}$, $\bm{\pi^{-+}_0}$]
		\begin{align*}
			\mean \hat{V}^{\pi^{-+}_0}(s_{1}) &\approx \frac{r_G}{1 - \gamma \, \mean \hat{F}_{1}} \, \mean \hat{F}_{1}
		\\
			\Var \hat{V}^{\pi^{-+}_0}(s_{1}) &\approx \frac{r_G^2}{(1 - \gamma \, \mean \hat{F}_{1}) ^ 4} \, \Var \hat{F}_{1}.
		\end{align*}
	\end{description}
	\label{lemMeanVarV}
\end{lemma}
\begin{proof}
	Proving (a) and (b) are trivial since $\mean[cX] = c\, \mean X$ and $\Var[cX] = c^2\,\Var X$ with $c$ being any constant.
	For case (c), since $V^{\pi^{-+}_0}(s_1) = \frac{F_1}{1-\gamma F_1} r_G$, we have to apply delta method.
	By Lemma \ref{lemDeltaMethod}, $\Var \hat{V}^{\pi^{-+}_0}(s_{1}) \approx ((\hat{V}^{\pi^{-+}_0})' (\mean \hat{F_1}))^2 \,\Var \hat{F}_1 = \big(\frac{r_G}{(1 - \gamma \, \mean \hat{F}_{1}) ^ 2}\big)^2 \, \Var \hat{F}_{1}$.
\end{proof}

By combining Lemma \ref{lemMeanVarF} and Lemma \ref{lemMeanVarV}, we are now able to approximate the mean and the variance of $\hat{V}^\pbf(s_{k+1})$.
Because the probability mass function of $\hat{V}^\pbf(s_{k+1})$ can be approximated by a log-normal, and the probability of $\pi$-success event equals to the joint probability of the satisfaction of conditions in Lemma \ref{lemCompPBF1} and Lemma \ref{lemCompPBF2}, the $\pi$-success probabilities can now be approximated as well.

As an example, the probability of $\pi^{-+}_0$-success event, $\psucc^{\pi^{-+}_0}_{M_C,\ops(m),\tau_m}$, can be approximated by the following theorem.

\begin{theorem}
	Let $\mu' := \ln \frac
			{\mu}
			{  \sqrt{ 1 + \sigma^2 / \mu^2 }  }
		$ and
	$\sigma' := 
		\sqrt{ \ln( 1 + \frac{\sigma^2}{\mu^2} ) }
	$ where $\mu = \mean \hat{V}^{\pi^{-+}_0}(s_{1})$ and $\sigma^2 = \Var \hat{V}^{\pi^{-+}_0}(s_{1})$ are computed as in Lemma \ref{lemMeanVarV}.
	Then it holds that
	\begin{equation*}
		\psucc^{\pi^{-+}_0}_{M_C,\ops(m),\tau_m} \approx \bigg[ 1 - \Phi \bigg(\frac{\ln \frac{r_D}{1-\gamma} - \mu'}{\sigma'} \bigg) \bigg] \prod_{i=1}^{n-1}(1-(1-p_i)^m)
	\end{equation*}
	where $\Phi$ is the cumulative distribution function of the standard normal distribution.
	\label{theoApproxPiSPE}
\end{theorem}
\begin{proof}
	A log-normal distribution with mean $\mu$ and variance $\sigma^2$ has a cumulative distribution of $\pr(X \leq x) = \Phi(\frac{\ln x - \mu'}{\sigma'})$ where $\mu'$ and $\sigma'$ are computed as above.
	By Theorem \ref{theoPiSPE}, we have $\pr(E^{\pi^{-+}_0}|\Etrav,M_C,\ops(m),\tau_m) = \pr(\hat{V}^{\pi^{-+}_0}(s_1) > \frac{1}{1-\gamma}r_D|\allsth{N_{s,a}}) \approx 1 - \Phi \bigg(\frac{\ln \frac{r_D}{1-\gamma} - \mu'}{\sigma'} \bigg)$.
	Then by the second-level decomposition Lemma \ref{lemL2Decomp} and Theorem \ref{theoTrav}.
\end{proof}

Again for the sake of space we will not fully elaborate the approximation results for the entire $\allsth{\pbf}$ family.
The remaining ones can be easily obtained by replacing the $\pi$-success conditions according to Lemma \ref{lemCompPBF1} and Lemma \ref{lemCompPBF2}.
The computational complexity of this approximation is only $O(n)$, and therefore is more computationally efficient than applying Theorem \ref{theoPiSPE} directly be computing the exact distribution of $\hat{V}^\pbf(s_{k+1})$.

\section{Empirical Verification}
\label{SectionExp}
We conducted several experiments to verify our main results, in particular Theorem \ref{theoTrav} for the probability of traverse event, Lemma \ref{lemExpressionNsa} for the expected visit numbers, and Theorem \ref{theoApproxPiSPE} for the approximation of $\pi$-success probabilities $\psucc^\pi$.

The experiments were carried out by executing the Optimistic Prototype Strategy (OPS) in prototype chains.
The OPS was implemented based on R-MAX algorithm \citep{brafman2002r,kakade2003sample}.
The only difference between the OPS and the original R-MAX is that the former has to decide a final output policy while the latter has not.
Therefore, in OPS, when the number of time steps reaches $\tau_m$, i.e. when the optimistic exploration ends, an additional Value Iteration process will be executed with all observations collected so far as its input, in order to decide the final output policy.
This modify does not actually change the exploration behaviour of R-MAX.
The implemented OPS is also behaviourally indistinguishable to the one defined in Definition \ref{defProtoStrategy}. 
The only parameter of OPS is $m$, just as in its own definition, and as in the original R-MAX.

In all of the following experiments, the discount factor $\gamma$ was set to 0.998, and the stopping criteria of Value Iteration algorithm was $\textit{Bellman residual}<10^{-6}$.
The goal reward $r_G$ was set to 1, while the distracting reward $r_D$ was set to 0.001, unless otherwise stated.
According to Remark \ref{remarkRewardSetting}, this setting is sufficient to ensure the always-forward policy $\pi^{-+}_0$ be the optimal one.

As for the forward transition probability $\allsth{p_i}$,
two sets of settings are used.
In the first setting, $\allsth{p_i}$ was set the same fixed value for all states; the specific values used will be stated later.
In the second setting, $\allsth{p_i}$ took first $n-1$ numbers from a pre-generated table of random real numbers, i.e. $p_i = \text{table}_i$.
The numbers in this table were uniformly randomly chosen between $0.3$ and $0.7$.
To keep the results comparable, the table was created before the following experiments, and had never been changed thereafter.

Because our experiments involved comparing empirical distributions of binary random variable to the theoretical ones, all experiments were repeated for $1,000$ times to ensure that the empirical distributions can properly reflect the actual ones.

\subsection{Verification for Traverse Probability}
\label{secVerifyTrav}
The first theoretical result to be verified is the expression for the traverse probability in Theorem \ref{theoTrav}.
What the theorem claims is rather straightforward: in all four prototype chains, the traverse probability for $\ops(m)$ at $\tau_m$ is $\prod_{i=1}^{n-1}(1-(1-p_i)^m)$, regardless of the hazardousness $H$ or the goal productivity $G$ of the chains.

\begin{figure}
	\centering
	{\includegraphics[scale=0.7,trim=1cm 9.4cm 1cm 9.2cm]{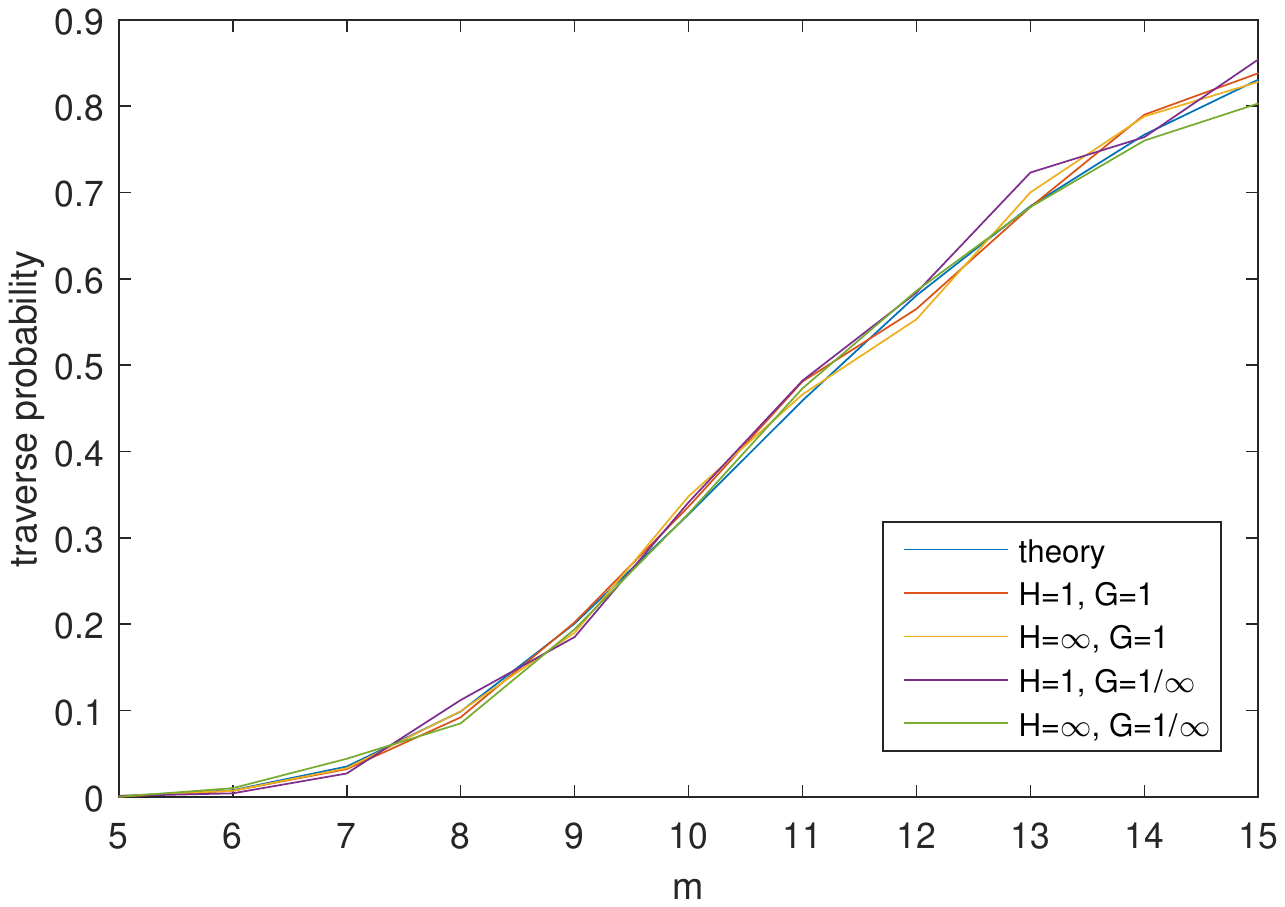}}
	\caption{Traverse probability with different $m$ under $n=40$, $\allsth{p_i}=0.3$.}
	\label{figTravTest3}
\end{figure}

We executed the OPS on all four prototype chains $M_C$ with length $n=40$ and $\allsth{p_i}=0.3$.
The exploration parameter $m$ of OPS was set from 5 to 15.
Each pair of $(M_C, m)$ were executed for 1000 times, and each run was terminated at $\tau=\tau_m$ or exceeding the maximum step budget of $300,000$.

The experimental results are shown in Figure \ref{figTravTest3}.
As can be seen from the figure, the difference between the traverse probabilities in four prototype chains is negligible, and all experimental results are very close to their theoretical values.
Friedman test was also carried out to test the null hypothesis that the results of the four prototype chains and the theoretical values come from the same distribution.
The resulting p-value was 0.8397, and thus the null hypothesis was not rejected, meaning that there is no significant difference between the traverse probability of four prototype chains and their theoretical values.

\begin{figure}
	\centering
	{\includegraphics[scale=0.7,trim=1cm 9.4cm 1cm 9.2cm]{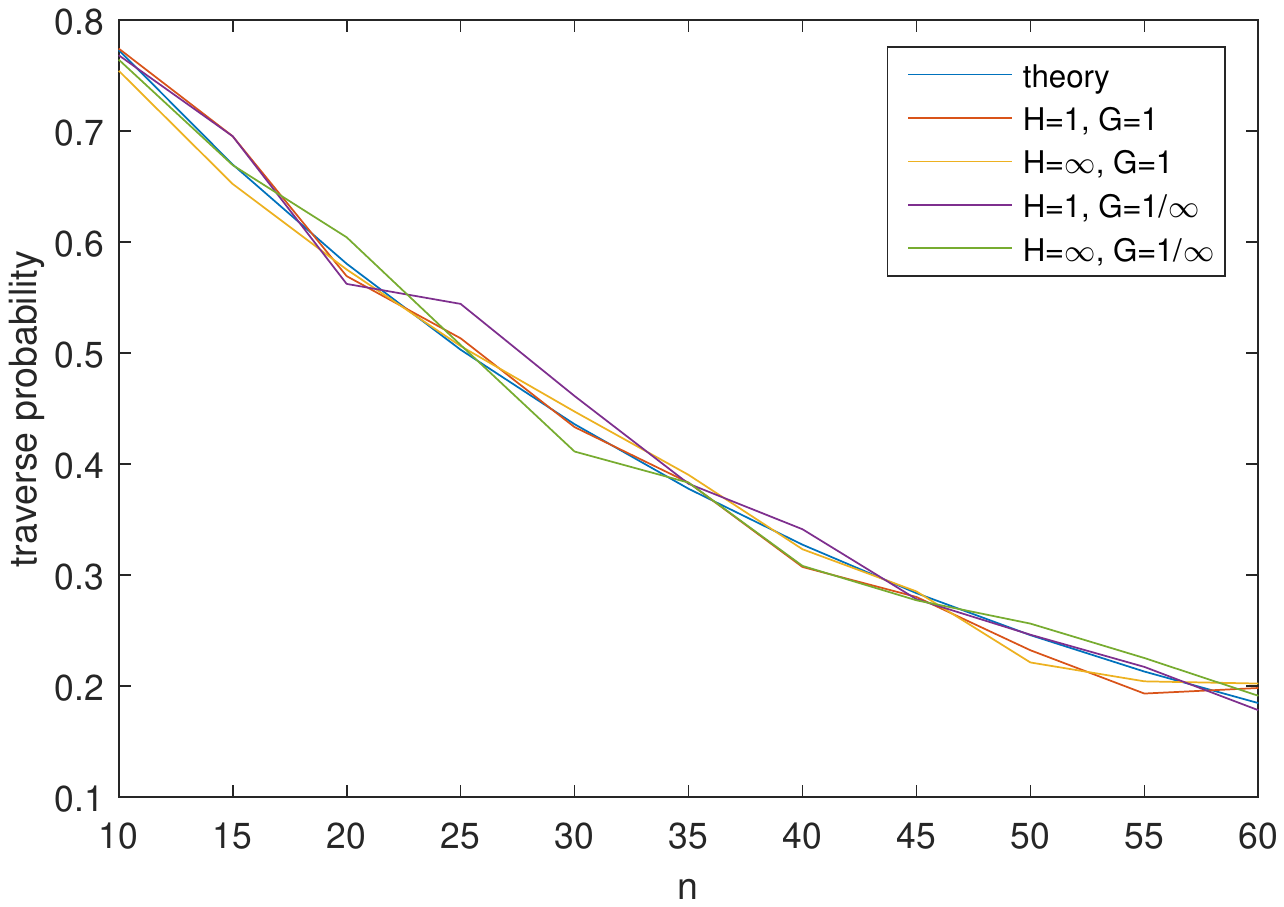}}
	\caption{Traverse probability with different $n$ under $m=10$ and $\allsth{p_i}=0.3$.}
	\label{figTravTest2}
\end{figure}

It is also interesting to see if the theoretical values are correct under different chain length.
The setting of $\allsth{p_i}=0.3$ was still used, while the length of chains $n$ was set as $10, 15, 20, ..., 60$.
The exploration parameter $m$ was set to 10.

The results are shown in Figure \ref{figTravTest2}.
Again, no significant difference between the actual traverse probabilities in four chains and their theoretical values were observed.
The result of Friedman test was 0.9665, failing to reject the null hypothesis that there is no difference between the actual and theoretical values.

\begin{figure}
	\centering
	{\includegraphics[scale=0.7,trim=1cm 9.4cm 1cm 9.2cm]{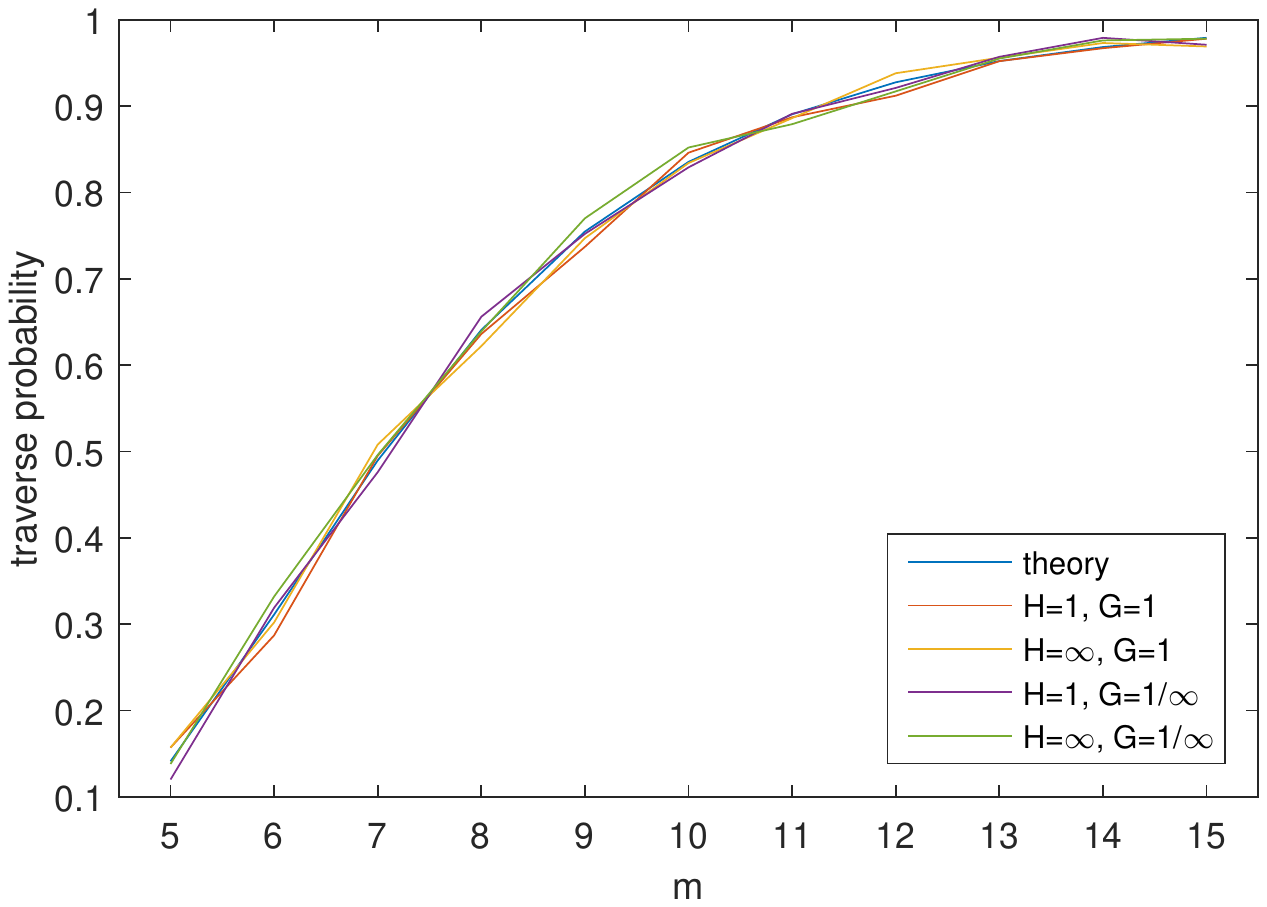}}
	\caption{Traverse probability with different $m$ under $n=40$, random $\allsth{p_i}$.}
	\label{figTravTest}
\end{figure}

Experiments on the chains with varying $\allsth{p_i}$ were conducted as well.
In these experiments, then length of chain was set to 40, and $\allsth{p_i}$ for four chains were set according to the same pre-generated random number table, as mentioned before.
The traverse probabilities under different $m$ (again, 5 to 15) were compared.

The results are shown in Figure \ref{figTravTest}.
Once more, there were no significant difference between the theoretical and the actual traverse probabilities.
The p-value of Friedman test was 0.4552, thus the null hypothesis claiming no difference was not rejected.

The above three sets of experimental results demonstrate that Theorem \ref{theoTrav} provides highly accurate predictions to the traverse probabilities.
Additionally, there are some intuitions that can be obtained from the results.
Figure \ref{figTravTest3} and Figure \ref{figTravTest} show that, in the same chain MDP, encouraging optimistic exploration by increasing $m$ will increase the traverse probability, and the rate of increase appears to be an s-shaped curve.
Figure \ref{figTravTest2}, on the other hand, shows that under the same degree of optimism, the traverse probability in larger chain MDPs are less than the smaller ones, and it can drop quite quickly (from $0.77$ to $0.18$ in our experiments) with the increase of scale.

Although the qualitative part of these intuitions can be drawn from general experiences easily, the quantitative part are not that obvious, especially for the rates of change.
By knowing the rates of change, practitioners are now able to reasonably make trade-off between the learning costs and the gain of such costs.
More importantly, because the theoretical values are highly accurate, these intuitions can be inferred directly from these theoretical values without any actual runs of experiment, eliminating the need of trial-and-error on the parameters.
This can be very helpful in real-world applications where interacting with the environment can be very expensive.

\subsection{Verification for Visit Numbers}
\label{secVerifyNsa}
The second theoretical result to be examined is Lemma \ref{lemExpressionNsa}, the expressions for expected visit numbers $\allsth{\hat{N}_s,a}$ at $\tau_m$ given that the traverse event has occurred.
The second and the third rules in Definition \ref{defThreeRulesNsa} should also be verified to see if 
these abstraction are able to conserve the dispersion of the visit numbers.
Because the occurrence of traverse event is a precondition of these results, all runs without occurrence of traverse event in this subsection were re-executed until the traverse event happens.

\begin{figure}
	\centering
	\begin{subfigure}[t]{0.46\textwidth}
		{\includegraphics[width=\textwidth,trim=6cm 9.5cm 6cm 9.5cm]{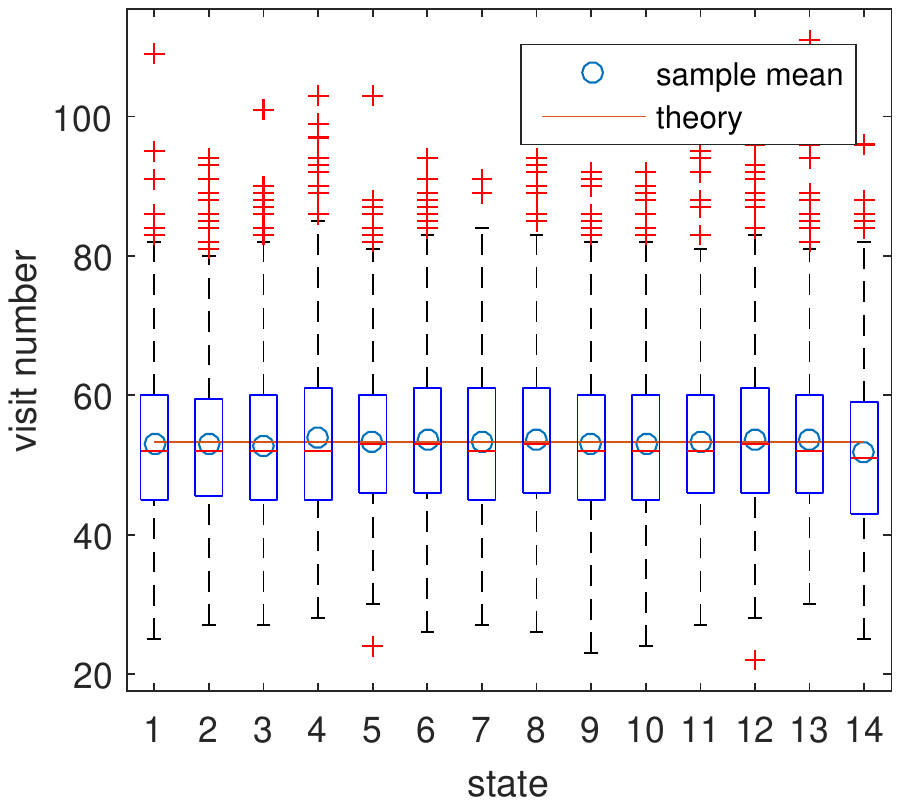}}
		\caption{$M_C(H=1, G=1)$}
	\end{subfigure}
	\begin{subfigure}[t]{0.46\textwidth}
		{\includegraphics[width=\textwidth,trim=6cm 9.5cm 6cm 9.5cm]{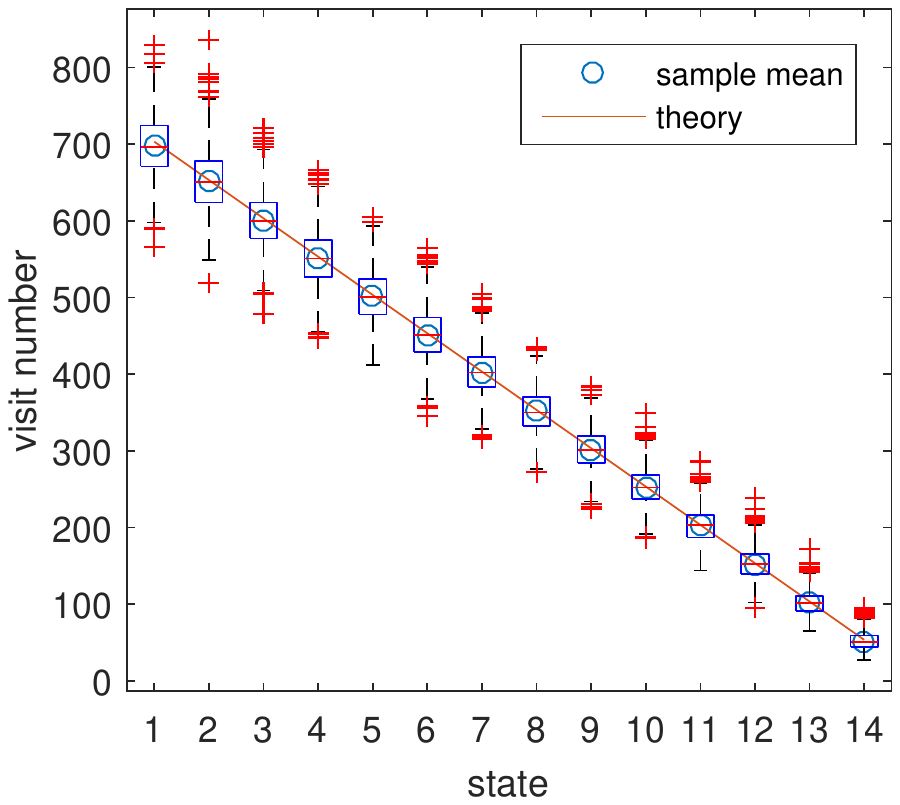}}
		\caption{$M_C(H=\infty, G=1)$}
	\end{subfigure}
	\begin{subfigure}[t]{0.46\textwidth}
		\includegraphics[width=\textwidth,trim=6cm 9.5cm 6cm 9.5cm]{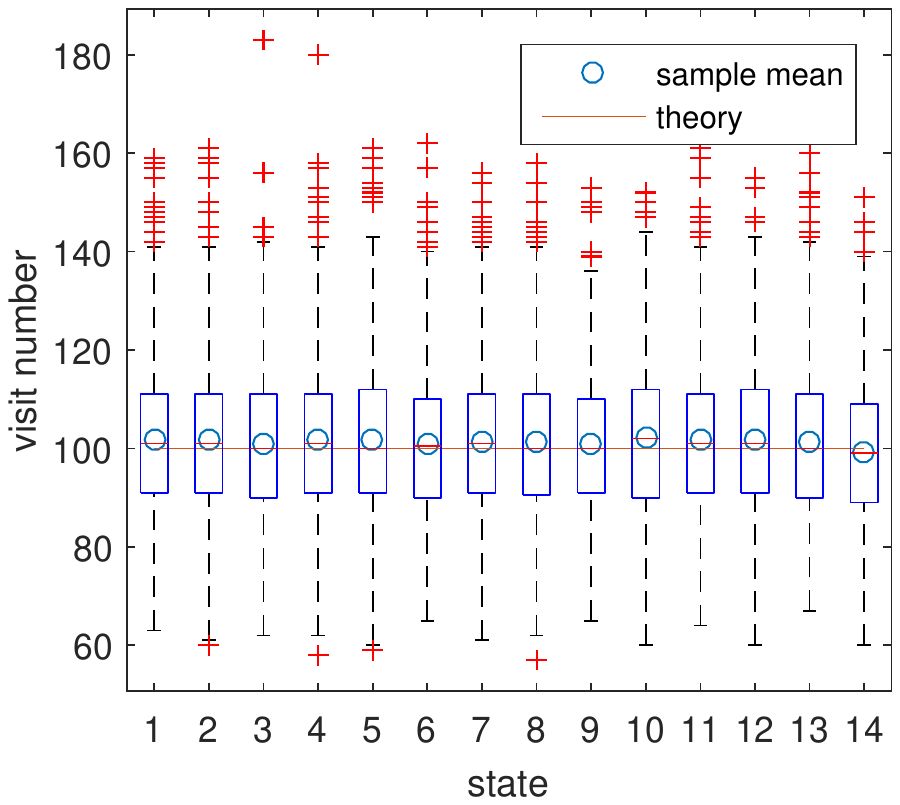}
		\caption{$M_C(H=1, G=1/\infty)$}
	\end{subfigure}
	\begin{subfigure}[t]{0.46\textwidth}
		\includegraphics[width=\textwidth,trim=6cm 9.5cm 6cm 9.5cm]{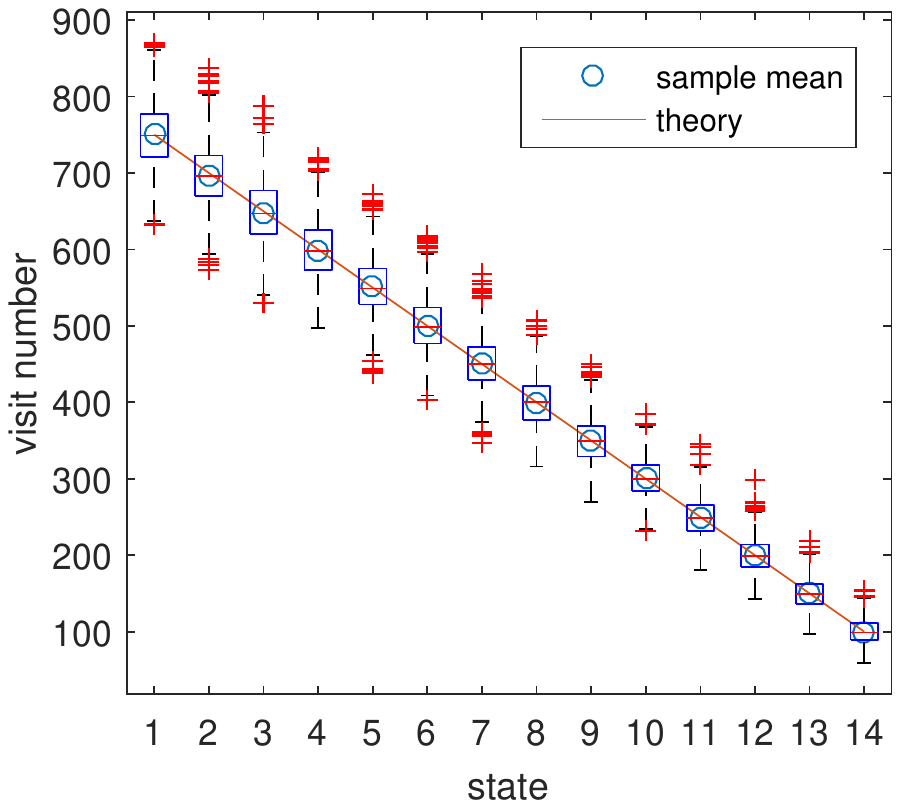}
		\caption{$M_C(H=\infty, G=1/\infty)$}
	\end{subfigure}
	\caption{Theoretical and empirical $\allsth{\bar{N}_i^+}$ in chains with $n=15$ and $\allsth{p_i}=0.3$.}
	\label{figNsa}
\end{figure}

According to Lemma \ref{lemExpressionNsa}, the expected visit numbers of backward actions and goal actions are trivially $m$, while the visit numbers of forward actions, $\bar{N}_i^+$, have different expressions in different prototype chains.
An interesting observation from these expressions is, $\bar{N}_i^+$ in $M_C(H=1)$ are invariant to the state position $i$, while in $M_C(H=\infty)$ they are linear to $i$.
If the forward transition probabilities $\allsth{p_i}$ are set identical, then $\bar{N}_i^+$ in $M_C(H=1)$ should have the same value, while in $M_C(H=\infty)$ they should be in a straight line.
Therefore, we executed OPS with $m=15$ in four prototype chains with $n=15$ and $\allsth{p_i}=0.3$ to see whether these properties actually exist.

The experimental results are shown in Figure \ref{figNsa}.
The distribution of the actual visit numbers in 1,000 runs are displayed by the box plots, and the sample means are marked by the circles.
The theoretical values are drawn in the solid line.
It is clear from the results that the expressions of expected visit numbers in Lemma \ref{lemExpressionNsa} are in accordance with the reality.
The observation mentioned above are confirmed as well: $\bar{N}_i^+$ remained the same in Figure \ref{figNsa} (a) and (c), while decreased linearly to the state position $i$ in (b) and (d).

Another important observation from Figure \ref{figNsa} is that all actual visit numbers $\allsth{N_i^+}$ showed certain degree of dispersion.
According to Rule 2 and Rule 3 in Definition \ref{defThreeRulesNsa}, the actual visit numbers $\allsth{N_i^+}$ should be dealt as if they were fixed to $\allsth{\bar{N}_i^+}$, while the dispersion of estimated transition probabilities $\allsth{\hat{p}_i}$ come from the binomial distribution followed by the transition visit numbers $\allsth{\bar{N}_{i,i}^+}$.
Since the impact of the dispersion on later analysis is propagated through $\allsth{\hat{p}_i}$, it is crucial to verify whether the dispersion are preserved through the abstraction of Definition \ref{defThreeRulesNsa}.

\begin{figure}
	\centering
	\begin{subfigure}[t]{0.46\textwidth}
		{\includegraphics[width=\textwidth,trim=6cm 9.6cm 6cm 9.0cm]{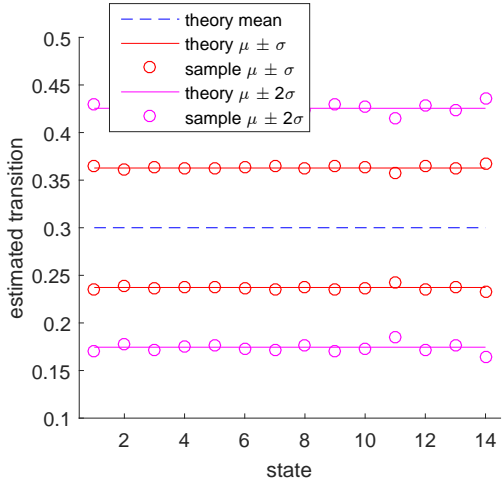}}
		\caption{$M_C(H=1, G=1)$}
	\end{subfigure}
	\begin{subfigure}[t]{0.46\textwidth}
		{\includegraphics[width=\textwidth,trim=6cm 9.6cm 6cm 9.0cm]{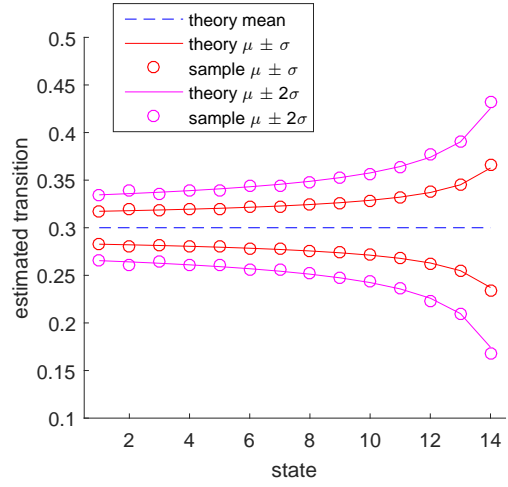}}
		\caption{$M_C(H=\infty, G=1)$}
	\end{subfigure}
	\begin{subfigure}[t]{0.46\textwidth}
		\includegraphics[width=\textwidth,trim=6cm 9.6cm 6cm 9.0cm]{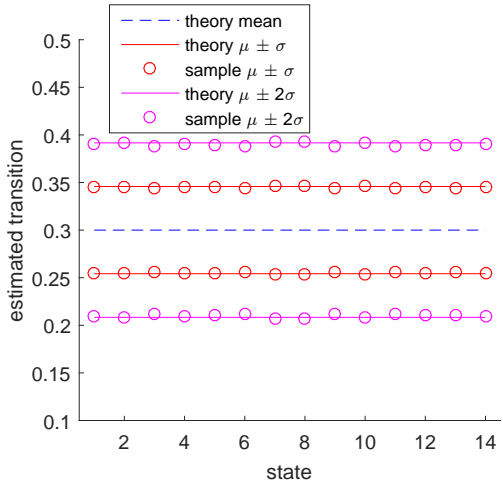}
		\caption{$M_C(H=1, G=1/\infty)$}
	\end{subfigure}
	\begin{subfigure}[t]{0.46\textwidth}
		\includegraphics[width=\textwidth,trim=6cm 9.6cm 6cm 9.0cm]{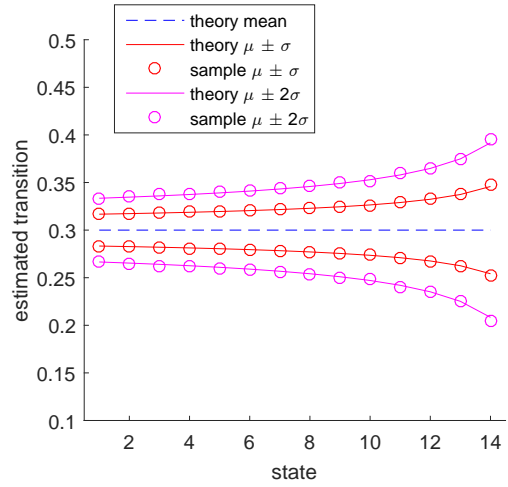}
		\caption{$M_C(H=\infty, G=1/\infty)$}
	\end{subfigure}
	\caption{Theoretical and empirical dispersion of $\allsth{\hat{p}_i}$ originated from $\allsth{N_{s,a}}$.}
	\label{figNsaDisp}
\end{figure}

Therefore, we calculated the standard deviation of  $\allsth{\hat{p}_i}$ from the experimental data above and compared them to the theoretical ones.
Specifically, according to Definition \ref{defThreeRulesNsa}, we have $N_{i,i}^+ \sim \text{Binomial}(\bar{N}_i^+, p_i)$ and  $\hat{p}_i = \frac{N_{i,i}^+}{\bar{N}_i^+}$, hence the theoretical standard deviation of $\hat{p}_i$ is $\sigma_{\hat{p}_i} = \sqrt{\frac{p_i(1-p_i)}{\hat{N}_i^+}}$.
This should match the empirical standard deviations of $\allsth{\hat{p}_i}$ if the abstraction of Definition \ref{defThreeRulesNsa} preserves the dispersion.

The results are shown in Figure \ref{figNsaDisp}.
In each plot, the dashed line in the middle is the theoretical mean $\mean \hat{p}_i = p_i$, and the pair of lines immediately above and below indicates the theoretical interval of $p_i \pm \sigma_{\hat{p}_i}$, while the outermost pair indicates $p_i \pm 2\sigma_{\hat{p}_i}$.
The empirical standard deviation from the experimental data are marked by the circles.
Clearly seen from the figure, the rules of Definition \ref{defThreeRulesNsa} effectively preserves the dispersion of estimated transitions originated from the visit numbers, which is exactly what we need for further analysis on $\pi$-success probabilities.

\subsection{Verification for Success Probability}
Now it comes to our primary result, the success probability, to be verified.
According to Theorem \ref{theoDecomp}, the $\varepsilon$-success probability $\psucce$ is merely a plain sum of $\pi$-success probabilities $\psucc^\pi$ over the set of relevant policies.
Since the $\pi$-success probabilities themselves are the product of the traverse probability $\pr(\Etrav^\pi)$ and the conditional probability $\pr(E^\pi|\Etrav^\pi)$,
and the former has already been verified in Section \ref{secVerifyTrav}, the latter is of main concern in this subsection.

As indicated in the conditions of $\pbf$-success (Lemma \ref{lemCompPBF1} and Lemma \ref{lemCompPBF2}), these probabilities are equivalent to the ones of which some certain estimated state values $\Vbf(s_{k+1})$ exceeding or being exceeded by some constants.
Section \ref{SectionApproxSPE} further provides a log-normal approximation to the cumulative distribution of $\Vbf(s_{k+1})$ to simplify the computation of theoretical $\pi$-success probabilities.
Therefore, our experiments were focused on verifying the goodness of this approximation to the distribution of $\Vbf(s_{k+1})$.

Experiments were conducted by executing OPS with $m$ ranged from 8 to 20 on four prototype chains.
As in Section \ref{secVerifyNsa}, the runs with traverse event not occurred were re-executed in order to block the effect of the possible absence of traverse event.
Nevertheless, we set length $n=20$ and the fixed transition probability $\allsth{p_i}=0.5$ to raise the probability of traverse event.
The same experiments were also conducted in four prototype chains with their $\allsth{p_i}$ set according to the pre-generated random number table.
The estimated state values $\hat{V}^{\pi^{-+}_0}(s_{1})$ at $\tau_m$ were collected for all these experiments.

\begin{figure}
	\centering
	\begin{subfigure}[t]{0.48\textwidth}
		{\includegraphics[width=\textwidth,trim=4cm 8.8cm 4cm 8.0cm]{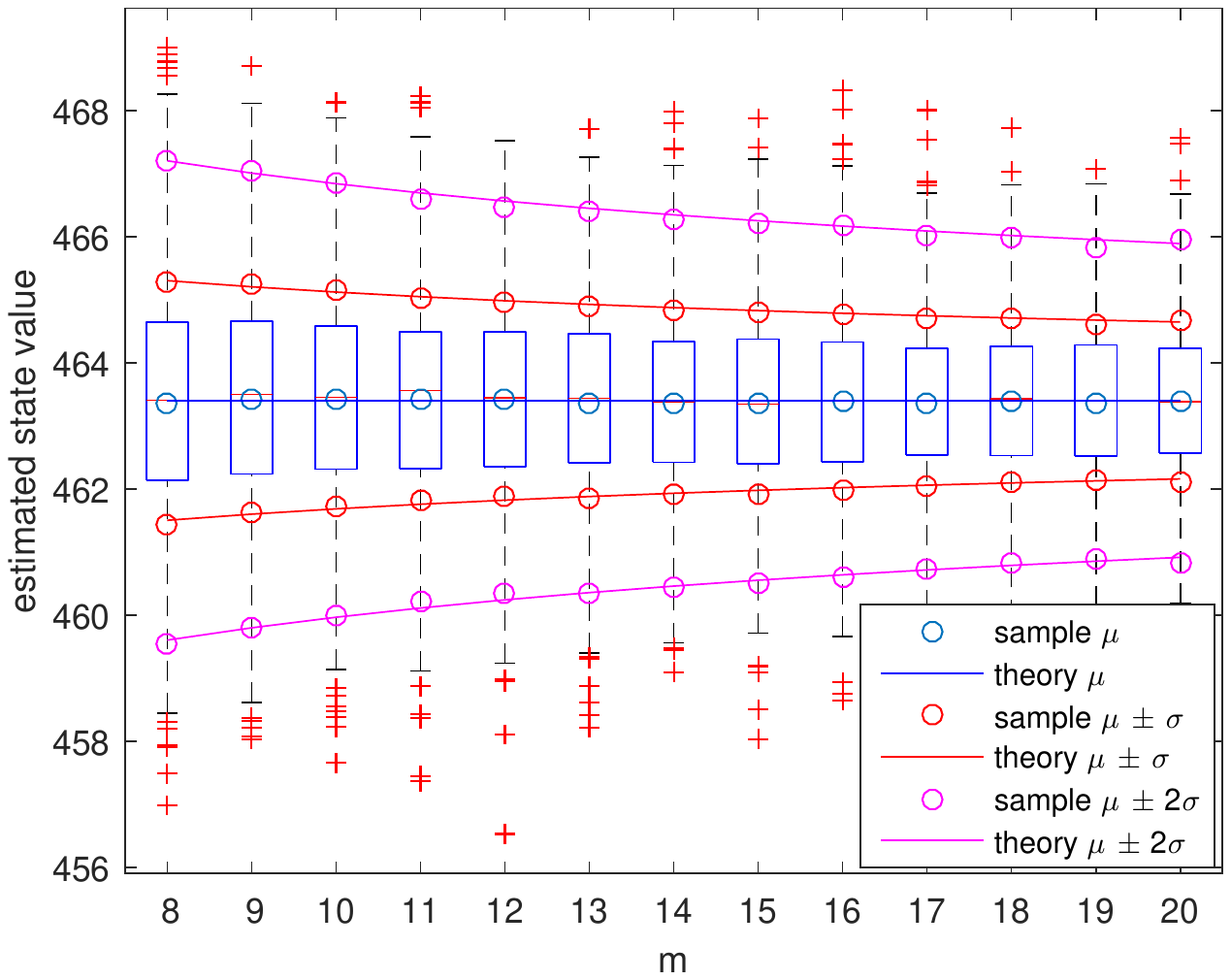}}
		\caption{$M_C(H=1, G=1)$}
	\end{subfigure}
	\begin{subfigure}[t]{0.48\textwidth}
		{\includegraphics[width=\textwidth,trim=4cm 8.8cm 4cm 8.0cm]{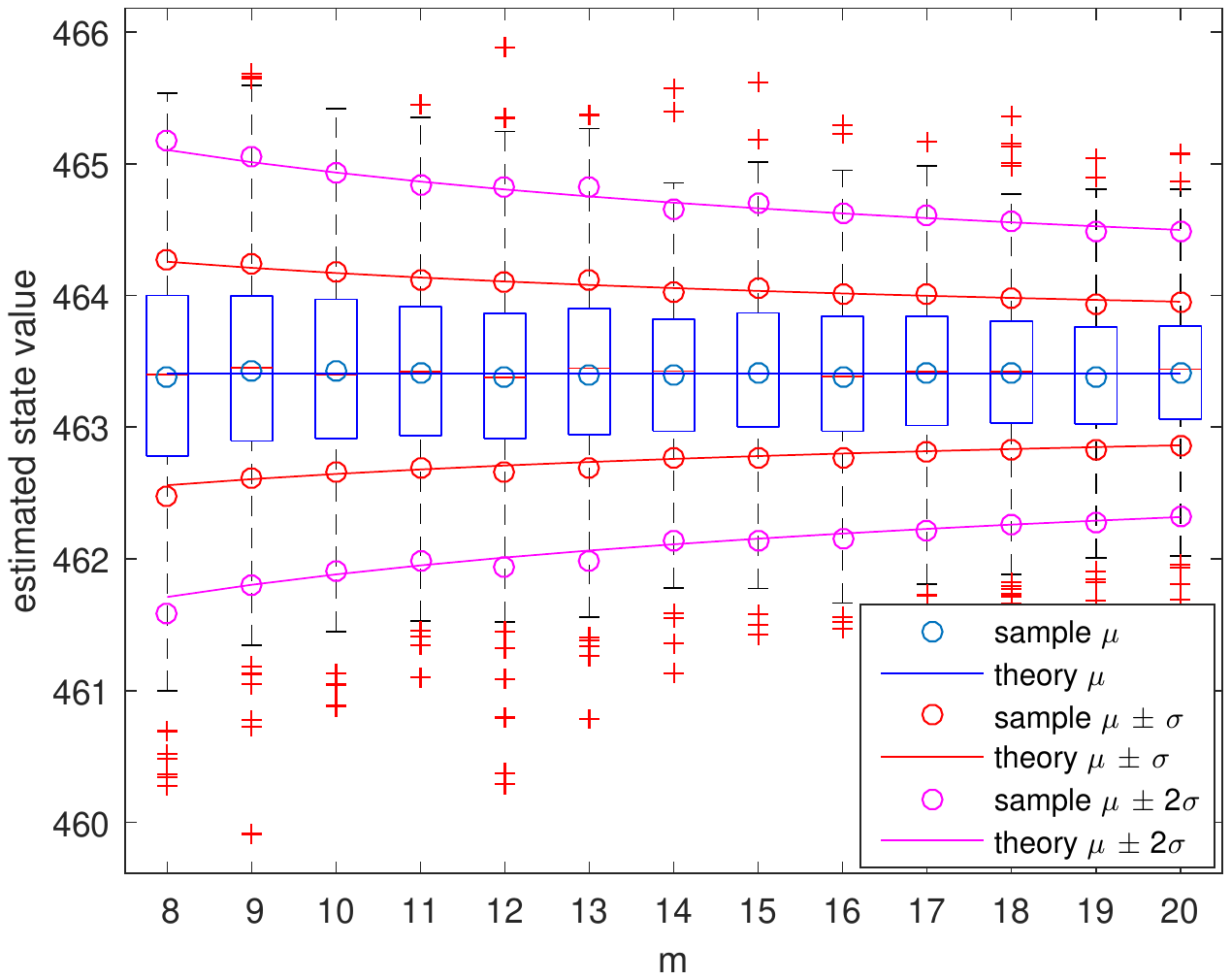}}
		\caption{$M_C(H=\infty, G=1)$}
	\end{subfigure}
	\begin{subfigure}[t]{0.48\textwidth}
		\includegraphics[width=\textwidth,trim=4cm 8.8cm 4cm 8.0cm]{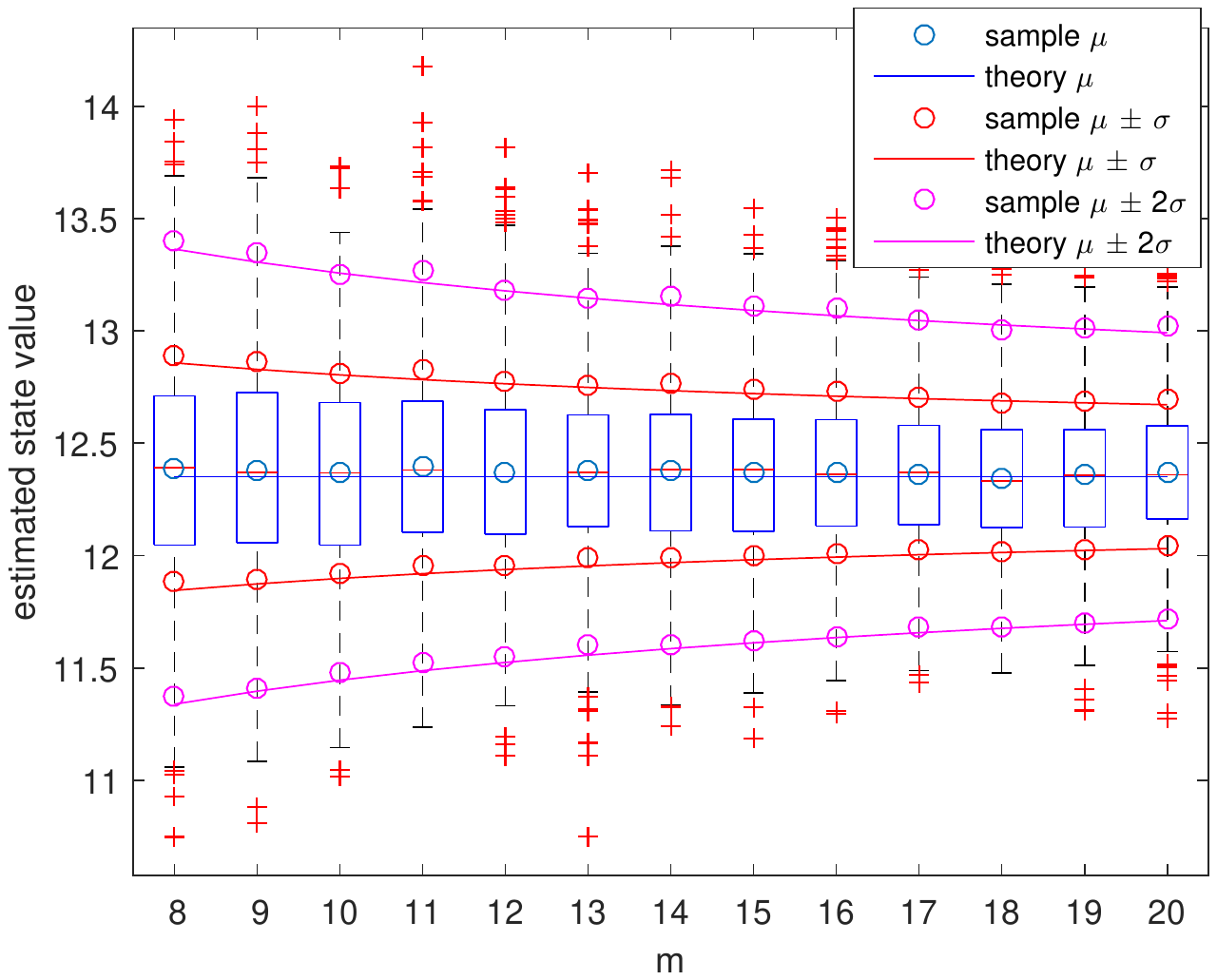}
		\caption{$M_C(H=1, G=1/\infty)$}
	\end{subfigure}
	\begin{subfigure}[t]{0.48\textwidth}
		\includegraphics[width=\textwidth,trim=4cm 8.8cm 4cm 8.0cm]{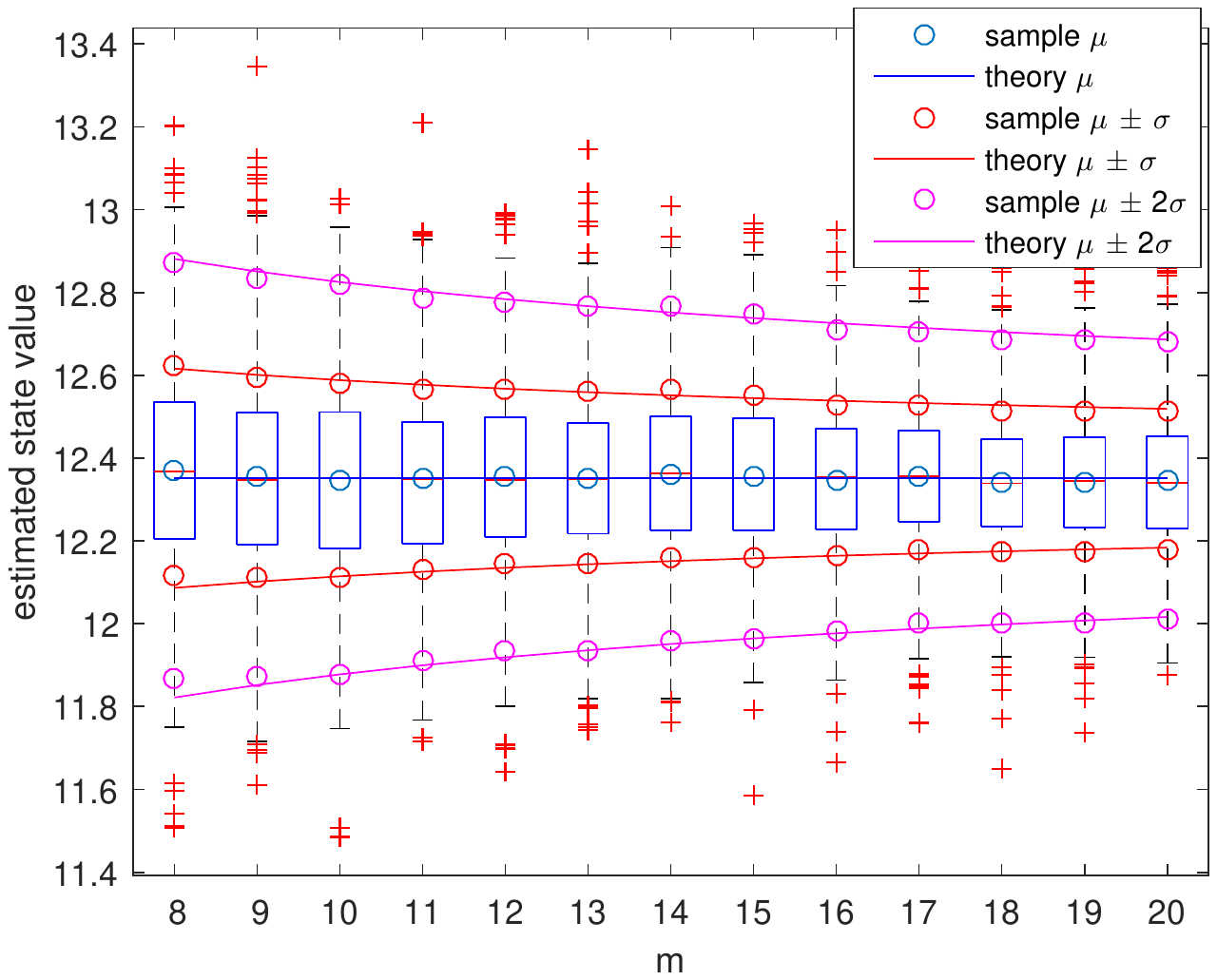}
		\caption{$M_C(H=\infty, G=1/\infty)$}
	\end{subfigure}
	\caption{Distribution of $V^{\pi^{-+}_0}(s_{1})$ with different $m$ under $n=20$ and $\allsth{p_i}=0.5$.}
	\label{figVbox}
\end{figure}

\begin{figure}
	\centering
	\begin{subfigure}[t]{0.48\textwidth}
		{\includegraphics[width=\textwidth,trim=4cm 8.8cm 4cm 8.0cm]{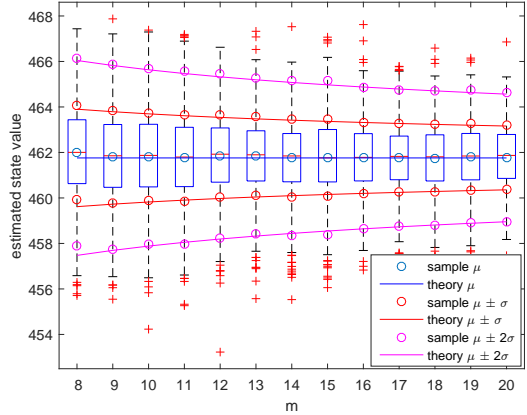}}
		\caption{$M_C(H=1, G=1)$}
	\end{subfigure}
	\begin{subfigure}[t]{0.48\textwidth}
		{\includegraphics[width=\textwidth,trim=4cm 8.8cm 4cm 8.0cm]{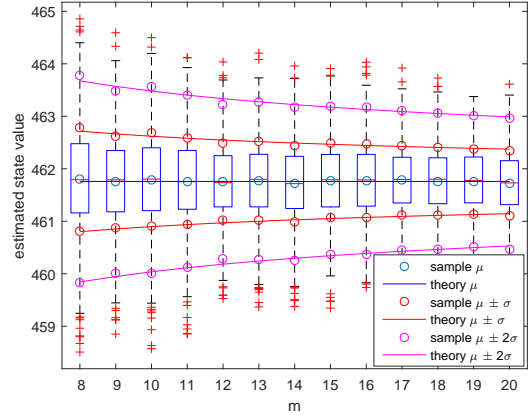}}
		\caption{$M_C(H=\infty, G=1)$}
	\end{subfigure}
	\begin{subfigure}[t]{0.48\textwidth}
		\includegraphics[width=\textwidth,trim=4cm 8.8cm 4cm 8.0cm]{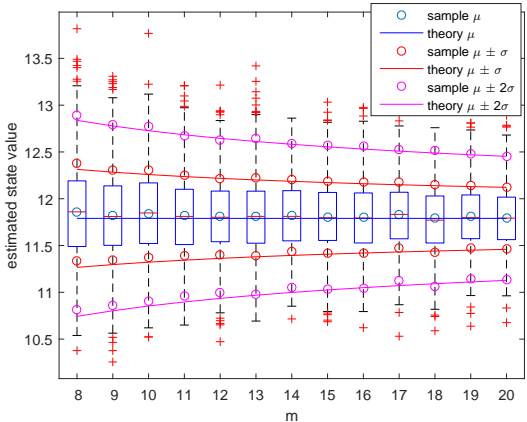}
		\caption{$M_C(H=1, G=1/\infty)$}
	\end{subfigure}
	\begin{subfigure}[t]{0.48\textwidth}
		\includegraphics[width=\textwidth,trim=4cm 8.8cm 4cm 8.0cm]{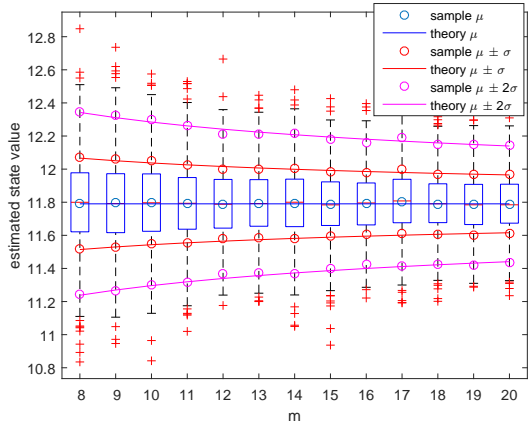}
		\caption{$M_C(H=\infty, G=1/\infty)$}
	\end{subfigure}
	\caption{Distribution of $V^{\pi^{-+}_0}(s_{1})$ with different $m$ under $n=20$ and random $\allsth{p_i}$.}
	\label{figVbox2}
\end{figure}

The theoretical and empirical distribution of $\hat{V}^{\pi^{-+}_0}(s_{1})$ are shown in Figure \ref{figVbox} for identical $\allsth{p_i}=0.5$, and in Figure \ref{figVbox2} for random $\allsth{p_i}$.
The empirical distributions are displayed in box plots.
Additionally, their sample mean and the points one and two standard deviation away are marked by circle.
Their corresponding theoretical values are drawn in solid lines.

No significant difference between the theoretical and empirical values can be observed from the figures.
One-sample Kolmogorov-Smirnov test was conducted to decide whether the null hypotheses that the data came from log-normal distributions with theoretical parameters suggested by Theorem \ref{theoApproxPiSPE} can be rejected or not.
The result was, at the 1\% significance level, none of the null hypotheses with different $m$ on different chains were rejected for all $m$ and all chains with $\allsth{p_i}=0.5$.
For chains with $\allsth{p_i}$ taken from the random number table, four rejects occurred within $4\times(20-8+1)=52$ groups of data.
One of them was $m=8$ in $M_C(H=1,G=1)$, and the rest were $m=8,10,17$ in $M_C(H=1,G=\frac 1 \infty)$. 

This is a reasonable result because as discussed in Section \ref{SectionApproxSPE}, according to Lemma \ref{lemLognormalF} and Lemma \ref{lemExpressionV}, $\hat{V}^{\pi^{-+}_0}(s_{1})$ approximates log-normal if $n$ is large enough, and if it is not the case that $G=\frac 1 \infty$.
Since no rejection occurred in $M_C(H=\infty,G=\frac 1 \infty)$, and only 3 rejections occurred within 13 groups of data in $M_C(H=1,G=\frac 1 \infty)$, it can be concluded that the log-normal approximation is highly accurate for $G=1$, and sufficiently accurate even for $G=\frac 1 \infty$.

Although the experiments above are already sufficient to support our theoretical results, it is still interesting to see how the actual success probability curve looks like.
Therefore, we conducted experiments to compare the theoretical success probability, approximated by cumulative distribution function of log-normal in Theorem \ref{theoApproxPiSPE}, with the empirical success probability of OPS running in prototype chains with $n=20$, $\allsth{p_i}=0.5$.
In order to increase the chance of the agent being distracted by backward actions so that it could be observed clearly, the distracting reward $r_D$ was set to the $0.993(1-\gamma)V^{\pi^{-+}_0}(s_{1})$.
Under this setting, the value $\hat{V}^{\pi^{-+}_0}(s_{1})$ estimated by the agent must be greater than $99.3\%$ of its real value, or the percent error must be less than $0.7\%$, to achieve a strict success, or otherwise the agent will be distracted by the backward actions.

\begin{figure}
	\centering
	{\includegraphics[scale=0.7,trim=1cm 9cm 1cm 9.2cm]{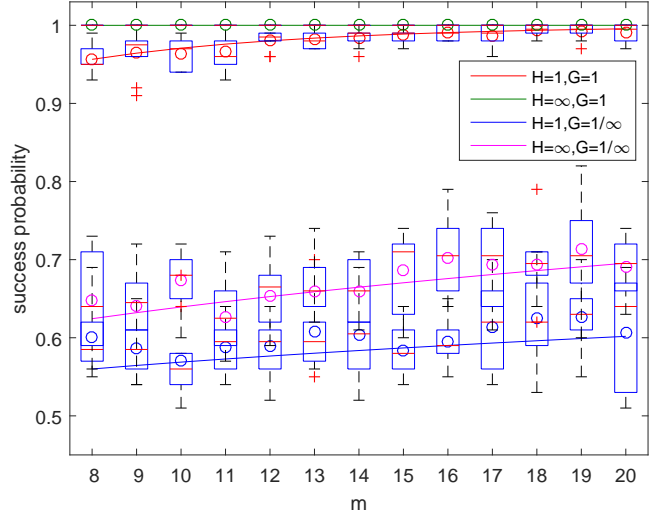}}
	\caption{Sample and theoretical $\pi^{-+}_0$-success probability with different $m$ under $n=20$, $\allsth{p_i}=0.5$, $r_D=0.993(1-\gamma)V^{\pi^{-+}_0}(s_{1})$.}
	\label{figSuccessProb}
\end{figure}

The experimental results is shown in Figure \ref{figSuccessProb}.
The sample success probability was computed from dividing the number of successful runs by the total number of runs 1000, and the results are marked by circles in the figure.
Additionally, each set of 1000 runs was uniformly separated to 10 groups, each with 100 runs, to draw the box plot for the success probability.
The approximated theoretical probabilities by Theorem \ref{theoApproxPiSPE} are displayed by the solid lines.
The four prototype chains respectively lying from the top to the bottom are $M_C(H=\infty,G=1)$, $M_C(H=1,G=1)$, $M_C(H=\infty,G=\frac 1 \infty)$, and $M_C(H=1,G=\frac 1 \infty)$.

From Figure \ref{figSuccessProb} it can be observed that our theoretical success probabilities can properly reflect their actual values.
There seems to be a slight underestimate of success probability for the theoretical results, most visible for $M_C(H=1, G=\frac 1 \infty)$.
The reason for this underestimate may be the inaccuracy of planning algorithm.
Actually, because the planning algorithm (Value Iteration in our experiments) was halted whenever Bellman residual was less than $10^{-6}$, the agent could not really tell the difference between the actions when their estimated values were very close.
This could increase the chance of fake success where if the stopping criteria was set smaller, the output policy could change.
Since this additional chance of success is not significant according the experimental results, it does not affect the validity of our approach.

More importantly, Figure \ref{figSuccessProb} demonstrates that the formulation of success probability  of exploration itself is helpful to capture and distinguish the difficulty of exploration in different MDPs.
Clearly, under the same degree of optimistic exploration, represented by $m$, the success probability of exploration in $M_C(H=\infty,G=1)$ is the highest with almost always 1, while in two $M_C(G=\frac 1 \infty)$ chains it is notably lower.
This displays the impact of goal productivity $G$.
Compared to $M_C(G=\frac 1 \infty)$, a $G=1$ indicates a much higher goal productivity, in which the fruitfulness of reaching the goal is less dependent to the hardness of arriving it.
Therefore, the values for policies always trying to achieve goal are more likely to be accurately estimated in these chains, resulting in a lower requirement for active exploration than the $G=\frac 1 \infty$ cases.

On the other hand, comparing the results with same goal productivity $G$ shows that the chains with higher hazardousness, $H=\infty$, require \textit{less} active exploration than the less hazardous ones, $H=1$.
This sounds a little bit counter-intuitive since a higher hazardousness should implicate a higher hardness of exploration by definition.
The reason behind this ostensible contradiction is actually very clear.
From the theoretical and empirical results of visit numbers in Section \ref{secVerifyNsa}, it can be discovered that in $M_C(H=\infty)$, the number of collected observations till the ending of exploration, $\tau=\tau_m$, is much greater than in $M_C(H=1)$.
In other words, the sample size used for estimating $\hat{V}$ in the cases of $H=\infty$ is larger than that of $H=1$, resulting in a higher accuracy for the former.
If we fix the success probability to the same value, then $H=\infty$ requires more observations than $H=1$, and thus the former is still harder to explore than the latter, which accords well with the intuition.
The moral is, as what we have discussed in Section \ref{secSolveQ3}, in order to compare the hardness of exploration among different MDPs, the number of observations $\tau_{\varepsilon,\delta}$ and the activeness of exploration $\theta_{\varepsilon,\delta}$ required to achieve certain outcome with respect to $(\varepsilon,\delta)$ must be considered together.

To summarize, the experimental results above verified that our analytical approach is able to predict the outcome of exploration with high accuracy.
Furthermore, they showed the usefulness of our approach in forecasting and explaining the exploration behaviour of the agent, as well as in comparing the hardness of exploration among different MDPs.

\section{Applying the Theory}
\label{SectionApplying}
Back in Sections \ref{SectionSolvingSPE} and \ref{SectionApproxSPE}, there has been a long journey climbing up the dependency graph (Figure \ref{figDependency}) to its very summit to obtain the success probability of exploration $\psucc$.
However, the actual procedure of obtaining $\psucc$ is not at all that complicated, since much of the previous analysis serves for justifying the approach rather than being a part of the necessary steps.
In this section, we provide an instance to demonstrate how the theory proposed in this paper can be applied in more general MDPs, then summarize the procedure of analysis in a short practice guide.

\subsection{An Instance of Application to General MDPs}
\label{secMaze}
In this subsection, we will demonstrate that our analytical approach is applicable not only to the prototypes, but also to the general MDPs.

\begin{figure}
	\centering
	\begin{subfigure}[t]{0.4\textwidth}
		\centering
		\begin{tikzpicture}[thick,font=\Large]
			\draw[ultra thick] (0,0) rectangle (100pt,100pt);
			\fill[black!70!white] (40pt,40pt) rectangle (60pt,60pt);
			\fill[black!70!white] (40pt,20pt) rectangle (60pt,40pt);
			\fill[black!70!white] (40pt,0pt) rectangle (60pt,20pt);
	
			\fill[gray!20!white] (20pt,40pt) rectangle (40pt,60pt);
			\fill[gray!20!white] (60pt,40pt) rectangle (80pt,60pt);
	
			\draw (30pt, 30pt) node {\textbf{S}};
			\draw (30pt, 50pt) node {\textbf{T}};
			\draw (70pt, 30pt) node {\textbf{G}};
			\draw (70pt, 50pt) node {\textbf{T}};
			
			\draw[step=20pt] (0,0) grid (100pt,100pt);
		\end{tikzpicture}
		\caption{}
	\end{subfigure}
	\begin{subfigure}[t]{0.4\textwidth}
		\centering
		\begin{tikzpicture}[thick,font=\Large]
			\draw[ultra thick] (0,0) rectangle (100pt,100pt);
			\fill[black!70!white] (40pt,40pt) rectangle (60pt,60pt);
			\fill[black!70!white] (40pt,20pt) rectangle (60pt,40pt);
			\fill[black!70!white] (40pt,0pt) rectangle (60pt,20pt);
	
			\fill[gray!20!white] (20pt,40pt) rectangle (40pt,60pt);
			\fill[gray!20!white] (60pt,40pt) rectangle (80pt,60pt);
	
			\draw (30pt, 30pt) node {{1}};
			\draw (10pt, 30pt) node {{2}};
			\draw (10pt, 50pt) node {{3}};
			\draw (10pt, 70pt) node {{4}};
			\draw (30pt, 70pt) node {{5}};
			\draw (50pt, 70pt) node {{6}};
			\draw (70pt, 70pt) node {{7}};
			\draw (90pt, 70pt) node {{8}};
			\draw (90pt, 50pt) node {{9}};
			\draw (90pt, 30pt) node {{10}};
			\draw (70pt, 30pt) node {{11}};
			\draw (10pt, 89pt) node {$4'$};
			\draw (30pt, 90pt) node {$5'$};
			
			\draw[step=20pt] (0,0) grid (100pt,100pt);
		\end{tikzpicture}	
		\caption{}	
	\end{subfigure}
	\caption{(a) The maze MDP. (b) State IDs are annotated for easier reference.}
	\label{figMazeDomain}
\end{figure}
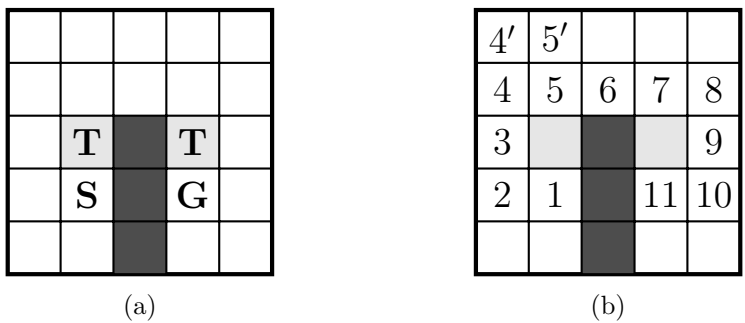

The MDP that will be investigated in this section is shown in Figure \ref{figMazeDomain} (a).
It is an instance of the maze domain, where the general objective for the agent is to reach goal from the start point without falling into the traps.

Specifically, the agent starts from the cell marked with ``S''.
At each time step, the agent can choose one of the four directions and attempt to go to the adjacent cell in that direction.
There is a probability $p$ that the attempt of moving succeeds, and if it fails, the agent will still be in the same cell at next time step.
This can be seen as an abstraction of the real-world source of uncertainty, such as the imperfection of state representation due to the discretization of continuous space, or the imperfection of action execution.
Also, the agent is not allowed to walk into the blocked cells, which are filled with black, or pass through the walls at the edges of the grid.
Stepping into the trapped cells marked with ``T'', on the other hand, will result in the agent to be transited back to the starting cell, and therefore this should be avoided.
Finally, at the goal cell marked with ``G'', an additional action of collecting the reward will be available to the agent.
By taking this action, the agent will be given some reward $r_G>0$, and will be send back to the start point.

The optimal policy in this MDP is to walk along the path that is annotated by the numbers [1-2-3-...-11] in Figure \ref{figMazeDomain} (b), and collect the reward at the goal point.
Other than this optimal one, there also exist many sub-optimal policies that are able to collect the goal reward.
For example, following the path of [1-...-4-$4'$-$5'$-5-...-11] is a possible policy of this kind.
Of course, there are also policies that can never reach goal, for example the one that take the path of [1-2-3-(trap)].

The research question here comes to be whether our approach for analysing the exploration efficiency can be applied to this maze MDP.
Although this maze MDP appears to be rather far from the typical chain MDPs since half of the possible positions are ``off-the-chain'' ones, exploring this MDP is nevertheless very similar to exploring a chain MDP.
Actually, from the chain perspective, the impact of the existence of off-the-chain states on the exploration for the optimal policy is very limited. 
Leaving from and returning to the chain states, for example from state $4$ to $4'$ and its reverse $4'$ to 4, are symmetric in this maze MDP.
The agent has no need to pass through chains states more often in order to explore off-the-chain ones, and vice versa.
As a result, despite that these off-the-chain states provide the agent some alternative policies to reach the goal, their actual impacts on the success probability of the optimal policy $\psucc^*$ are negligible.

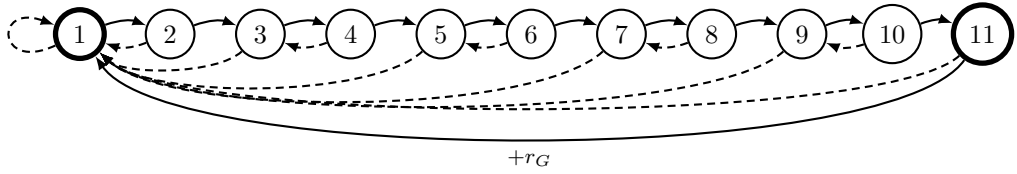
\begin{figure}
	\centering
	\begin{tikzpicture}[
		auto, ->, >=latex,
		thick, 
		font=\scriptsize,
		main node/.style={circle,draw,font=\small},
		special node/.style={circle,draw,font=\small, line width=2pt},
		b reset/.style={densely dashed, bend left = 43, max distance = 40},
		b start/.style={densely dashed, in = 160, out = 205, looseness = 8, left = 1},
		a fail/.style={in = 105, out = 60, looseness = 8, above = 1},
		a goal/.style={bend left = 55, looseness = 8,  max distance = 50},
	]

		\foreach [evaluate=\x as \pos using 1.2*\x-1] \x in {1, 11}
		{
			\node[special node] (\x) at (\pos,0) {$\x$};
		}
		\foreach [evaluate=\x as \pos using 1.2*\x-1] \x in {2,...,10}
		{
			\node[main node] (\x) at (\pos,0) {$\x$};
		}

		\foreach \x/\y in {1/2,2/3,3/4,4/5,5/6,6/7,7/8,8/9,9/10,10/11}
		{
			\path (\x) 	edge [bend left = 20] node {} (\y);
		}
		\foreach \x/\y in {1/2,3/4,5/6,7/8,9/10}
				{
					\path (\y) 	edge [bend left = 20, densely dashed] node {} (\x);
				}
		\foreach  [evaluate=\x as \bent using 0.7-0.03*\x] \x in {3,5,7,9,11}
		{
			\path (\x)  edge [b reset, looseness = \bent] node {} (1);
		}
		\path (1) edge [b start] node {} (1);
		\path (11) edge [a goal] node {$+r_G$} (1);
		
	\end{tikzpicture}
	\caption{Abstracted chain from the maze MDP.}
	\label{figMazeChain}
\end{figure}

Therefore, the original maze MDP can be abstracted to the chain MDP shown in Figure \ref{figMazeChain}.
As in the previous figures of chains, the forward actions, in this case the ones following [1-2-3-...-11], are drawn in solid arrows, while the backward actions are drawn in dashed arrows.
The self-transitions due to the failure of movement are not drawn in the figure for clarity.
All ``off-the-chain'' states and actions are ignored for aforementioned reasons.

The resulting chain MDP differs from the prototypes in two points. 
First, the hazardousness of the states are mixed with $H_i=1$ for odd $i$ and $H_i=\infty$ for even $i$.
The variety in hazardousness is due to the difference cases of connection to the trapped cells.
Second, all actions except for the goal action, including the backward actions, have a chance of $(1-p)$ to fail and result in a self-transition, while in prototype chains only forward actions have failure probabilities.
The second point actually does not change the expression of estimated state values for the optimal policy, since the possible effect of taking backward actions is irrelevant to the optimal policy itself.
Therefore, the main concern here is how the first points affects the visit numbers.

By undertaking the method provided in Section \ref{secClosedVisitNum}, specifically by solving the in and out numbers of transitions according to Lemma \ref{lemLambda} and Equation \ref{eqNsa}, the visit numbers for this chain MDP can be obtained easily.
Actually, it can be proved that they follow the recurrence formula as below:
\begin{equation*}
	\bar{N}^+_{i} = \begin{cases}
		\frac{(1+2p)m}{p} & i=n-1 \\
		\bar{N}^+_{i+1} & i \text{ is odd and } i<n-1 \\
		\bar{N}^+_{i+1}+m & i \text{ is even and } i<n-1. 
	\end{cases}
\end{equation*}

Other theoretical results are directly applicable to this abstracted chain MDP since there are no further critical difference between this one and the prototypes that could affect the validity of our theory.

\begin{figure}
	\centering
	{\includegraphics[scale=0.7,trim=1cm 9cm 1cm 9.2cm]{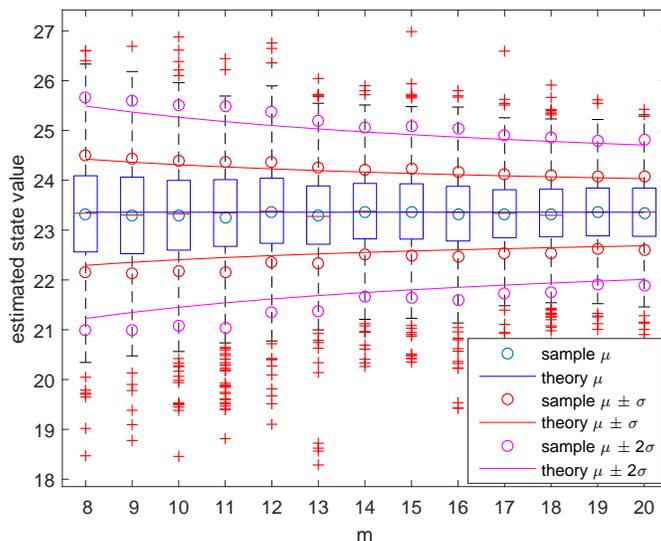}}
	\caption{Sample and theoretical distribution of $\hat{V}^{\pi^*}(s_1)$ in the maze MDP.}
	\label{figMazeResult}
\end{figure}

We conducted some experiments to verify our approach to the maze MDP.
The Optimistic Prototype Strategy with $m$ ranging from 8 to 20 was executed in the original maze MDP shown in Figure \ref{figMazeDomain} (a) with $p=0.5$ and $r_G=1$.
The theoretical and empirical distribution of $\hat{V}^{\pi^*}(s_1)$ is shown in Figure \ref{figMazeResult}.

As can be seen from the figure, our theoretical distribution is generally accurate, although there is a slight lose of dispersion in the theoretical distribution of $\hat{V}^{\pi^*}(s_1)$.
This is a reasonable result because we simply ignored all off-the-chain states when considering the visit numbers for the convenience of analysis.
By dealing with the visit numbers more carefully, this error is likely to be reduced.
Nevertheless, the lose of dispersion here does not have a strong impact on the validity of our analysis, since the difference is only noteworthy at the tails of the distribution.
Furthermore, Kolmogorov-Smirnov test did not reject the null hypothesis of the data being generated from the theoretical log-normal distribution at 1\% significance level under any setting of $m$, and thus the theoretical values are sufficiently accurate.

Having the cumulative distribution of $\hat{V}^{\pi^*}(s_1)$, we are now able to assess the success probability of exploration $\psucc^*$ according to Theorem \ref{theoPiSPE} without actually performing the learning procedure in the maze MDP.
The critical value appeared in Theorem \ref{theoPiSPE} is actually not of importance here.
If we are interested in a set of near-optimal policies that could yield $k\%$ of the optimal cumulative rewards, then we just need to compute the cumulative distribution at $k\%$ of the mean value of $\hat{V}^{\pi^*}(s_1)$, and see what the chance is for $\pi^*$ to be regarded as a near-optimal solution.

Given these points, our approach is applicable to general MDPs with only slight changes to the theoretical results given in previous sections.

\subsection{A Short Practice Guide to Our Approach}
The general steps of obtaining and analysing the success probability of exploration is summarized as follows.

\begin{enumerate}[label={(\arabic*)}]
	\item Abstract a chain MDP from the original MDP as in Section \ref{SectionChainPerspective} and Section \ref{secMaze}. 
	\begin{itemize}[nosep]
		\item Try to keep as much as possible of the critical characteristics that affects the hardness of exploration.
		\item It does not need to be one of the prototype chain defined in Section \ref{secProtoChains}, since most of the main results apply to non-prototype chains as well.
	\end{itemize}
	
	\item Abstract an easy-to-analysis version of exploration strategy from the original one.
	\begin{itemize}[nosep]
		\item If the original one is optimistic, it is highly recommended to carry out abstraction based on the Optimistic Prototype Strategy defined in Section \ref{secProtoStrategy}.
	\end{itemize}
	
	\item Evaluate the probability of traverse event.
	\begin{itemize}[nosep]
		\item In most chain MDPs, Theorem \ref{theoTrav} can be applied directly or with trivial modification for OPS and its modifications.
	\end{itemize}
	
	\item Evaluate the visit numbers.
	\begin{itemize}[nosep]
		\item Although Lemma \ref{lemExpressionNsa} cannot directly be applied to general chains, Lemma \ref{lemLambda} and Equation \ref{eqNsa} are still valid, and thus the expected visit numbers can usually be solved without effort as in Section \ref{secMaze}. 
		The result is likely to be similar in the form with the ones in Lemma \ref{lemExpressionNsa}.
		\item After obtaining the expected visit numbers, apply Rule 2 and Rule 3 of Definition \ref{defThreeRulesNsa} to evaluate the distribution of actual visit numbers.
	\end{itemize}
	
	\item Solve Bellman equations to obtain the expressions of state values.
	\begin{itemize}[nosep]
		\item Not only that this step is independent to  exploration strategies, but the expressions in Lemma \ref{lemExpressionV} (c) and (d) are also irrelevant to the hazardousness.
		Since most general chains differ from the prototype ones in that the hazardousness level varies among the chain states, these expressions are directly applicable to general chains.
		
		\item Expressions in Lemma \ref{lemExpressionV} (a) and (b) merely serve for deciding the critical values that the key state values must exceed or being exceeded, as discussed Lemma \ref{lemCompPBF1} and Lemma \ref{lemCompPBF2}.
		It is likely that they can be inferred directly from prior knowledge without solving Bellman equations, as in the example of maze domain in Section \ref{secMaze}.
	\end{itemize}
	
	\item Evaluate the final success probability.
	\begin{itemize}[nosep]
		\item  There are two key steps: first, decide the relevant estimated state values and their corresponding critical values; second, estimate the cumulative distribution of relevant estimated state values.
		
		\item The first step can be accomplished by following the dominance analysis in Lemma \ref{lemDominance}, Lemma \ref{lemCompPBF1} and Lemma \ref{lemCompPBF2}. In many cases Lemma \ref{lemCompPBF1} and Lemma \ref{lemCompPBF2} can be applied directly or by slight modification.
		
		\item The approximation by delta method introduced in Lemma \ref{lemDeltaMethod} in Section \ref{SectionApproxSPE} is recommended to decide the distribution of estimated state values.
		The results of Lemma \ref{lemMeanVarF} and Lemma \ref{lemMeanVarV} are widely applicable because the expressions here are not dependent to the hazardousness.
	\end{itemize}
\end{enumerate}

As can be seen from this summary, a considerable number of main results are directly applicable to the general  MDPs.
Although some of them may have to be adjusted if the original MDP or the exploration strategy deviate from the prototypes too much, in practice this might not be a problem due to the reasons explained in Section \ref{secProtoChains} and Section \ref{secProtoStrategy}, as long as the strategy is optimistic.

\section{Conclusion and Discussion}
\label{SectionDiscussion}
There has been a long-lasting gap between theory and practice of reinforcement learning.
Although several analytical frameworks have been established in the literature, in general they lack the ability to satisfy the practical needs.
This paper is an attempt of bridging the gap under a new framework, namely the success probability of exploration, so that the practice can actually benefit from the theory, rather than simply ignores it and relies on experiences and domain knowledge.

Looking back to the previous sections, we have formulated the success probability of exploration, introduced its basic properties, elaborated our concrete approach to evaluating it, and verified our approach via empirical results.
We have also showed that our novel framework does not suffer from the problems as the previous ones, and demonstrated that it can be used to comprehensively solve the three groups of questions mentioned in Section \ref{sec3Q}.
Although our framework and approach may not cover every problem encountered by us RL practitioners every day, we believe that our attempt here will be helpful for accelerating the process of closing the gap and making theories more useful.

Our analytical framework is the first one that is able to explain and predict the behaviour and possible outcome of exploration strategy in such a detailed manner.
To our best knowledge, there is no previous work concerning elaborating the visit numbers for state-action pairs and transitions as ours in Section \ref{secClosedVisitNum}, nor to connect them with the final outcomes as in Sections \ref{secClosedValueFunc} and \ref{secClosedSPE}.
We consider these parts as necessary steps for gaining deeper vision of exploration behaviour and sample efficiency of reinforcement learning, and therefore we hope our results could benefit a wider range of theorists as well. 

Nevertheless, much works still remain to be done in the future. 
The chain perspective in Section \ref{SectionChainPerspective} is still under construction, and a more rigorous formulation is in need.
A constructive or algorithmic method of abstracting chain MDPs from general ones is also highly desirable.
If such method is made available, then we will be able to compare the hardness of exploration between any two general MDPs, which may lead to deeper understanding of the efficiency of exploration and better strategies.
Another possible research direction is to extend the results to non-optimistic exploration strategies, for example the Bayesian approaches.
Lastly, our current theoretical analysis mainly focus on finite discrete MDPs; it is interesting to see if our theoretical results can be generalized to other cases, in particular the continuous MDPs.



%

\acks{We would like to acknowledge... }

\vskip 0.2in
\bibliography{paper_rtp}

\end{document}